\documentclass[12pt]{article}

\usepackage{natbib}
\usepackage{microtype}
\usepackage{subfigure}
\usepackage{booktabs} 
\usepackage{fullpage}

\usepackage[breaklinks,colorlinks,
            linkcolor=red,
            anchorcolor=blue,
            citecolor=blue
            ]{hyperref}

\usepackage{amsmath}
\usepackage{amssymb}
\usepackage{mathtools}
\usepackage{amsthm}
\usepackage{subcaption}
\usepackage{graphicx}
\usepackage{float}
\graphicspath{{Figures/}}

\usepackage[capitalize,noabbrev]{cleveref}

\theoremstyle{plain}
\newtheorem{theorem}{Theorem}[section]
\newtheorem{proposition}[theorem]{Proposition}
\newtheorem{lemma}[theorem]{Lemma}
\newtheorem{corollary}[theorem]{Corollary}
\theoremstyle{definition}
\newtheorem{definition}[theorem]{Definition}
\newtheorem{assumption}[theorem]{Assumption}
\theoremstyle{remark}
\newtheorem{remark}[theorem]{Remark}
\Crefname{assumption}{Assumption}{Assumptions}

\usepackage{relsize}
\usepackage[utf8]{inputenc} 
\usepackage[T1]{fontenc}    
\usepackage{url}            
\usepackage{amsfonts}       
\usepackage{nicefrac}       
\usepackage{smile}
\usepackage{xcolor}
\usepackage{bbm}

\title{\LARGE Adaptive Estimation and Inference in Conditional Moment Models via the Discrepancy Principle}
\author{Jiyuan Tan\thanks{Jiyuan Tan was supported by NSF Award IIS-2337916.}\\
Vasilis Syrgkanis\thanks{Vasilis Syrgkanis was supported by NSF Award IIS-2337916.}}
\date{}
\begin{document}
\maketitle

\begin{abstract}
We study adaptive estimation and inference in linear inverse problems defined by conditional moment restrictions. These problems are typically ill-posed, so estimators require regularization, yet the optimal regularization strength depends on a source parameter that is almost never known. We develop a discrepancy principle that chooses this parameter from empirical losses and prescribed bounds on statistical error, without knowing the source parameter. The framework covers Regularized DeepIV (RDIV) and the Tikhonov Regularized Adversarial Estimator (TRAE). Under the stated realizability conditions and comparable fitting and calibration error bounds, both one-step procedures attain strong- and weak-error bounds that match known-source rates for source parameters up to one. We also extend the framework to iterative procedures under realizability of the population iterates, widening the range of source parameters covered by the rate comparison. Finally, we use these adaptive estimators to construct an asymptotically normal doubly robust estimator that adapts to the well-posedness of the primal and dual problems.

\end{abstract}

\section{Introduction}
Ill-posed inverse problems arise throughout causal inference and econometrics \citep{horowitz2014ill}, for example in nonparametric instrumental variable (IV) regression \citep{neweyInstrumentalVariableEstimation2003,darollesNonparametricInstrumentalRegression2011,neweyNonparametricInstrumentalVariables2013,aiEfficientEstimationModels2003}, proximal causal inference \citep{miao2018confounding,tchetgen2020introduction}, and missing-not-at-random problems \citep{d2010new,miao2015identification,breunig2021nonparametric,li2023non}. We study estimating parameters defined by
\[
\theta _{0} =\mathbb{E}\left[\tilde{m}( W;h_{0})\right] ,
\]
where $h\mapsto\tilde m(W;h)$ is a known linear map. The nuisance function $h_0$ solves the conditional moment problem
\begin{equation}\label{eq:cm}
\mathbb{E}[ h_0(X) \mid Z=z] =r_{0}(z) ,
\end{equation}
where $X$ and $Z$ are $W$-measurable random variables. The function $r_0$ represents the expectation of a known linear map: $\mathbb E[m(W;q)]=\mathbb E[r_0(Z)q(Z)]$ for the test functions $q$ that \cref{sec:pre} introduces.

Two established approaches estimate nuisance functions from conditional moment restrictions. Linear inverse problem (LIP) methods formulate the restriction as an operator equation and regularize its solution \citep{carrascoChapter77Linear2007,darollesNonparametricInstrumentalRegression2011,bennett2023source}. Sieve minimum distance methods approximate the unknown functions in finite-dimensional spaces and minimize a sample criterion \citep{aiEfficientEstimationModels2003,chenEstimationNonparametricConditional2012}. We follow the LIP approach. Let $\Pi_{\tilde{\mathcal Q}}$ denote the mean-square projection onto a closed test space. The operator $\mathcal T h=\Pi_{\tilde{\mathcal Q}}\mathbb E[h(X)\mid Z=\cdot]$ then turns \eqref{eq:cm} into the equation $\mathcal T h_0=r_0$. We estimate the minimum-norm solution of this equation with regularized estimators over general hypothesis classes.

While many works use Tikhonov regularization for this inverse problem \citep{neweyNonparametricInstrumentalVariables2013,li2024regularized,bennettInferenceStronglyIdentified2023a,bennett2023source}, a central challenge in applying LIP methods is selecting the regularization parameter. Known-smoothness rate analyses for the Tikhonov regularized adversarial estimator (TRAE) \citep{bennettInferenceStronglyIdentified2023a,bennett2023source} and Regularized DeepIV (RDIV) \citep{li2024regularized} express the regularization level in terms of the source parameter---often encoded by the $\beta$-source condition (i.e., $h_{0}=\left(\mathcal{T}^{*}\mathcal{T}\right)^{\beta/2}w_{0}$ for some $w_{0}\in\tilde{\mathcal{H}}$). In practice, $\beta$ is rarely known. A misspecified regularization parameter can lead to suboptimal convergence rates or even divergence. The L-curve offers a geometric criterion for choosing the regularization parameter from the solution and residual norms \citep{hansen1992analysis}. Cross-validation (CV) is also widely used, but it requires repeatedly minimizing the loss and is computationally expensive. Moreover, CV typically selects the parameter that minimizes an estimate of the weak error, and a small weak error does not by itself imply a bound in the strong metric. We address this gap with a discrepancy-principle framework that selects the regularization parameter without knowledge of $\beta$.

We build on the discrepancy principle (DP) \citep{morozov1966solution} from classical inverse problems. DP selects the regularization parameter $\lambda$ so that the empirical loss is of the same order as the noise level. We extend the classical DP to address the following challenges in our setting. First, classical LIP assumes known operators; in our setting the conditional expectation operator $\mathcal{T}$ is unknown and must be estimated, either explicitly or implicitly, before solving the inverse problem. Second, classical DP assumes fixed, known noise in $r_0$, whereas our noise comes from empirical-process fluctuations and depends on the estimator; see, e.g., the localized concentration bounds in \citet[Lemma~14]{foster2023orthogonal}. Third, classical DP is typically analyzed over full Hilbert spaces with closed-form regularized solutions, while we work with general hypothesis classes (e.g., neural networks) that satisfy inductive-bias assumptions.

Our main contributions are as follows.
\begin{itemize}
    \item We give a discrepancy rule and an oracle guarantee that apply to a broad class of estimators, and we apply them to RDIV and TRAE. For source parameters up to one, and when fitting and calibration errors are comparable, we show that the discrepancy principle selects a regularization parameter with the same strong- and weak-error rates as in the known-source analysis.
    \item We additionally extend the adaptive estimator to iterated regularized RDIV and TRAE, and we show that, for a prescribed iteration count $t$, the adaptive iterated estimator matches the known-source rates over the range $\beta\leqslant2t-1$.
    \item We combine adaptive TRAE fits of the primal and dual problems into a doubly robust estimator of $\theta_0$. We show that this estimator is asymptotically normal and can adapt to the well-posedness of either problem.
    \item We evaluate the rule with kernel classes in two simulation designs, where the best fixed regularization parameter changes by three orders of magnitude across settings. With a data-driven noise level, the selected parameter follows the best fixed value of each setting, and, averaged over settings, adaptive TRAE has smaller strong, weak, and target errors than the best single fixed value.
\end{itemize}

The rest of the paper proceeds as follows. Section \ref{sec:pre} sets up the problem, fixes notation, and reviews linear inverse problems and the discrepancy principle. Section \ref{sec:dp_regularization} develops the adaptive regularization framework and proves the guarantees for RDIV, TRAE, and doubly robust estimation. Section \ref{sec:experiments} evaluates the rule in simulations with kernel classes, and the conclusion discusses implications and open directions.

\subsection{Related Work}
\paragraph{Nonparametric instrumental variables.}
Nonparametric IV estimation provides a central example of a statistical inverse problem. Sieve and kernel estimators address the ill-posedness of the conditional expectation operator through finite-dimensional approximation or regularization \citep{neweyInstrumentalVariableEstimation2003,hallNonparametricMethodsInference2005,darollesNonparametricInstrumentalRegression2011,chenEstimationNonparametricConditional2012}. The resulting convergence rates depend on both the regularity of the structural function and the operator \citep{chenRATEOPTIMALITYILLPOSED2011,florensIDENTIFICATIONESTIMATIONPENALIZATION2011}. Recent estimators accommodate flexible function classes through kernel methods \citep{singh2019kernel,muandet2020dual}, minimax formulations of the moment restrictions \citep{lewis2018adversarial,dikkala2020minimax,liao2020provably,bennett2023variational}, or two-stage learning of conditional expectations \citep{hartford2017deep,xu2021features,li2024regularized}. Our analysis builds on the regularized least-squares estimator of \citet{li2024regularized} and the regularized adversarial estimator of \citet{bennett2023source}.

\paragraph{Adaptive estimation in IV models.}
Data-driven regularization has an established literature in nonparametric IV models. \citet{johannes2011partially} develop a model-selection procedure that adapts over smoothness classes under restrictions on the conditional expectation operator. \citet{horowitz2014adaptive} select the dimension of a series estimator through an empirical criterion and establish adaptation to unknown smoothness. \citet{breunig2016adaptive} combine model selection and Lepski's method to estimate linear functionals at adaptive rates. More recently, \citet{chen2025adaptive} develop data-driven sieve dimension choices that attain optimal sup-norm rates for structural functions and their derivatives, together with adaptive uniform confidence bands. These procedures adapt the dimension of a sieve or series expansion. We instead fix a general hypothesis class and adapt the Tikhonov regularization parameter itself, combining localized loss bounds for RDIV and TRAE \citep{li2024regularized,bennett2023source} with a discrepancy search over a geometric grid.

\paragraph{Discrepancy principles and iterative regularization.}
The discrepancy principle chooses the amount of regularization by comparing an observed residual with a noise threshold \citep{morozov1966solution,morozov1984methods}. Classical inverse-problem theory relates this rule to source conditions and explains how iterated Tikhonov regularization extends the range of smoothness that the estimator can use \citep{engl1996regularization,nair2003morozov}. Statistical analyses account for the fluctuations of the residual itself. \citet{blanchard2018optimal} establish oracle adaptation in the strong norm for residual stopping rules in linear inverse problems. \citet{celisse2021discrepancy} study discrepancy rules for kernel spectral estimators, including Tikhonov regularization, and derive oracle inequalities and adaptive rates for ordinary and smoothed residuals. \citet{harrach2020beyond} analyze regularization with noise estimated from repeated measurements. Our conditional moment setting adds estimation of the operator and optimization over hypothesis classes. The projected RDIV discrepancy and the adversarial discrepancy account for these statistical errors within a common finite-grid selection argument. The iterative analyses apply this argument to a prescribed number of Tikhonov steps.

\paragraph{Inference on linear functionals.}
Inference on functionals of conditional moment solutions connects this regularization problem to semiparametric estimation \citep{aiEfficientEstimationModels2003}. \citet{bennettInferenceStronglyIdentified2023a} give a condition under which a linear functional is strongly identified and construct debiased estimators from regularized primal and dual fits. \citet{bennett2023source} develop source-condition double robustness, under which sufficiently favorable regularity of either inverse problem can support asymptotic normality, and analyze iterated adversarial Tikhonov estimators. We use their primal--dual representation and remainder analysis with discrepancy-selected nuisance estimators. This yields a sufficient condition for asymptotic normality expressed in the primal and dual source parameters while choosing the regularization levels from the data.

\section{Preliminaries}\label{sec:pre}

\textbf{Problem formulation.} We observe independent copies of a random variable $W$ and estimate a linear functional whose nuisance function is defined by an LIP. The target is 
\[
\theta _{0} =\mathbb{E} [\tilde{m} (W;h_{0} )],
\]
with a known linear functional $h\mapsto \tilde{m} (W;h)$. Let $X$ and $Z$ be two $W$-measurable random variables, and let $ \tilde{\mathcal{H}}$ and $ \tilde{\mathcal{Q}}$ be closed linear subspaces of $ L_{2}( X)$ and $ L_{2}( Z)$, respectively, where $L_2(X)$ denotes square-integrable functions of $X$. Let $\mathcal{H} \subset \tilde{\mathcal{H}}$ and $\mathcal{Q} \subset \tilde{\mathcal{Q}}$ be subsets of these function spaces. The classes $\mathcal H$ and $\mathcal Q$ may consist of bounded subsets of reproducing kernel Hilbert spaces, neural networks, or parametric families. We define the linear operator $ \mathcal{T} h=\Pi _{\tilde{\mathcal{Q}}}\mathbb{E}[ h( X) \mid Z=\cdot ]$, where $ \Pi _{\mathcal{V}}$ denotes mean-square projection onto a closed subspace $ \mathcal{V}$. The nuisance function $h_{0}$ is defined as the minimum-norm solution in $\tilde{\mathcal{H}}$ to the LIP
\begin{equation}\label{eq:lin_inv_prob}
\mathcal{T} h_0=r_{0} ,
\end{equation}
where $r_0$ is the Riesz representer of the continuous linear functional $q\mapsto\mathbb E[m(W;q)]$, so that 
\[
\mathbb{E}[ m( W;q)] =\mathbb{E}[ r_0( Z) q( Z)] ,\quad \forall q\in \tilde{\mathcal{Q}} .
\]
We make the following assumption to ensure that solutions exist. 

\begin{assumption}\label{asp:range}
We have $ r_{0} \in \mathcal{R}(\mathcal{T}) :=\{\mathcal{T} h:h\in \tilde{\mathcal{H}}\}$.
\end{assumption}
Given an estimated solution $ \hat{h}$ of (\ref{eq:lin_inv_prob}), the strong and weak metrics are defined as 
\[
\| \hat{h} -h_{0} \| _{\text{strong}} =\| \hat{h} -h_{0} \| _{L_{2}} ,\qquad
\| \hat{h} -h_{0} \| _{\text{weak}} =\| \mathcal{T}(\hat{h} -h_{0}) \| _{L_{2}} .
\]
The weak metric measures error only after applying $\mathcal T$, so ill-posedness lets a small weak error coexist with a large strong one. The adjoint of $\mathcal T$ is 
\[
\mathcal T^*q=\Pi_{\tilde{\mathcal H}}\mathbb E[q(Z)\mid X=\cdot] .
\]
Let $a_0$ be the Riesz representer of $h\mapsto\mathbb E[\tilde m(W;h)]$, and let $q_0$ denote the minimum-norm solution, when it exists, of the dual inverse problem
\[
\mathcal{T}^{*} q_{0} =a_{0} .
\]

When the dual problem has a solution, estimates of $h_{0}$ and $q_{0}$ yield the DR estimator \citep{bennett2023source} of $\theta_0$,
\[
\hat{\theta }(\hat{h},\hat{q}) =\mathbb{E}_{n}\left[\tilde{m}( W;\hat{h}) +m( W;\hat{q}) -\hat{q}( Z) \hat{h}( X)\right] .
\]
For \cref{cor:normality}, this empirical average is taken over an evaluation sample independent of the data used to fit and tune the nuisance estimators.

\textbf{Linear inverse problems.} Classical inverse-problem theory studies a noisy observation $r_\delta$ of $r_0$, often with a known operator $\mathcal T$ and error bound $\|r_\delta-r_0\|\leqslant\delta$. Inverting an ill-posed operator can amplify this perturbation. To illustrate this point, suppose that the operator $ \mathcal{T}$ admits a countable singular value decomposition 
\begin{equation}\label{eq:svd}
\mathcal{T} h =\sum _{i=1}^{\infty } \sigma _{i} \langle h,v_{i} \rangle u_{i} ,
\end{equation}
where $\sigma_1\geqslant\sigma_2\geqslant\cdots>0$ are the positive singular values, and $\{v_i\}_{i=1}^{\infty}$ and $\{u_i\}_{i=1}^{\infty}$ form orthonormal bases of $(\ker\mathcal T)^\perp$ and $(\ker\mathcal T^*)^\perp$, respectively. The minimum-norm solution of (\ref{eq:lin_inv_prob}) is
\[
h_{0} =\sum _{i=1}^{\infty }\frac{\langle r_{0} ,u_{i} \rangle }{\sigma _{i}} v_{i} .
\]
The inversion amplifies a perturbation along $u_k$ by $1/\sigma_k$. Tikhonov regularization controls this amplification by solving
\begin{equation}\label{eq:tikh_sol}
h_{\lambda }^{\delta }=\arg\min_{h}\left\{\| \mathcal{T} h-r_{\delta} \| ^{2}+\lambda \| h\| ^{2}\right\}.
\end{equation}
Larger values of $\lambda$ shrink the solution toward zero; smaller values allow greater noise amplification. The discrepancy principle balances these effects through the residual. 
\begin{definition}[Discrepancy Principle]\label{def:dp} Given constants $\ell\in(1,\infty)$ and $\kappa\in(1,\infty)$, choose the regularization parameter $\lambda$ as follows:
\begin{enumerate}
    \item If $ \| r_{\delta } \| \leqslant \kappa \delta $, choose $ h=0$, which corresponds to $ \lambda =\infty $.
    \item If $ \| r_{\delta } \|  >\kappa \delta $, choose $ \lambda $ such that 
    \[
\| \mathcal{T} h_{\lambda }^{\delta } -r_{\delta } \| \leqslant \kappa \delta \leqslant \| \mathcal{T} h_{\lambda '}^{\delta } -r_{\delta } \| ,
\]
    for some $ \lambda '\in [ \lambda ,\ell \lambda ]$.
\end{enumerate}

\end{definition}
The first rule uses the default estimator $h=0$ when the observed right-hand side is no larger than the prescribed noise threshold. The second rule makes the residual $\|\mathcal T h_\lambda^\delta-r_\delta\|$ comparable to that threshold. This residual is measured after applying $\mathcal T$; controlling the reconstruction error in the strong metric additionally requires regularity, such as a source condition. The safety factor $\kappa>1$ makes the second rule feasible: the true solution satisfies $\|\mathcal T h_0-r_\delta\|=\|r_0-r_\delta\|\leqslant\delta<\kappa\delta$, so a residual at the threshold is attainable.

In our statistical setting, both $\mathcal T$ and $r_0$ are unknown, which introduces perturbations absent from the classical setting. With an explicit operator estimate $\hat{\mathcal T}_n$, the identity $\hat{\mathcal T}_n h_0=r_0+e_n$, where $e_n=(\hat{\mathcal T}_n-\mathcal T)h_0$, exhibits one such perturbation. Calibration must also absorb empirical fluctuations and, for the raw RDIV squared loss, an irreducible outcome-noise floor. An adversarial estimator does not form $\hat{\mathcal T}_n$; it accesses the operator only through a minimax objective, so its discrepancy contains different error terms. A single search rule, with a statistical-error threshold that may depend on $\lambda$, covers both cases.

To analyze Tikhonov regularization, we introduce the source condition, a standard assumption in the inverse-problem literature. 
\begin{assumption}[Source Condition]\label{asp:source_condition}
Problem (\ref{eq:lin_inv_prob}) satisfies the $ \beta $-source condition for some $ \beta  >0$, i.e., there exists $ w_{0} \in \tilde{\mathcal{H}}$ such that the minimum-norm solution of (\ref{eq:lin_inv_prob}) satisfies $ h_{0} =\left(\mathcal{T}^{*}\mathcal{T}\right)^{\beta /2} w_{0}$.
\end{assumption}

If $\mathcal{T}$ is compact, then by (\ref{eq:svd}), \cref{asp:source_condition} implies that 
$$ \sum_{i=1}^\infty \frac{\langle h_0,v_i \rangle^2}{\sigma^{2\beta}_i} < \infty, $$
which quantifies the decay of the solution coefficients relative to the singular values. This assumption is standard \citep{engl1996regularization,darollesNonparametricInstrumentalRegression2011,li2024regularized,bennettInferenceStronglyIdentified2023a,bennett2023source}. Throughout, we write
\begin{equation}\label{eq:p_alpha}
p=\min\{\beta+1,2\}=1+\min\{\beta,1\},
\qquad
\alpha=\frac{p-1}{p}=\frac{\min\{\beta,1\}}{1+\min\{\beta,1\}}.
\end{equation}
For one-step Tikhonov regularization, the squared bias in the weak metric is of order $\lambda^p$, and $\alpha$ is the exponent with which a weak-error bound controls the strong error.

\textbf{Notation.} Throughout the paper, we use $ \| \cdotp \| $ for the $ L_{2}$ norm with respect to the input distribution, i.e., \[
\| h\| :=\| h\| _{L_{2}} =\sqrt{\mathbb{E}[ h( X)^2]} ,\ \forall h\in L_{2}( X)
\]
and the $L_1$ norm $\|h\|_1=\mathbb E[|h(X)|]$. The same notation applies to functions of $Z$ under their input distribution. Write $\|h\|_{2,n}=\{\mathbb E_n[h(X)^2]\}^{1/2}$ for the empirical norm. For real numbers, $a\vee b=\max\{a,b\}$, and $\log$ denotes the natural logarithm. We use $\mathbb E_n[\cdot]$ to denote an empirical average, with subscripts indicating the sample when several independent samples are used. We use the asymptotic-order notation $o(\cdot)$, $O(\cdot)$, $\Omega(\cdot)$, and $\Theta(\cdot)$; $o_p(1)$ denotes convergence to zero in probability. For a function class $\mathcal F$, let $\operatorname{star}(\mathcal F)=\{\gamma f:\gamma\in[0,1],\ f\in\mathcal F\}$ and $\mathcal F-\mathcal F=\{f-f':f,f'\in\mathcal F\}$. For independent observations $X_1,\ldots,X_n$ from the input distribution, define the expected Rademacher complexity and its localized version by
\begin{align*}
\mathfrak R_n(\mathcal F)&=\mathbb E_{X_{1:n},\epsilon}\left[\sup_{f\in\mathcal F}\left|\frac1n\sum_{i=1}^n\epsilon_i f(X_i)\right|\right],\\
\mathfrak R_n(\delta;\mathcal F)&=\mathbb E_{X_{1:n},\epsilon}\left[\sup_{f\in\mathcal F,\,\|f\|\leqslant\delta}\left|\frac1n\sum_{i=1}^n\epsilon_i f(X_i)\right|\right],
\end{align*}
where the $\epsilon_i$ are independent Rademacher variables, independent of the observations. The localized class uses the population norm. A deterministic critical radius $\delta_{n,\mathcal F}>0$ satisfies
\[
\delta_{n,\mathcal F}^2\geqslant \mathfrak R_n\bigl(\delta_{n,\mathcal F};\operatorname{star}(\mathcal F-\mathcal F)\bigr).
\]
For classes fitted on an independent preliminary sample, we apply this definition conditionally on that sample and use a common deterministic upper bound. We assume throughout that the function classes are regular enough for the suprema in these empirical-process bounds to be measurable.
 
\section{Adaptive Regularization via the Discrepancy Principle} \label{sec:dp_regularization}

We now state a finite-grid oracle guarantee for the discrepancy rule applied to estimators of \eqref{eq:lin_inv_prob}, and then check its conditions for two estimators: RDIV, which estimates the operator explicitly, and TRAE, which estimates the weak loss through an adversarial problem.

\subsection{Discrepancy Principle for Conditional Moment Problems}

Consider a family of candidate estimators $\{\hat h_\lambda:\lambda\in\Lambda_n\}$ and a nonnegative calibration statistic $D_n(\hat h_\lambda)$. Fix $\lambda_0=2$ and $\rho\in[1/2,1)$, independent of $n$, and set
\[
K_n=\left\lceil\frac{\log(2n^2)}{\log(1/\rho)}\right\rceil,
\qquad
\Lambda_n=\{\lambda_k=2\rho^k:k=0,\ldots,K_n\}.
\]
Then $|\Lambda_n|=O(\log n)$ and $\lambda_{K_n}\leqslant n^{-2}$. For a prescribed noise level $A_n(\lambda)$, we use the threshold $\tfrac32A_n(\lambda)$. The rule computes the candidates from large to small regularization and returns the first point whose calibration statistic lies below this threshold.

\begin{algorithm}[H]
    \caption{Adaptive Regularization via Discrepancy Principle}
    \label{alg:discrepancy}
    \begin{algorithmic}[1]
        \STATE \textbf{Input:} Candidate estimators $\hat h_\lambda$, discrepancy $D_n$, noise level $A_n$.
        \STATE \textbf{Input:} Initial parameter $\lambda_0=2$, search factor $1/2\leqslant\rho<1$.
        \STATE Set $K_n=\lceil\log(2n^2)/\log(1/\rho)\rceil$.
        \FOR{$k=0,\ldots,K_n$}
            \STATE Set $\lambda=\lambda_0\rho^k$ and compute $\hat h_\lambda$.
            \STATE Evaluate $D_n(\hat h_\lambda)$.
            \IF{$D_n(\hat h_\lambda)\leqslant\tfrac32 A_n(\lambda)$}
                \STATE \textbf{Output:} $\lambda$ and $\hat h_\lambda$; terminate.
            \ENDIF
        \ENDFOR
        \STATE \textbf{Output:} $\lambda_{K_n}$ and $\hat h_{\lambda_{K_n}}$ (fallback, used only if no candidate is accepted).
    \end{algorithmic}
\end{algorithm}

If no candidate is accepted, the rule returns the terminal point $\lambda_{K_n}$, so the selected estimator is defined for every outcome. \Cref{thm:common_dp} and \cref{prop:exist_lambda_deepiv,prop:exist_lambda} show that the rule accepts a candidate, and so never reaches the fallback, on the events they describe.

When the rule accepts a noninitial parameter, the selected point $\lambda_{\mathrm{dp}}$ and its rejected predecessor $\lambda'=\lambda_{\mathrm{dp}}/\rho\in[\lambda_{\mathrm{dp}},2\lambda_{\mathrm{dp}}]$ satisfy
\begin{equation}\label{eq:discrepancy_principle}
D_n(\hat h_{\lambda_{\mathrm{dp}}})\leqslant\tfrac32A_n(\lambda_{\mathrm{dp}}),
\qquad
D_n(\hat h_{\lambda'})>\tfrac32A_n(\lambda').
\end{equation}
Acceptance at $\lambda_{\mathrm{dp}}$ bounds the weak error; rejection at $\lambda'$, or acceptance at an oracle point (the point $\lambda^\circ$ of \cref{thm:common_dp}), keeps $\lambda_{\mathrm{dp}}$ from falling too far.

The statistic $D_n(\hat h_\lambda)$ plays the role of the squared residual in \cref{def:dp}, and $\tfrac32A_n(\lambda)$ plays the role of $\kappa^2\delta^2$. For RDIV, $D_n$ is a squared residual computed from estimates of the operator and right-hand side. For TRAE, it is an empirical adversarial loss. The statistical bounds in \eqref{eq:common_accepted_bounds} turn a small discrepancy into a small squared weak error. When the first accepted parameter has a predecessor, the two inequalities in \eqref{eq:discrepancy_principle} bracket the threshold at adjacent grid points with ratio $1/\rho\in(1,2]$.

The following theorem states conditions on the discrepancy bound and threshold under which the rule has an oracle guarantee. Throughout, $\xi\in(0,1)$ is a generic confidence level; we absorb fixed numerical rescalings of $\xi$ into the rate bounds and relabel it after intersecting finitely many event families.

\begin{theorem}[Oracle guarantee for the discrepancy rule]\label{thm:common_dp}
Let $A_n:\Lambda_n\to[0,\infty)$ and fix a deterministic grid point $\lambda^\circ\in\Lambda_n$. Fix constants $C_A,C_V>0$ and $\gamma\in(0,1]$. Suppose that, on an event $\mathcal E_n$ with probability at least $1-\xi$, the following hold.
\begin{enumerate}
  \item[(i)] \emph{The oracle point passes the test:} $D_n(\hat h_{\lambda^\circ})\leqslant\tfrac32A_n(\lambda^\circ)$.
  \item[(ii)] \emph{The error is controlled by the discrepancy:} every $\lambda\in\Lambda_n\cap[\lambda^\circ,2]$ satisfies
\begin{equation}\label{eq:common_accepted_bounds}
\begin{aligned}
\|\mathcal T(\hat h_\lambda-h_0)\|^2
&\leqslant C_A\{D_n(\hat h_\lambda)+A_n(\lambda)\},\\
\|\hat h_\lambda-h_0\|^2
&\leqslant C_V\left[\frac{D_n(\hat h_\lambda)+A_n(\lambda)}{\lambda}+\{D_n(\hat h_\lambda)+A_n(\lambda)\}^\gamma\right].
\end{aligned}
\end{equation}
  \item[(iii)] \emph{The noise level is monotone:} $\lambda\mapsto A_n(\lambda)$ is nonincreasing over $\Lambda_n\cap[\lambda^\circ,2]$.
\end{enumerate}
Then, on $\mathcal E_n$, \cref{alg:discrepancy} accepts a candidate before exhausting the grid and returns $\lambda_{\mathrm{dp}}\geqslant\lambda^\circ$ with
\begin{align}\label{eq:common_rate}
\|\mathcal T(\hat h_{\lambda_{\mathrm{dp}}}-h_0)\|^2
&\leqslant \tfrac52C_A\,A_n(\lambda^\circ),\\
\|\hat h_{\lambda_{\mathrm{dp}}}-h_0\|^2
&\leqslant C_V\left[\tfrac52\,\frac{A_n(\lambda^\circ)}{\lambda^\circ}+\{\tfrac52A_n(\lambda^\circ)\}^\gamma\right].
\end{align}
\end{theorem}

The rule compares $D_n$ with $\tfrac32A_n$ and never evaluates the regularization bias. The theorem addresses two ways in which the rule could stop at a poor parameter. First, it could stop at a large $\lambda$ with a large bias. This does no harm: condition (ii) bounds the error of every candidate by its observed discrepancy, acceptance caps that discrepancy at $\tfrac32A_n(\lambda_{\mathrm{dp}})$, and condition (iii) bounds $A_n(\lambda_{\mathrm{dp}})$ and $A_n(\lambda_{\mathrm{dp}})/\lambda_{\mathrm{dp}}$ by their values at $\lambda^\circ$. Second, it could stop at a $\lambda$ so small that the factor $1/\lambda$ in the strong-error bound is large. This cannot happen: the search runs from large to small regularization and $\lambda^\circ$ passes the test by condition (i), so the rule stops at some $\lambda_{\mathrm{dp}}\geqslant\lambda^\circ$.

\begin{remark}[Verifying condition (i)]\label{rem:verify_oracle}
In the applications, condition (i) follows from a discrepancy bound with an explicit bias term $B_n(\lambda)$. If
\begin{equation}\label{eq:common_discrepancy_bound}
D_n(\hat h_{\lambda^\circ})\leqslant A_n(\lambda^\circ)+B_n(\lambda^\circ)
\end{equation}
and $\lambda^\circ$ is chosen so that
\begin{equation}\label{eq:common_oracle_crossing}
B_n(\lambda^\circ)\leqslant\tfrac12A_n(\lambda^\circ),
\end{equation}
then $D_n(\hat h_{\lambda^\circ})\leqslant\tfrac32A_n(\lambda^\circ)$. In the one-step analyses the bias bound has the form $B_n(\lambda)=C\|w_0\|^2\lambda^{p}$, and the iterative analyses replace $p$ by a larger exponent. The bias bound depends on the source parameter, so this step, and only this step, uses $\beta$. The bounds in \eqref{eq:common_rate} improve as $\lambda^\circ$ increases, while $B_n$ increases with $\lambda$, so the applications take $\lambda^\circ$ to be the largest grid point satisfying \eqref{eq:common_oracle_crossing}, up to the factor $\rho$. This point balances the noise level $A_n(\lambda^\circ)$, which plays the role of a variance, against the bias $B_n(\lambda^\circ)$. It is the choice one would make if $\beta$ were known; hence we call $\lambda^\circ$ the oracle point, and $A_n(\lambda^\circ)$ attains the known-source rate.
\end{remark}

The proof of \cref{thm:common_dp} is in \cref{sec:proof_common_dp}. The following subsections specify $D_n$, $A_n$, $B_n$, and $\lambda^\circ$ for RDIV and TRAE.

\paragraph{Noise calibration.}
The tuning rule takes valid statistical-error bounds as inputs. In the following analyses, $\delta_n$ and $\nu_n$ denote deterministic upper bounds multiplied by fixed constants large enough to absorb the concentration and stability constants. These multipliers may depend on the boundedness, density-ratio, and mean-square-continuity constants, but are independent of $n$, $\lambda$, and the unknown source parameter. They are fixed before fitting the candidate estimators. The constants $a_t$ and $b_t$ used for the iterative procedures also absorb the propagation constants for the prescribed iteration count. The discrepancy bounds in the Appendix show that such choices exist. This convention establishes adaptation to the source parameter conditional on valid bounds for the function classes and preliminary nuisances. Estimating all of these bounds from the data requires a separate calibration analysis.

\subsection{Adaptive Regularized DeepIV Estimator}\label{sec:rdiv}

In the IV setting, $r_0(z)=\mathbb E[Y\mid Z=z]$ and $\mathcal T h=\mathbb E[h(X)\mid Z=\cdot]$; thus, in this subsection, $\tilde{\mathcal Q}=L_2(Z)$ and $\mathcal T$ is the conditional expectation operator. The RDIV estimator first estimates the conditional density of $X$ given $Z$ with respect to a common dominating measure $\mu$, then uses the resulting operator in a penalized least-squares problem \citep{li2024regularized}. For the analysis that follows, split the observations into three independent samples $I_{\mathrm{pre}},I_{\mathrm{fit}},I_{\mathrm{cal}}$, each of size $n$. Fixed positive sample fractions give the same rates in the total sample size. On $I_{\mathrm{pre}}$, fit
\[
\hat g=\arg\max_{g\in\mathcal G}\mathbb E_{n,\mathrm{pre}}[\log g(X\mid Z)],
\qquad
(\hat{\mathcal T}h)(z)=\int h(x)\hat g(x\mid z)\,\mu(dx),
\]
and a regression estimate $\hat r$ of $r_0$ from $Y$ on $Z$. Write $\mathbb E_n=\mathbb E_{n,\mathrm{fit}}$ within this subsection. The RDIV fit is
\[
L_n(h)=\mathbb E_n[(Y-\hat{\mathcal T}h(Z))^2],
\qquad
\hat h_\lambda=\arg\min_{h\in\mathcal H}\{L_n(h)+\lambda\mathbb E_n[h(X)^2]\}.
\]
We assume these empirical minima are attained and measurable.

The population squared loss decomposes into an outcome-noise level and the squared weak error:
\[
\sigma^2_{Y\mid Z}:=\mathbb E[(Y-r_0(Z))^2],
\qquad
L(h):=\mathbb E[(Y-\mathcal T h(Z))^2]
=\sigma^2_{Y\mid Z}+\|\mathcal T(h-h_0)\|^2.
\]
The raw loss exceeds the squared weak error by the constant $\sigma^2_{Y\mid Z}$, so it cannot be compared directly with a threshold for statistical error. We instead calibrate on $I_{\mathrm{cal}}$ with the projected residual
\begin{equation}\label{eq:rdiv_projected_loss}
D_n(h)=\mathbb E_{n,\mathrm{cal}} [(\hat r(Z)-\hat{\mathcal T}h(Z))^2].
\end{equation}
The preliminary sample makes $\hat r$ and $\hat{\mathcal T}$ independent of $I_{\mathrm{cal}}$. At their population values, with ground-truth $r_0$ and $\mathcal T$, for a fixed $h$,
\[
\mathbb E[D_n(h)]=\|r_0-\mathcal T h\|^2=\|\mathcal T(h-h_0)\|^2.
\]
Calibration therefore requires one regression fit for $\hat r$ plus one loss evaluation per candidate, on top of the $O(\log n)$ regularized fits.

\begin{assumption}[Solution realizability]\label{asp:closeness_deepiv}
The minimum-norm solution $h_0$ lies in $\mathcal H$.
\end{assumption}

\begin{assumption}[Conditional-density realizability and nuisance rate]\label{asp:re_cd}
The conditional density $g_0$ belongs to $\mathcal G$, whose elements are conditional densities with respect to a common dominating measure. For fixed constants $c_0,C_G>0$,
\[
g_0(x\mid z)\geqslant c_0, \qquad \sup_{g\in\mathcal G}\|g\|_\infty\leqslant C_G.
\]
For a prescribed $\beta$-independent accuracy bound $\delta_{n,r}=\delta_{n,r}(\xi)\geqslant0$, the nuisance satisfies
\[
\|\hat r-r_0\|_{L_2(Z)}\leqslant \delta_{n,r}
\]
with probability at least $1-\xi$.
\end{assumption}

\begin{assumption}[Boundedness and critical radius]\label{asp:cr_deepiv}
For a fixed $B<\infty$, $|Y|\leqslant B$ almost surely, every $h\in\mathcal H$ satisfies $|h(x)|\leqslant B$ on the common support of the conditional densities, and $|\hat r(z)|\leqslant B$. Clipping the nuisance to $[-B,B]$ weakly decreases its $L_2$ error. For $g\in\mathcal G$, write
\[
(\mathcal T_g h)(z)=\int h(x)g(x\mid z)\,\mu(dx).
\]
Define the likelihood-transform class
\[
\mathcal G_0=
\left\{(x,z)\mapsto
\sqrt{\frac{g(x\mid z)+g_0(x\mid z)}{2g_0(x\mid z)}}:
g\in\mathcal G\right\}.
\]
Let $\eta_n=\xi/(K_n+1)$. The $\beta$-independent sequence $\delta_{n,\mathcal G}=\delta_{n,\mathcal G}(\xi)$ upper-bounds the critical radius of $\mathcal G_0$, while $\delta_{n,\mathcal H}=\delta_{n,\mathcal H}(\eta_n)$ is a common upper bound on the critical radii of
\[
\mathcal H
\quad\text{and}\quad
\mathcal T_g\mathcal H:=\{\mathcal T_g h:h\in\mathcal H\},
\qquad g\in\mathcal G,
\]
under their respective input distributions. Moreover,
\[
\delta_{n,\mathcal G}\geqslant
C\sqrt{\frac{\log\log(en)+\log(1/\xi)}{n}},
\qquad
\delta_{n,\mathcal H}\geqslant
C\sqrt{\frac{\log\log(en)+\log(1/\eta_n)}{n}}.
\]
\end{assumption}

The density, boundedness, and complexity conditions follow the RDIV analysis of \citet{li2024regularized}; the nuisance rate controls the additional regression used for calibration. Define calibrated scales satisfying
\begin{equation}\label{eq:delta_deepiv}
\begin{aligned}
\delta_{n,\xi}&:=\sqrt{\frac{\log((K_n+1)/\xi)}{n}},
&\qquad \delta_n&\geqslant\delta_{n,\mathcal G}+\delta_{n,\mathcal H}+\delta_{n,\xi},\\
\nu_n&\geqslant\delta_{n,\mathcal G}+\delta_{n,r}+\delta_{n,\xi},
&\qquad \delta_n+\nu_n&=o(1).
\end{aligned}
\end{equation}
The scale $\delta_n$ controls the RDIV fitting error, whereas $\nu_n$ controls the projected discrepancy. 

With $\delta_n$ and $\nu_n$ chosen as in the noise-calibration paragraph, \cref{cor:rdiv_discrepancy_bound} gives a fixed constant $C_{\mathrm R}>0$ such that, on an event with probability at least $1-\xi$,
\begin{equation}\label{eq:rdiv_discrepancy_bound}
D_n(\hat h_\lambda)
\leqslant
\frac{\delta_n^2}{\lambda}+\nu_n^2
+C_{\mathrm R}\|w_0\|^2\lambda^p,
\end{equation}
for every $\lambda\in\Lambda_n$, with $p$ as in \eqref{eq:p_alpha}. The first two terms are the noise level used by the rule, while the last term is the regularization bias used only in the analysis.
These bounds fit \cref{thm:common_dp} with
\begin{equation}\label{eq:rdiv_common_scales}
A_n^{\mathrm R}(\lambda)
=\frac{\delta_n^2}{\lambda}+\nu_n^2,
\qquad \lambda\in\Lambda_n.
\end{equation}
The rule therefore compares $D_n$ with $\tfrac32A_n^{\mathrm R}(\lambda)$. As $\lambda$ decreases, the fitting bound $\delta_n^2/\lambda$ grows and the threshold loosens accordingly.
The Appendix proves the error bounds in
\eqref{eq:common_accepted_bounds} with $\gamma=\alpha$. For a sufficiently small fixed
$c_{\mathrm R}>0$, let $\lambda_{\mathrm R}^\circ$ be the first grid point no larger than
\[
r_{n,\beta}^{\mathrm R} =c_{\mathrm R}\max\left\{ \delta_n^{2/(p+1)},\nu_n^{2/p} \right\}.
\]
This scale balances the bias bound $C_{\mathrm R}\|w_0\|^2\lambda^p$ against the noise level $A_n^{\mathrm R}(\lambda)$: the first term in the maximum solves $\lambda^p=\delta_n^2/\lambda$, and the second solves $\lambda^p=\nu_n^2$. When $\nu_n=O(\delta_n)$ and $\beta\leqslant1$, $r_{n,\beta}^{\mathrm R}\asymp\delta_n^{2/(\beta+2)}$, the regularization level chosen with known $\beta$ in \citet[Theorem~7]{li2024regularized}.
Then \eqref{eq:common_oracle_crossing} holds and
\begin{equation}\label{eq:rdiv_oracle_scales}
A_n^{\mathrm R}(\lambda_{\mathrm R}^\circ)=O\{\delta_n^{2p/(p+1)}+\nu_n^2\},
\qquad
\frac{A_n^{\mathrm R}(\lambda_{\mathrm R}^\circ)}{\lambda_{\mathrm R}^\circ}=O\{\delta_n^{2(p-1)/(p+1)}+\nu_n^{2\alpha}\}.
\end{equation}

The rule accepts a candidate whenever the last grid point passes the test, which happens once the sample size is large enough. For RDIV the required sample size is
\[
N_{\mathrm R}(B) :=\max\left\{2, \left\lceil\frac{8B^2}{3\log 2}\right\rceil\right\},
\]
which equals $4$ under the normalization $B=1$.

\begin{proposition}\label{prop:exist_lambda_deepiv}
Suppose \cref{asp:range,asp:source_condition,asp:closeness_deepiv,asp:re_cd,asp:cr_deepiv} hold. For every $n\geqslant N_{\mathrm R}(B)$, \cref{alg:discrepancy}, with the projected statistic \eqref{eq:rdiv_projected_loss} and noise level \eqref{eq:rdiv_common_scales}, accepts a candidate after at most $K_n+1=O(\log n)$ regularized fits, with probability at least $1-\xi$. If $h_0=0$, the initial parameter is accepted. In general, the returned parameter either is the initial parameter or is a noninitial grid point satisfying \eqref{eq:discrepancy_principle}.
\end{proposition}

\begin{theorem}\label{thm:deepiv_rate}
Under the assumptions of \cref{prop:exist_lambda_deepiv}, let $\lambda_{\mathrm{dp}}$ be its output and write
\[
\varepsilon_{n,\beta}=\delta_n^{2p/(p+1)}+\nu_n^2,
\]
where $p$ and $\alpha$ are as in \eqref{eq:p_alpha}.
For every $n\geqslant N_{\mathrm R}(B)$, with probability at least $1-\xi$,
\begin{align*}
\|\mathcal T(\hat h_{\lambda_{\mathrm{dp}}}-h_0)\|^2
\leqslant O(\varepsilon_{n,\beta}), \quad \|\hat h_{\lambda_{\mathrm{dp}}}-h_0\|^2
\leqslant O\!\left(
\delta_n^{2(p-1)/(p+1)}
+\varepsilon_{n,\beta}^{\alpha}\right)
\leqslant O(\varepsilon_{n,\beta}^\alpha).
\end{align*}
If $h_0=0$, the squared strong and weak errors are both $O(\delta_n^2)$.
\end{theorem}

If $\nu_n=O(\delta_n)$, fitting and projected calibration operate at the same statistical rate, and the theorem gives
\[
\|\hat h_{\lambda_{\mathrm{dp}}}-h_0\|^2
=O\!\left(\delta_n^{2(p-1)/(p+1)}\right),
\qquad
\|\mathcal T(\hat h_{\lambda_{\mathrm{dp}}}-h_0)\|^2
=O\!\left(\delta_n^{2p/(p+1)}\right).
\]
In particular, for $0<\beta\leqslant1$, the squared strong-error rate is $O(\delta_n^{2\beta/(\beta+2)})$, which matches the noniterative RDIV rate obtained when $\beta$ is known \citep[Theorem~7]{li2024regularized}. For $\beta>1$, the projected-residual bias bound uses smoothness up to $\beta=1$, giving the squared strong-error rate $O(\delta_n^{2/3})$. A nuisance with $\delta_{n,r}=O(\delta_n)$ suffices for these rates. More generally, the nuisance rate enters the calibration term through $\nu_n^2$. The proof applies \cref{thm:common_dp} with \eqref{eq:rdiv_common_scales}--\eqref{eq:rdiv_oracle_scales} by establishing the discrepancy bound and the weak- and strong-error bounds at every grid point for RDIV.

\subsection{Adaptive Iterative Regularized DeepIV Estimator}

Iterated Tikhonov regularization extends the range of source smoothness used by RDIV and improves the convergence rates \citep{li2024regularized}. Fix an integer $t\geqslant1$, chosen independently of the data and the source parameter $\beta$, and initialize
\[
h_{0,\lambda}^*=\hat h_{0,\lambda}=0.
\]
For $j=1,\ldots,t$, define the population and empirical iterates by
\begin{align}
h_{j,\lambda}^*
&=\arg\min_{h\in\tilde{\mathcal H}}
\left\{\|\mathcal T h-r_0\|^2+\lambda\|h-h_{j-1,\lambda}^*\|^2\right\},
\label{eq:iter_population} \\ \hat h_{j,\lambda}
&=\arg\min_{h\in\mathcal H}
\left\{L_n(h)+\lambda\|h-\hat h_{j-1,\lambda}\|_{2,n}^2\right\}.
\label{eq:iter_rdiv_estimator}
\end{align}
We assume that each empirical minimum is attained and that the selected minimizers are measurable. The first iterate equals the RDIV estimator in \cref{sec:rdiv}. The iterative candidate at $\lambda$ is $\hat h_{t,\lambda}$, and its discrepancy remains the projected calibration loss $D_n(\hat h_{t,\lambda})$ in \eqref{eq:rdiv_projected_loss}.

For the $t(K_n+1)$ fitted iterates, set
\[
\eta_{n,t}=\frac{\xi}{t(K_n+1)} .
\]
We continue to write $\delta_n$ and $\nu_n$ as in \eqref{eq:delta_deepiv}. Because $t$ is fixed, the resulting $\log t$ term is a constant adjustment to $\delta_n$ and $\nu_n$.

Let $a_t\geqslant1$ be a deterministic constant that depends only on the
prescribed iteration count $t$ and fixed population quantities. Define
\[
q_t=\min\{\beta+1,2t\}, \qquad s_{n,t}^{\mathrm R}=a_t\delta_n^2.
\]
We then use 
\begin{equation}\label{eq:iter_rdiv_common_scales}
A_{n,t}^{\mathrm R}(\lambda)=\frac{s_{n,t}^{\mathrm R}}{\lambda}+\nu_n^2
\end{equation}
for the noise level and run \cref{alg:discrepancy} with the threshold $\tfrac32A_{n,t}^{\mathrm R}(\lambda)$.
For a sufficiently small fixed $c_{\mathrm{IR}}>0$, let
$\lambda_{\mathrm{IR}}^\circ$ be the first grid point no larger than
\[
r_{n,t,\beta}^{\mathrm{IR}}
=c_{\mathrm{IR}}\max\left\{
(s_{n,t}^{\mathrm R})^{1/(q_t+1)},\nu_n^{2/q_t}
\right\}.
\]
As in the one-step case, the two terms balance the bias bound $C_{\mathrm{IR}}\|w_0\|^2\lambda^{q_t}$ against the two parts of $A_{n,t}^{\mathrm R}(\lambda)$. When $\nu_n=O(\delta_n)$ and $2t\geqslant\beta+1$, $r_{n,t,\beta}^{\mathrm{IR}}\asymp\delta_n^{2/(\beta+2)}$, the regularization level chosen with known $\beta$ in \citet[Theorem~14]{li2024regularized}.
For all sufficiently large $n$, this point satisfies
\eqref{eq:common_oracle_crossing}, and
\begin{equation}\label{eq:iter_rdiv_oracle_scales}
A_{n,t}^{\mathrm R}(\lambda_{\mathrm{IR}}^\circ)=O_t\{(s_{n,t}^{\mathrm R})^{q_t/(q_t+1)}+\nu_n^2\},
\qquad
\frac{s_{n,t}^{\mathrm R}}{(\lambda_{\mathrm{IR}}^\circ)^2}=O_t\{(s_{n,t}^{\mathrm R})^{(q_t-1)/(q_t+1)}\}.
\end{equation}
For the strong error, the iterative analysis uses a sharper bound than condition (ii) of \cref{thm:common_dp}: \cref{lemma:iter_rdiv_accepted_bounds} replaces $A_{n,t}^{\mathrm R}(\lambda)/\lambda$ by $s_{n,t}^{\mathrm R}/\lambda^2$, which removes the calibration term $\nu_n^2/\lambda$.

\begin{assumption}[Iterative RDIV realizability]\label{asp:iter_realizability_rdiv}
For every $n$ for which $\lambda_{\mathrm{IR}}^\circ$ is defined,
$h_{j,\lambda}^*\in\mathcal H$ for $j=1,\ldots,t$ and
$\lambda\in\Lambda_n\cap[\lambda_{\mathrm{IR}}^\circ,2]$.
\end{assumption}
On this segment, \cref{cor:iter_rdiv_discrepancy_bound} gives
\begin{equation}\label{eq:iter_rdiv_discrepancy_bound}
D_n(\hat h_{t,\lambda})
\leqslant
A_{n,t}^{\mathrm R}(\lambda)
+C_{\mathrm{IR}}\|w_0\|^2\lambda^{q_t},
\qquad
\lambda\in\Lambda_n\cap[\lambda_{\mathrm{IR}}^\circ,2].
\end{equation}

\begin{proposition}\label{prop:exist_lambda_iter_rdiv}
  Suppose the assumptions of \cref{prop:exist_lambda_deepiv} hold. For every $n\geqslant N_{\mathrm R}(B)$, \cref{alg:discrepancy}, applied to $\hat h_{t,\lambda}$ with the noise level \eqref{eq:iter_rdiv_common_scales}, accepts a candidate with probability at least $1-\xi$. It solves at most $t(K_n+1)=O(t\log n)$ regularized least-squares problems. If $h_0=0$, the initial parameter is accepted. Otherwise, the output is either the initial parameter or a noninitial point satisfying \eqref{eq:discrepancy_principle} with $D_n(\hat h_{t,\lambda})$ and $A_{n,t}^{\mathrm R}(\lambda)$.
\end{proposition}

\begin{theorem}[Adaptive iterative RDIV]\label{thm:iter_rdiv_rate}
Under the assumptions of
\cref{prop:exist_lambda_iter_rdiv,asp:iter_realizability_rdiv}, let
$\lambda_{\mathrm{dp},t}^{\mathrm R}$ denote the selected parameter and set
\[
\varepsilon_{n,t,\beta}^{\mathrm R} =(s_{n,t}^{\mathrm R})^{q_t/(q_t+1)}+\nu_n^2.
\]
For every fixed $t$ and every $n\geqslant N_{\mathrm R}(B)$, with probability at least $1-\xi$,
\begin{align*}
\|\mathcal T(\hat h_{t,\lambda_{\mathrm{dp},t}^{\mathrm R}}-h_0)\|^2
\leqslant O_t(\varepsilon_{n,t,\beta}^{\mathrm R}), \quad  \|\hat h_{t,\lambda_{\mathrm{dp},t}^{\mathrm R}}-h_0\|^2
\leqslant O_t\!\left(
(s_{n,t}^{\mathrm R})^{(q_t-1)/(q_t+1)}
+(\varepsilon_{n,t,\beta}^{\mathrm R})^{\beta/(\beta+1)}
\right).
\end{align*}
Here $O_t(\cdot)$ allows constants that depend on the prescribed iteration count and fixed population quantities. In particular, $a_t$ is independent of $n$. If $\nu_n=O(\delta_n)$ and $2t\geqslant\beta+1$, the bounds reduce to
\begin{align*}
\|\hat h_{t,\lambda_{\mathrm{dp},t}^{\mathrm R}}-h_0\|^2
=O_t\!\left(\delta_n^{2\beta/(\beta+2)}\right), \quad  \|\mathcal T(\hat h_{t,\lambda_{\mathrm{dp},t}^{\mathrm R}}-h_0)\|^2
=O_t\!\left(\delta_n^{2(\beta+1)/(\beta+2)}\right).
\end{align*}
\end{theorem}

The condition $2t\geqslant\beta+1$ marks the range in which a fixed $t$-step filter exploits the full source smoothness: iterating widens that range from $\beta\leqslant1$ to $\beta\leqslant2t-1$, and the discrepancy rule still selects $\lambda$ without knowing $\beta$.
 
\subsection{Adaptive Tikhonov Regularized Adversarial Estimator} \label{sec:trae}

The RDIV bounds include the error from estimating the operator explicitly. TRAE uses a minimax formulation to control the weak loss without fitting a conditional density. We reuse $L_n$ and $L$ for the TRAE losses and write $\mathbb E_n$ for the average over its fitting sample of size $n$. For a general linear moment functional $m(W;f)$, its empirical loss is
\[
L_n(h) =\sup_{f\in\mathcal F}\mathbb{E}_{n}\left[ 2m( W;f) -2h( X) f( Z) -f( Z)^{2}\right],
\]
where $\mathcal F\subset\tilde{\mathcal Q}$ is the test class. The population loss is
\[
L(h) =\sup_{f\in\mathcal F}\mathbb{E}\left[ 2m( W;f) -2h( X) f( Z) -f( Z)^{2}\right].
\]
The TRAE estimator then minimizes the penalized empirical risk:
\[
\hat{h}_{\lambda } =\arg\min_{h\in \mathcal{H}}\left\{L_n(h)+\lambda\mathbb E_n[h(X)^2]\right\}.
\] 
We assume that the empirical loss is measurable, its minimum is attained, and the selected minimizer is measurable.
The following assumption makes $L(h)$ equal to the squared weak error.

\begin{assumption}[Closeness]\label{asp:closeness}
    For all $h \in \mathcal{H}$, $\mathcal T(h_0-h)\in\mathcal F$.
\end{assumption}

Under \cref{asp:closeness}, the population loss simplifies to the desired squared error:
\begin{align*}
    L(h) & = \sup_{f\in\mathcal F}\mathbb{E}\left[ 2(r_0(Z)-(\mathcal T h)(Z))f(Z) - f(Z)^{2}\right] \\
    & = \sup_{f\in\mathcal F}\mathbb{E}\left[-\{f(Z)-(r_0(Z)-(\mathcal T h)(Z))\}^2 +\{r_0(Z)-(\mathcal T h)(Z)\}^2\right] \\
    & = \|\mathcal T h-r_0\|^2 = \|\mathcal{T}(h-h_0)\|^2,
\end{align*}
where the first equality uses $f\in\tilde{\mathcal Q}$ and hence
$\mathbb E[h(X)f(Z)]=\langle\mathcal T h,f\rangle$.
The next two assumptions control fluctuations of the empirical loss. 

\begin{assumption}[Mean-square continuity]\label{asp:mean_sq_continuity}
 For a fixed constant $C<\infty$ and all $ f\in \mathcal{F}$, $ \mathbb{E}\left[ m( W;f)^{2}\right] \leqslant C\| f\| ^{2}$.
\end{assumption}

\begin{assumption}[Boundedness and critical radius]\label{asp:critical_radius}
 Let
\[
\mathcal{H} \cdotp \mathcal{F} =\left\{(x,z)\mapsto h(x)f(z):h\in\mathcal H,f\in\mathcal F\right\} ,\qquad
m\circ \mathcal{F} =\left\{w\mapsto m(w;f):f\in\mathcal F\right\} .
\]
    For a fixed $B<\infty$, suppose that $|m(W;f)|$, $|h(X)|$, and $|f(Z)|$ are a.s. bounded by $B$, uniformly over $ h\in \mathcal{H} ,f\in \mathcal{F}$. $ \tilde{\delta}_{n}(\xi) = \Omega \left(\sqrt{(\log\log(en) +\log( 1/\xi )) /n}\right)$ is a common upper bound on the critical radii of $ \operatorname{star}(\mathcal{H} \cdotp \mathcal{F})$, $ \operatorname{star}( m\circ \mathcal{F})$, $ \operatorname{star}(\mathcal{F})$, $ \operatorname{star}(\mathcal{H})$, and $\operatorname{star}(\mathcal H-\mathcal H)$.
\end{assumption}
Critical-radius conditions such as \cref{asp:critical_radius} yield fast rates through localized concentration \citep{foster2023orthogonal}. Because its function classes and nuisance requirements differ from those of RDIV, its critical-radius scale is unrelated to the scales in \eqref{eq:delta_deepiv}; within this subsection, $\delta_n$ always means the TRAE scale. For the finite grid $\Lambda_n$, set $\eta_n=\xi/(K_n+1)$ and choose $\delta_n$ so that
\begin{equation}\label{eq:delta_trae}
\delta_n \geqslant \max\{\tilde{\delta}_n(\eta_n),n^{-1/2}\}
 = \Omega \left(\sqrt{\frac{\log\log(en)+\log((K_n+1)/\xi)}{n}}\right),
 \qquad \delta_n=o(1).
\end{equation}
Up to the rescaling of the confidence level described before \cref{thm:common_dp}, the concentration bounds for each fixed $\lambda$ hold jointly with failure probability at most $\eta_n$. A union bound over the $K_n+1$ candidates then gives a common event $\mathcal E_n$ with probability at least $1-\xi$. Because $K_n=O(\log n)$, the grid contributes $O(\log\log(en))$ inside the logarithmic confidence term.

For TRAE, the uniform loss comparison has an additive remainder $O(\delta_n^2)$ over the regularization grid. We therefore take $D_n=L_n$ and use the noise level
\begin{equation}\label{eq:trae_common_scales}
A_n^{\mathrm T}(\lambda)=\delta_n^2.
\end{equation}
The threshold is $\tfrac32A_n^{\mathrm T}(\lambda)$. The constant noise level has the order of the squared weak-error bound available when $\beta$ is known \citep{bennett2023source}, while the factor $1/\lambda$ in $A_n^{\mathrm T}(\lambda)/\lambda$ measures the cost of turning that weak control into a strong-error bound.

With $\delta_n$ chosen as in the noise-calibration paragraph, \cref{cor:trae_discrepancy_bound} gives a fixed constant $C_{\mathrm T}>0$ such that, on $\mathcal E_n$,
\begin{equation}\label{eq:trae_discrepancy_bound}
L_n(h_0)\leqslant\delta_n^2,\qquad
L_n(\hat h_\lambda)
\leqslant\delta_n^2
+C_{\mathrm T}\|w_0\|^2\lambda^{p},
\quad \lambda\in\Lambda_n.
\end{equation}
The last term is the regularization bias, with $p$ as in \eqref{eq:p_alpha}. For all sufficiently large $n$, let
$\lambda_{\mathrm T}^\circ$ be the first grid point no larger than
\[
r_{n,\beta}^{\mathrm T}=c_{\mathrm T}\delta_n^{2/p},
\]
where $c_{\mathrm T}>0$ is a sufficiently small fixed constant. This scale solves $\lambda^p\asymp\delta_n^2$, so the bias term $C_{\mathrm T}\|w_0\|^2\lambda^p$ in \eqref{eq:trae_discrepancy_bound} balances the noise level $A_n^{\mathrm T}=\delta_n^2$. It is the order of the regularization level chosen with known $\beta$ in the discussion following \citet[Theorem~4]{bennett2023source}. Then the regularization bias at $\lambda_{\mathrm T}^\circ$ is at most
$A_n^{\mathrm T}(\lambda_{\mathrm T}^\circ)/2$, and
\begin{equation}\label{eq:trae_oracle_scales}
A_n^{\mathrm T}(\lambda_{\mathrm T}^\circ)
=\delta_n^2,\qquad
\frac{A_n^{\mathrm T}(\lambda_{\mathrm T}^\circ)}{\lambda_{\mathrm T}^\circ}
=O\!\left(\delta_n^{2(p-1)/p}\right).
\end{equation}

As for RDIV, acceptance of a candidate requires a minimum sample size, here
\[
N_{\mathrm T}(B) :=\max\left\{2,\left\lceil 2B^2\right\rceil\right\},
\]
which equals $2$ under the normalization $B=1$.

\begin{proposition}\label{prop:exist_lambda}
    Suppose that Assumptions \ref{asp:range}, \ref{asp:source_condition},
    \ref{asp:closeness_deepiv}, \ref{asp:closeness},
    \ref{asp:mean_sq_continuity}, and \ref{asp:critical_radius} hold.
    Run \cref{alg:discrepancy} with inputs $\lambda_0 = 2$, a fixed $1/2\leqslant\rho<1$, and
    noise level $A_n^{\mathrm T}(\lambda)=\delta^2_n$, where $\delta_n$ satisfies \eqref{eq:delta_trae}.
    For every $n\geqslant N_{\mathrm T}(B)$, the algorithm accepts a candidate after at most $K_n+1=O(\log n)$
    optimization problems, with probability at least $1-\xi$.
    The returned parameter either is $\lambda_0$ or is a noninitial grid point satisfying
    \eqref{eq:discrepancy_principle}.
\end{proposition}

\begin{theorem}\label{thm:tikhonv_rate}
Under the assumptions of \cref{prop:exist_lambda}, let $ {\lambda_{\mathrm{dp}}} $ be chosen by \cref{alg:discrepancy}. For every $n\geqslant N_{\mathrm T}(B)$, with probability at least $ 1-\xi $, 
\[
\| \hat{h}_{\lambda_{\mathrm{dp}}} -h_{0} \| ^{2} \leqslant O\left( \delta _{n}^{2\alpha}\right) ,\qquad
\| \mathcal{T}(\hat{h}_{\lambda_{\mathrm{dp}}} -h_{0}) \| ^{2} \leqslant O\left( \delta _{n}^{2}\right) ,
\]
where $\delta_n$ is defined in (\ref{eq:delta_trae}) and $\alpha=\min\{\beta,1\}/(1+\min\{\beta,1\})$ as in \eqref{eq:p_alpha}.
\end{theorem}

When the two methods have comparable $\delta_n$ and $\nu_n=O(\delta_n)$ for RDIV, the exponent in \cref{thm:tikhonv_rate} improves on the one in \cref{thm:deepiv_rate}: the squared strong-error rate $O(\delta_n^{2\alpha})$ recovers the known-smoothness rate of \citet{bennett2023source} under \cref{asp:closeness}. The proof controls the inner maximization with localized concentration inequalities and compares the empirical optimizer with $h_0$ through the variational source inequality; this yields the discrepancy bound and the weak- and strong-error bounds that \cref{thm:common_dp} requires, with \eqref{eq:trae_common_scales}--\eqref{eq:trae_oracle_scales}.

\subsection{Adaptive Iterative Tikhonov Regularized Adversarial Estimator}

The iterative TRAE uses the same population iterates \eqref{eq:iter_population} and initializes $\hat h_{0,\lambda}=0$. For $j=1,\ldots,t$, its empirical recursion is
\[
\hat h_{j,\lambda}=\arg\min_{h\in\mathcal H}\left\{L_n(h)+\lambda\|h-\hat h_{j-1,\lambda}\|_{2,n}^2\right\},
\]
where $L_n$ is the adversarial loss in \cref{sec:trae}. The first iterate equals the noniterative TRAE. The higher-qualification analysis uses the realizability condition in \cref{asp:iter_realizability_trae}; we evaluate the critical-radius bound in \eqref{eq:delta_trae} at $\eta_{n,t}=\xi/\{t(K_n+1)\}$.

\begin{assumption}[Mean-square continuity on differences]\label{asp:iter_mean_sq_continuity}
For all $f,f'\in\mathcal F$,
\[
\mathbb{E}\left[m(W;f-f')^{2}\right]\leqslant C\|f-f'\|^{2}
\]
for a fixed constant $C<\infty$.
\end{assumption}
\Cref{asp:iter_mean_sq_continuity} extends \cref{asp:mean_sq_continuity} from single test functions to their differences, which the recursive loss comparison needs; the iterative analysis uses both. It holds for $m(W;f)=Yf(Z)$ when $Y$ is bounded. As for iterative RDIV, we assume that each empirical minimum is attained and that the selected minimizers are measurable.

The iterative TRAE analysis gives two bounds on the statistical part of the weak error, of orders $\lambda$ and $\delta_n^2/\lambda$; compare \citet[Theorem~6]{bennett2023source}. Let $b_t\geqslant2$ be a deterministic constant that depends on the prescribed iteration count $t$ and the fixed class bounds, independently of $\beta$ and $\|w_0\|$, and write
$s_{n,t}^{\mathrm T}=b_t\delta_n^2$ and
$q_t=\min\{\beta+1,2t\}$. Define
\begin{equation}\label{eq:iter_trae_scale}
A_{n,t}^{\mathrm T}(\lambda)
=\delta_n^2+\min\left\{\lambda,\frac{s_{n,t}^{\mathrm T}}{\lambda}\right\}.
\end{equation}
Keeping the minimum retains the sharper of the two weak-metric bounds but makes the threshold nonmonotone on the full grid. The oracle point defined below satisfies $\lambda_{\mathrm{IT}}^\circ\geqslant\sqrt{s_{n,t}^{\mathrm T}}$. Hence, on $[\lambda_{\mathrm{IT}}^\circ,2]$, the noise level equals $\delta_n^2+s_{n,t}^{\mathrm T}/\lambda$ and is nonincreasing, as required by \cref{thm:common_dp}. Because the oracle point is accepted, the rule stops in this interval and never reaches the grid points where the minimum equals $\lambda$. For all sufficiently large $n$, let $\lambda_{\mathrm{IT}}^\circ$ be the first grid point no larger than
\[
r_{n,t}^{\mathrm{IT}} =c_{\mathrm{IT}}(s_{n,t}^{\mathrm T})^{1/(q_t+1)},
\]
where $c_{\mathrm{IT}}>0$ is a sufficiently small fixed constant. On $[\sqrt{s_{n,t}^{\mathrm T}},2]$ the noise level is at least $s_{n,t}^{\mathrm T}/\lambda$, and this scale solves $\lambda^{q_t}\asymp s_{n,t}^{\mathrm T}/\lambda$, so it balances the bias bound $C_{\mathrm{IT}}\|w_0\|^2\lambda^{q_t}$ against the noise level. When $2t\geqslant\beta+1$, $r_{n,t}^{\mathrm{IT}}\asymp\delta_n^{2/(\beta+2)}$, the regularization level chosen with known $\beta$ in the discussion following \citet[Theorem~6]{bennett2023source}. Then
$\lambda_{\mathrm{IT}}^\circ\geqslant\sqrt{s_{n,t}^{\mathrm T}}$, the regularization bias at $\lambda_{\mathrm{IT}}^\circ$ is at most
$A_{n,t}^{\mathrm T}(\lambda_{\mathrm{IT}}^\circ)/2$, and
\begin{equation}\label{eq:iter_trae_oracle_scales}
A_{n,t}^{\mathrm T}(\lambda_{\mathrm{IT}}^\circ)
=O_t\!\left((s_{n,t}^{\mathrm T})^{q_t/(q_t+1)}\right),\qquad
\frac{A_{n,t}^{\mathrm T}(\lambda_{\mathrm{IT}}^\circ)}{\lambda_{\mathrm{IT}}^\circ}
=O_t\!\left((s_{n,t}^{\mathrm T})^{(q_t-1)/(q_t+1)}\right).
\end{equation}

\begin{assumption}[Iterative TRAE realizability]\label{asp:iter_realizability_trae}
For every $n$ for which $\lambda_{\mathrm{IT}}^\circ$ is defined,
$h_{j,\lambda}^*\in\mathcal H$ for $j=1,\ldots,t$ and
$\lambda\in\Lambda_n\cap[\lambda_{\mathrm{IT}}^\circ,2]$.
\end{assumption}
On this grid segment, the discrepancy bound in
\cref{cor:iter_trae_discrepancy_bound} gives
\begin{equation}\label{eq:iter_trae_discrepancy_bound}
L_n(\hat h_{t,\lambda})
\leqslant
A_{n,t}^{\mathrm T}(\lambda)
+C_{\mathrm{IT}}\|w_0\|^2\lambda^{q_t},
\qquad
\lambda\in\Lambda_n\cap[\lambda_{\mathrm{IT}}^\circ,2].
\end{equation}
Every grid point in this segment satisfies $\lambda\geqslant\sqrt{s_{n,t}^{\mathrm T}}$.
\Cref{thm:common_dp} therefore applies once the TRAE loss comparison gives
the weak- and strong-error bounds at every grid point on that segment.

The corresponding minimum sample size for the iterative rule is
\[
N_{\mathrm{IT}}(B) :=\max\left\{2,\left\lceil 8B^2\right\rceil\right\},
\]
which equals $8$ under the normalization $B=1$.

\begin{proposition}\label{prop:exist_lambda_iter_trae}
Suppose the assumptions of \cref{prop:exist_lambda} hold. For every $n\geqslant N_{\mathrm{IT}}(B)$, \cref{alg:discrepancy}, applied to $\hat h_{t,\lambda}$ with the noise level \eqref{eq:iter_trae_scale}, accepts a candidate after at most $t(K_n+1)=O(t\log n)$ minimax problems, with probability at least $1-\xi$. The returned parameter is either the initial point or a noninitial point satisfying \eqref{eq:discrepancy_principle} with the iterative loss and threshold.
\end{proposition}

\begin{theorem}[Adaptive iterative TRAE]\label{thm:iter_trae_rate}
Under the assumptions of
\cref{prop:exist_lambda_iter_trae,asp:iter_realizability_trae,asp:iter_mean_sq_continuity}, let
$\lambda_{\mathrm{dp},t}^{\mathrm T}$ denote the selected parameter, with
$s_{n,t}^{\mathrm T}$ and $q_t$ defined above.
For every fixed $t$ and every $n\geqslant N_{\mathrm{IT}}(B)$, with probability at least $1-\xi$,
\begin{align*}
\|\mathcal T(\hat h_{t,\lambda_{\mathrm{dp},t}^{\mathrm T}}-h_0)\|^2
\leqslant O_t\!\left(
(s_{n,t}^{\mathrm T})^{q_t/(q_t+1)}
\right), \quad  \|\hat h_{t,\lambda_{\mathrm{dp},t}^{\mathrm T}}-h_0\|^2
\leqslant O_t\!\left(
(s_{n,t}^{\mathrm T})^{(q_t-1)/(q_t+1)}
\right).
\end{align*}
Because $t$ is fixed, $s_{n,t}^{\mathrm T}=b_t\delta_n^2=O_t(\delta_n^2)$,
so $b_t$ affects only the constants hidden by $O_t(\cdot)$. If
$2t\geqslant\beta+1$, these bounds become
\begin{align*}
\|\hat h_{t,\lambda_{\mathrm{dp},t}^{\mathrm T}}-h_0\|^2
=O_t\!\left(\delta_n^{2\beta/(\beta+2)}\right), \quad
\|\mathcal T(\hat h_{t,\lambda_{\mathrm{dp},t}^{\mathrm T}}-h_0)\|^2
=O_t\!\left(\delta_n^{2(\beta+1)/(\beta+2)}\right).
\end{align*}
\end{theorem}

As for RDIV, a fixed $t$ widens the range of source smoothness that the TRAE rate exploits to $\beta\leqslant2t-1$. The one-step bounds in \cref{thm:tikhonv_rate} remain sharper for the weak error at every $\beta$, and for the strong error when $\beta<2$; iteration improves the strong-error rate only when $\beta>2$.

\subsection{Application to Doubly Robust Estimation}

Debiased inference on $\theta_0$ also requires the dual nuisance function, and the method of \cref{sec:trae} applies to the dual problem unchanged:
\[
\theta_0=\mathbb E[m(W;q_0)],\qquad \mathcal T^*q_0=a_0,
\]
where $a_{0}$ is the Riesz representer of the linear functional $h\mapsto \mathbb{E}\left[\tilde{m}( W;h)\right]$. For a test class $\mathcal S\subseteq\tilde{\mathcal H}$, define the empirical loss
\[
L_{n}^{\text{dual}}( q) =\sup_{s\in\mathcal S}\mathbb{E}_{n}\left[ 2\tilde{m}( W;s) -2q(Z) s(X) -s( X)^{2}\right]
\]
and the adversarial Tikhonov regularization estimator,
\[
\hat{q} =\arg\min_{q\in \mathcal{Q}}\left\{ L_{n}^{\text{dual}}( q) +\lambda_{\mathrm{dp}}^{\text{dual}} \mathbb{E}_n[q(Z)^2]\right\} ,
\]
where we select $\lambda_{\mathrm{dp}}^{\text{dual}}$ with the discrepancy principle. Given estimates of $h_{0}$ and $q_{0}$ from the adaptive Tikhonov method, let $\mathbb E_{n,\mathrm{eval}}$ denote an empirical average over an evaluation sample of size $n$ that is independent of the data used to train and tune both nuisance estimators. We define the DR estimator of $\theta_0$ as
\[
\hat{\theta }( h,q) =\mathbb{E}_{n,\mathrm{eval}}\left[\tilde{m}( W;h) +m( W;q) -h( X) q( Z)\right] .
\]
The next corollary establishes its asymptotic normality.

\begin{corollary} \label{cor:normality}
Suppose that the assumptions of \cref{thm:tikhonv_rate} hold for both $h_0$ and $q_0$ with (potentially different) source conditions $\beta_h,\beta_q>0$. Choose a deterministic sequence $\xi_n\downarrow0$, and run the primal and dual tuning rules at confidence levels $\xi_n/2$. Let a common deterministic sequence $\delta_n$ satisfy \eqref{eq:delta_trae} for both problems at these confidence levels. Assume that, for some $\gamma_h,\gamma_q>0$, the score maps satisfy
\begin{align*}
\mathbb{E}\left[\left(\tilde{m}( W;h_{1}) -\tilde{m}( W;h_{2})\right)^{2}\right] & \leqslant C\| h_{1} -h_{2} \| ^{\gamma_h },\\
\mathbb{E}\left[\left(m( W;q_{1}) -m( W;q_{2})\right)^{2}\right] & \leqslant C\| q_{1} -q_{2} \| ^{\gamma_q }
\end{align*}
for all $h_1,h_2\in\mathcal H$ and $q_1,q_2\in\mathcal Q$. Let $\beta_m=\max\{\beta_h,\beta_q\}$ and assume that $\delta_n$ satisfies
\[
\delta _{n} =o\left( n^{-\alpha _{\mathrm{dr}}}\right) ,\qquad
\alpha _{\mathrm{dr}} :=\frac{1+\min\{\beta _{m} ,1\}}{2+4\min\{\beta _{m} ,1\}} .
\]
Let $\hat h,\hat q$ be the adaptive adversarial Tikhonov estimators, trained and tuned independently of the evaluation sample used in $\hat\theta(\hat h,\hat q)$. Finally, assume that $0<\sigma_0^2:=\mathbb E[\psi_0(W)^2]<\infty$. Then,
\[
\sqrt{n}(\hat{\theta }(\hat{h},\hat{q}) -\theta _{0}) =\frac{1}{\sqrt{n}}\sum _{i=1}^{n} \psi _{0}( W_i) +o_{p}( 1) ,
\]
where $\psi _{0}( W) =\tilde{m}( W;h_{0}) +m( W;q_{0}) -q_{0}( Z) h_{0}( X) -\theta _{0}$. Consequently, $\sqrt n(\hat\theta(\hat h,\hat q)-\theta_0)\rightsquigarrow N(0,\sigma_0^2)$.
\end{corollary}
Taking $\xi_n=1/\log(en)$ drives the simultaneous probability of the nuisance bounds to one at a cost of only $O(\log\log(en))$ in the logarithmic confidence term. The sufficient condition for asymptotic normality depends only on $\beta_m$, the larger of the primal and dual source parameters, and neither parameter needs to be known. For $\beta_m\geqslant1$ the requirement becomes $\delta_n=o(n^{-1/3})$, which some nonparametric classes meet; for every fixed $\beta_m>0$ the required exponent stays below $1/2$ and approaches $1/2$ as $\beta_m\downarrow0$.

\section{Experiments}
\label{sec:experiments}

The guarantees in \cref{sec:dp_regularization} compare the selected estimator with a regularization level that depends on the unknown source parameter. A simulation can test this comparison only when the best regularization level changes from one problem to another. We therefore vary the instrument strength, the noise level, the sample size, and the shape of the structural function, and we compare the discrepancy rule with the best fixed value of $\lambda$ in each setting. In this section we call a fixed value of $\lambda$ a penalty. We use kernel hypothesis and critic classes. For these classes each candidate $\hat h_\lambda$ is the exact penalized solution, so the comparison separates the choice of $\lambda$ from optimization error. \Cref{sec:experiment_details} gives the implementation details.

\subsection{Design}
\label{sec:experiment_design}

In the linear Gaussian design, $Z$, $V$, and $\varepsilon$ are independent standard normal variables, and
\[
X=\rho Z+\sqrt{1-\rho^2}\,V,
\qquad
Y=h_s(X)+aV+\sigma\varepsilon,
\]
with $a=0.5$, where $h_s$ is a Gaussian bump of width $s$ scaled to unit variance. In this section, $\rho$ denotes the first-stage correlation rather than the search factor of \cref{alg:discrepancy}. The disturbance $aV+\sigma\varepsilon$ is correlated with $X$ and has mean zero given $Z$. In the basis of normalized Hermite polynomials, the conditional expectation operator is diagonal with singular values $\rho^k$, $k\geqslant0$, so a smaller $\rho$ gives a more ill-posed problem. \Cref{sec:experiment_details} shows that $h_s$ satisfies \cref{asp:source_condition} if and only if
\[
\beta<\beta_{\max}(s,\rho):=\frac{\log(1+s^2)}{-2\log\rho}.
\]
Over the regimes below, $\beta_{\max}$ ranges from $0.46$ to $17.2$, so the source parameter lies on both sides of one. The nonlinear design adds a binary treatment, a nonlinear first stage, and heteroskedastic noise, and its operator is again injective.

We cross $\rho\in\{0.65,0.8,0.95\}$, $\sigma\in\{0.1,1,3\}$, and $n\in\{500,\allowbreak1000,\allowbreak2000,\allowbreak4000,\allowbreak8000\}$. For TRAE, the linear design also crosses $s\in\{0.7,2.2\}$, which gives $90$ regimes; the comparison across the two designs uses $s=0.7$, so each design has $45$ regimes. Each regime uses $50$ data sets for TRAE and $20$ for RDIV, and all comparisons between methods are paired.

\subsection{Methods and Comparison}
\label{sec:experiment_estimators}

Both estimators use a Gaussian kernel with a constant term. TRAE uses an RKHS critic, and RDIV estimates the conditional density of the regressor with a mixture density network. The search uses the grid of \cref{alg:discrepancy} with $\lambda_0=2$ and search factor $1/2$, so that $\lambda_k=2^{1-k}$. On every data set we fit every grid point once, and the discrepancy rule and every fixed penalty are evaluated on these same fits.

The rule needs a noise level $A_n$, which the theory specifies only up to constants: $\delta_n$ in \eqref{eq:delta_trae} is a critical-radius bound multiplied by unspecified constants. We compare two practical choices. The fixed-scale rule uses the rate in \eqref{eq:delta_trae} with all constants set to one, for $\delta_n$ and, in RDIV, for $\nu_n$. The noise-floor rule instead estimates from the data the value that the statistic $D_n(\hat h_\lambda)$ would take by chance alone, if the residual of $\hat h_\lambda$ were pure noise, so that the rule stops at the first candidate whose residual looks like noise. \Cref{sec:experiment_details} gives both choices in detail. We chose the noise-floor rule as the primary rule on pilot data sets, before drawing the main ones. Because its $A_n$ is estimated from the data, the theorems, which take $A_n$ as given, do not cover it directly.

For a method and a regime, let $R$ be the mean squared error of the method over the data sets of the regime, and let $R^\ast$ be the smallest such value among the fixed grid penalties in this regime. We report the risk ratio $R/R^\ast$. A ratio of one means that the method does as well as the best penalty for this regime, chosen with knowledge of the truth. In practice a fixed penalty is set once and used for every problem. We therefore also report the best single fixed penalty, the grid value with the smallest geometric mean ratio over the regimes. It is chosen after all results are known, which favors the fixed approach. We compute the strong, weak, and target errors by numerical integration against the analytic structural function and operator.

\subsection{Results}
\label{sec:experiment_results}

Across the $90$ TRAE regimes of the linear design, the best fixed penalty for the strong error ranges from $2$ to $2^{-9}$, and \cref{fig:kernel_tracking} shows that the selected penalty moves with it; the correlation between their logarithms is $0.90$. No single penalty works in every regime. Each fixed grid value is more than $8$ times worse than the best penalty in some regime, while the discrepancy rule is never more than $1.83$ times worse. On the same data sets, the rule has lower strong error than the best single fixed penalty $\lambda=2^{-4}$ in $58$ of the $90$ regimes. The paired difference exceeds two standard errors in $43$ regimes in favor of the rule and in $12$ against it. The fixed-scale rule is somewhat less accurate, with geometric mean strong-error ratio $1.24$ against $1.15$ for the noise-floor rule. \Cref{fig:kernel_curves} shows the strong error against $n$ for $s=0.7$. The adaptive curve stays close to the best fixed penalty of each regime, while each fixed penalty is accurate only for some $(\rho,\sigma,n)$: $\lambda=2^{-1}$ at small $n$ and high noise, and $\lambda=2^{-9}$ at large $n$ and low noise. \Cref{sec:experiment_more_results} gives the full summary.

\begin{figure}[!htp]
\centering
\includegraphics[width=0.9\linewidth]{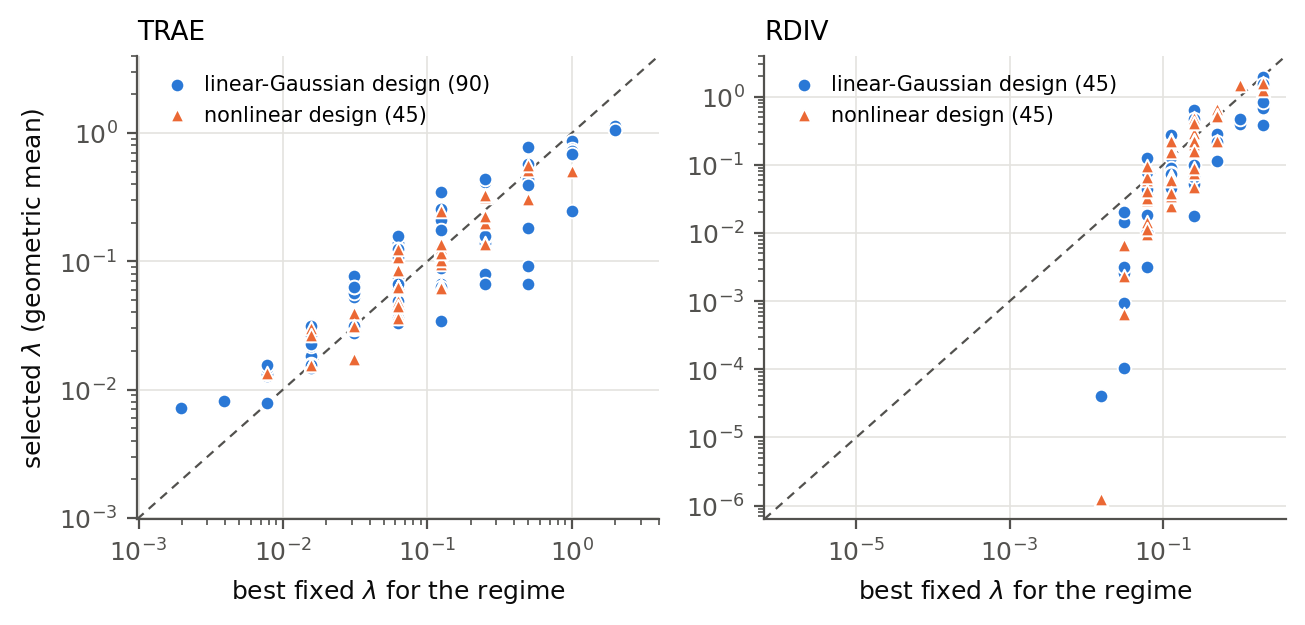}
\caption{Selected penalty against the best fixed penalty for the strong error, one point per regime. The selected penalty is the geometric mean over the data sets of a regime. TRAE uses $90$ linear and $45$ nonlinear regimes; RDIV uses $45$ regimes of each design. The lowest RDIV points are regimes in which some searches reached the end of the grid. The dashed line is the diagonal.}
\label{fig:kernel_tracking}
\end{figure}

\begin{figure}[!htp]
\centering
\includegraphics[width=\linewidth]{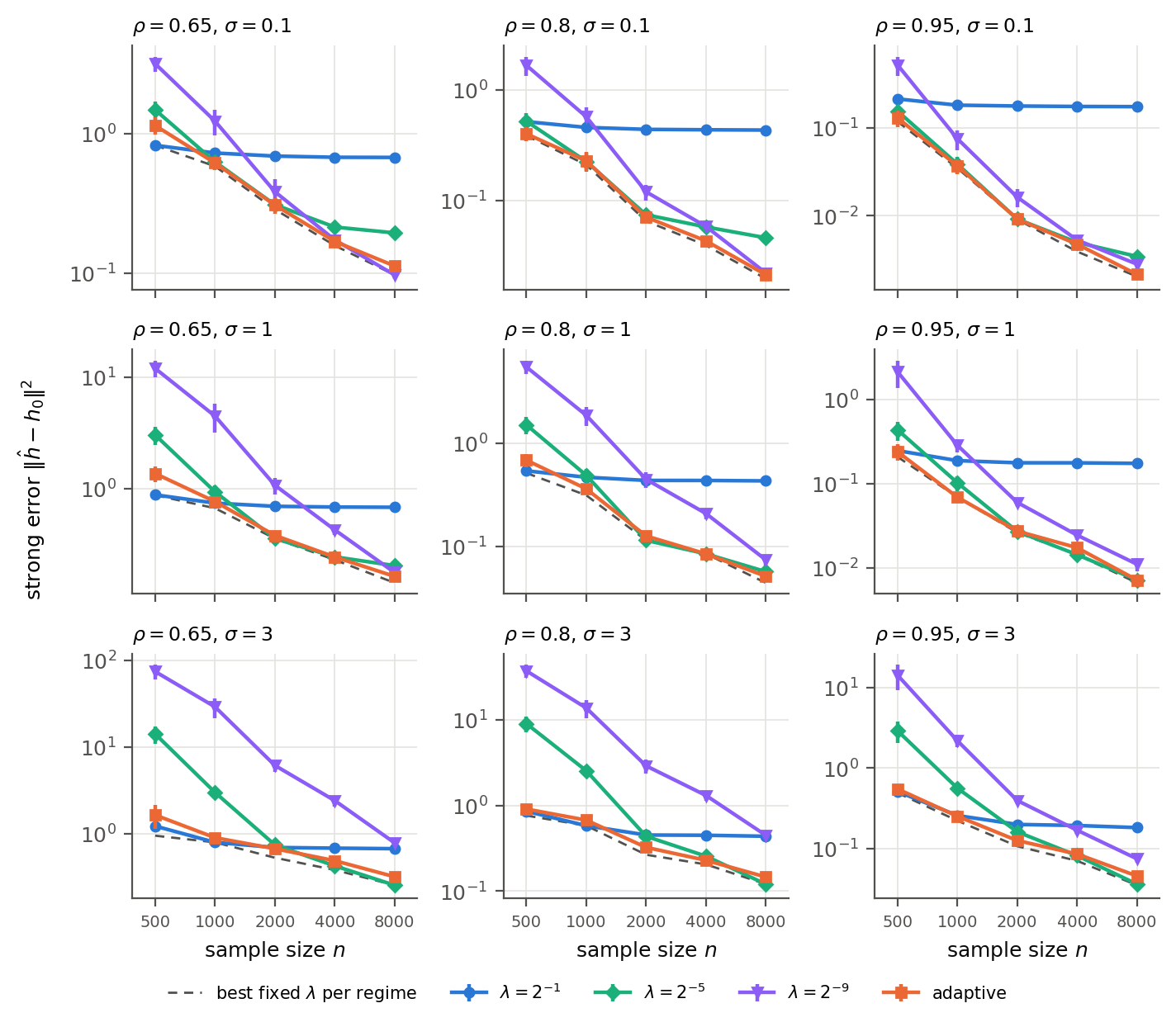}
\caption{Strong error of TRAE against the sample size in the linear design with $s=0.7$, for three fixed penalties and the noise-floor rule. The dashed curve is the best fixed penalty of each regime. Bars show $\pm1$ Monte Carlo standard error over $50$ data sets.}
\label{fig:kernel_curves}
\end{figure}

\Cref{tab:kernel_designs} adds the nonlinear design, RDIV, and the target error. TRAE improves on the best single fixed penalty in both designs and for all three errors. Its worst ratio is at most $2.19$, compared with $4.72$ to $13.43$ for the best single fixed penalty. RDIV improves on the best single fixed penalty in the linear design for all three errors. In the nonlinear design it improves the target error, is close on the strong error, and is worse on the weak error. Its search is also less stable. RDIV fits and calibrates on different blocks with an estimated operator, so its statistic need not decrease along the grid, and its noise floor depends on the estimated operator error.

\begin{table}[!htp]
\centering
\caption{Geometric mean risk ratio over $45$ regimes per design, with the largest ratio in parentheses. The linear design uses $s=0.7$. ``Fixed'' is the best single fixed penalty, chosen after seeing the results, separately for each estimator, design, and error. The smaller geometric mean in each pair is in bold.}
\label{tab:kernel_designs}
\small
\setlength{\tabcolsep}{2.5pt}
\begin{tabular}{llcccccc}
\toprule
 & & \multicolumn{2}{c}{Strong error} & \multicolumn{2}{c}{Weak error} & \multicolumn{2}{c}{Target error}\\
\cmidrule(lr){3-4}\cmidrule(lr){5-6}\cmidrule(lr){7-8}
Estimator & Design & Adaptive & Fixed & Adaptive & Fixed & Adaptive & Fixed\\
\midrule
TRAE & Linear & $\mathbf{1.16}$ $(1.73)$ & $1.60$ $(7.78)$ & $\mathbf{1.24}$ $(1.90)$ & $1.59$ $(7.75)$ & $\mathbf{1.23}$ $(2.19)$ & $1.66$ $(13.43)$\\
 & Nonlinear & $\mathbf{1.07}$ $(1.30)$ & $1.48$ $(8.36)$ & $\mathbf{1.31}$ $(1.87)$ & $1.41$ $(4.72)$ & $\mathbf{1.12}$ $(1.80)$ & $1.58$ $(9.58)$\\
\addlinespace
RDIV & Linear & $\mathbf{1.34}$ $(3.69)$ & $1.43$ $(8.18)$ & $\mathbf{1.38}$ $(3.53)$ & $1.58$ $(12.93)$ & $\mathbf{1.39}$ $(7.81)$ & $1.56$ $(8.00)$\\
 & Nonlinear & $1.30$ $(2.61)$ & $\mathbf{1.26}$ $(3.57)$ & $1.64$ $(2.62)$ & $\mathbf{1.23}$ $(3.36)$ & $\mathbf{1.22}$ $(2.60)$ & $1.40$ $(7.90)$\\
\bottomrule
\end{tabular}
\end{table}

\section{Conclusion}

We develop a discrepancy rule for choosing the regularization level in conditional moment models when the source smoothness is unknown. The rule sets its threshold from prescribed bounds on statistical error and applies to both RDIV and TRAE. Its oracle guarantee reduces the analysis of any such rule to two ingredients: a discrepancy bound at an oracle grid point, and weak- and strong-error bounds at the grid points above it. We verify these ingredients for the noniterative and iterative versions of RDIV and TRAE, which yields adaptive rates that match known-smoothness comparisons within the smoothness ranges stated in the theorems when the underlying statistical errors are comparable. The threshold still requires valid statistical-error bounds. In simulations with kernel classes, where every candidate is an exact penalized solution, the rule followed the best fixed regularization parameter across settings in which that value changed by three orders of magnitude, and it stayed close to the best value of each setting. These results use a data-driven noise level. A calibration of the noise level that the theory covers, and the role of optimization error for neural network classes, remain open.

\bibliography{reference.bib}

\begin{thebibliography}{40}
\providecommand{\natexlab}[1]{#1}
\providecommand{\url}[1]{\texttt{#1}}
\expandafter\ifx\csname urlstyle\endcsname\relax
  \providecommand{\doi}[1]{doi: #1}\else
  \providecommand{\doi}{doi: \begingroup \urlstyle{rm}\Url}\fi

\bibitem[Ai and Chen(2003)]{aiEfficientEstimationModels2003}
Chunrong Ai and Xiaohong Chen.
\newblock Efficient estimation of models with conditional moment restrictions
  containing unknown functions.
\newblock \emph{Econometrica}, 71\penalty0 (6):\penalty0 1795--1843, 2003.
\newblock \doi{10.1111/1468-0262.00470}.

\bibitem[Bennett and Kallus(2023)]{bennett2023variational}
Andrew Bennett and Nathan Kallus.
\newblock The variational method of moments.
\newblock \emph{Journal of the Royal Statistical Society Series B: Statistical
  Methodology}, 85\penalty0 (3):\penalty0 810--841, 2023.
\newblock \doi{10.1093/jrsssb/qkad025}.

\bibitem[Bennett et~al.(2023)Bennett, Kallus, Mao, Newey, Syrgkanis, and
  Uehara]{bennett2023source}
Andrew Bennett, Nathan Kallus, Xiaojie Mao, Whitney Newey, Vasilis Syrgkanis,
  and Masatoshi Uehara.
\newblock Source condition double robust inference on functionals of inverse
  problems.
\newblock \emph{arXiv preprint arXiv:2307.13793}, 2023.

\bibitem[Bennett et~al.(2026)Bennett, Kallus, Mao, Newey, Syrgkanis, and
  Uehara]{bennettInferenceStronglyIdentified2023a}
Andrew Bennett, Nathan Kallus, Xiaojie Mao, Whitney~K. Newey, Vasilis
  Syrgkanis, and Masatoshi Uehara.
\newblock Inference on strongly identified functionals of weakly identified
  functions.
\newblock \emph{Journal of the Royal Statistical Society Series B: Statistical
  Methodology}, 88\penalty0 (3):\penalty0 998--1028, 2026.
\newblock \doi{10.1093/jrsssb/qkaf075}.

\bibitem[Blanchard et~al.(2018)Blanchard, Hoffmann, and
  Rei{\ss}]{blanchard2018optimal}
Gilles Blanchard, Marc Hoffmann, and Markus Rei{\ss}.
\newblock Optimal adaptation for early stopping in statistical inverse
  problems.
\newblock \emph{SIAM/ASA Journal on Uncertainty Quantification}, 6\penalty0
  (3):\penalty0 1043--1075, 2018.
\newblock \doi{10.1137/17M1154096}.

\bibitem[Breunig and Haan(2021)]{breunig2021nonparametric}
Christoph Breunig and Peter Haan.
\newblock Nonparametric regression with selectively missing covariates.
\newblock \emph{Journal of Econometrics}, 223\penalty0 (1):\penalty0 28--52,
  2021.
\newblock \doi{10.1016/j.jeconom.2020.07.050}.

\bibitem[Breunig and Johannes(2016)]{breunig2016adaptive}
Christoph Breunig and Jan Johannes.
\newblock Adaptive estimation of functionals in nonparametric instrumental
  regression.
\newblock \emph{Econometric Theory}, 32\penalty0 (3):\penalty0 612--654, 2016.
\newblock \doi{10.1017/S0266466614000966}.

\bibitem[Carrasco et~al.(2007)Carrasco, Florens, and
  Renault]{carrascoChapter77Linear2007}
Marine Carrasco, Jean-Pierre Florens, and Eric Renault.
\newblock Linear inverse problems in structural econometrics: Estimation based
  on spectral decomposition and regularization.
\newblock In James~J. Heckman and Edward~E. Leamer, editors, \emph{Handbook of
  Econometrics}, volume~6B, chapter~77, pages 5633--5751. Elsevier, 2007.
\newblock \doi{10.1016/S1573-4412(07)06077-1}.

\bibitem[Celisse and Wahl(2021)]{celisse2021discrepancy}
Alain Celisse and Martin Wahl.
\newblock Analyzing the discrepancy principle for kernelized spectral filter
  learning algorithms.
\newblock \emph{Journal of Machine Learning Research}, 22\penalty0
  (76):\penalty0 1--59, 2021.
\newblock URL \url{https://www.jmlr.org/papers/v22/20-358.html}.

\bibitem[Chen and Pouzo(2012)]{chenEstimationNonparametricConditional2012}
Xiaohong Chen and Demian Pouzo.
\newblock Estimation of nonparametric conditional moment models with possibly
  nonsmooth generalized residuals.
\newblock \emph{Econometrica}, 80\penalty0 (1):\penalty0 277--321, 2012.
\newblock \doi{10.3982/ECTA7888}.

\bibitem[Chen and Reiss(2011)]{chenRATEOPTIMALITYILLPOSED2011}
Xiaohong Chen and Markus Reiss.
\newblock On rate optimality for ill-posed inverse problems in econometrics.
\newblock \emph{Econometric Theory}, 27\penalty0 (3):\penalty0 497--521, June
  2011.
\newblock \doi{10.1017/S0266466610000381}.

\bibitem[Chen et~al.(2025)Chen, Christensen, and Kankanala]{chen2025adaptive}
Xiaohong Chen, Timothy Christensen, and Sid Kankanala.
\newblock Adaptive estimation and uniform confidence bands for nonparametric
  structural functions and elasticities.
\newblock \emph{The Review of Economic Studies}, 92\penalty0 (1):\penalty0
  162--196, 2025.
\newblock \doi{10.1093/restud/rdae025}.

\bibitem[Darolles et~al.(2011)Darolles, Fan, Florens, and
  Renault]{darollesNonparametricInstrumentalRegression2011}
Serge Darolles, Yanqin Fan, Jean-Pierre Florens, and Eric Renault.
\newblock Nonparametric instrumental regression.
\newblock \emph{Econometrica}, 79\penalty0 (5):\penalty0 1541--1565, 2011.
\newblock \doi{10.3982/ECTA6539}.

\bibitem[D'Haultf{\oe}uille(2010)]{d2010new}
Xavier D'Haultf{\oe}uille.
\newblock A new instrumental method for dealing with endogenous selection.
\newblock \emph{Journal of Econometrics}, 154\penalty0 (1):\penalty0 1--15,
  2010.
\newblock \doi{10.1016/j.jeconom.2009.06.005}.

\bibitem[Dikkala et~al.(2020)Dikkala, Lewis, Mackey, and
  Syrgkanis]{dikkala2020minimax}
Nishanth Dikkala, Greg Lewis, Lester Mackey, and Vasilis Syrgkanis.
\newblock Minimax estimation of conditional moment models.
\newblock \emph{Advances in Neural Information Processing Systems},
  33:\penalty0 12248--12262, 2020.

\bibitem[Engl et~al.(1996)Engl, Hanke, and Neubauer]{engl1996regularization}
Heinz~Werner Engl, Martin Hanke, and Andreas Neubauer.
\newblock \emph{Regularization of inverse problems}, volume 375 of
  \emph{Mathematics and Its Applications}.
\newblock Kluwer Academic Publishers, Dordrecht, 1996.

\bibitem[Florens et~al.(2011)Florens, Johannes, and van
  Bellegem]{florensIDENTIFICATIONESTIMATIONPENALIZATION2011}
Jean-Pierre Florens, Jan Johannes, and S{\'e}bastien van Bellegem.
\newblock Identification and estimation by penalization in nonparametric
  instrumental regression.
\newblock \emph{Econometric Theory}, 27\penalty0 (3):\penalty0 472--496, June
  2011.
\newblock \doi{10.1017/S026646661000037X}.

\bibitem[Foster and Syrgkanis(2023)]{foster2023orthogonal}
Dylan~J. Foster and Vasilis Syrgkanis.
\newblock Orthogonal statistical learning.
\newblock \emph{The Annals of Statistics}, 51\penalty0 (3):\penalty0 879--908,
  2023.
\newblock \doi{10.1214/23-AOS2258}.

\bibitem[Hall and Horowitz(2005)]{hallNonparametricMethodsInference2005}
Peter Hall and Joel~L. Horowitz.
\newblock Nonparametric methods for inference in the presence of instrumental
  variables.
\newblock \emph{The Annals of Statistics}, 33\penalty0 (6):\penalty0
  2904--2929, 2005.
\newblock \doi{10.1214/009053605000000714}.

\bibitem[Hansen(1992)]{hansen1992analysis}
Per~Christian Hansen.
\newblock Analysis of discrete ill-posed problems by means of the {L}-curve.
\newblock \emph{SIAM Review}, 34\penalty0 (4):\penalty0 561--580, 1992.
\newblock \doi{10.1137/1034115}.

\bibitem[Harrach et~al.(2020)Harrach, Jahn, and Potthast]{harrach2020beyond}
Bastian Harrach, Tim Jahn, and Roland Potthast.
\newblock Beyond the {Bakushinkii} veto: regularising linear inverse problems
  without knowing the noise distribution.
\newblock \emph{Numerische Mathematik}, 145\penalty0 (3):\penalty0 581--603,
  2020.
\newblock \doi{10.1007/s00211-020-01122-2}.

\bibitem[Hartford et~al.(2017)Hartford, Lewis, Leyton-Brown, and
  Taddy]{hartford2017deep}
Jason Hartford, Greg Lewis, Kevin Leyton-Brown, and Matt Taddy.
\newblock Deep {IV}: A flexible approach for counterfactual prediction.
\newblock In \emph{Proceedings of the 34th International Conference on Machine
  Learning}, volume~70 of \emph{Proceedings of Machine Learning Research},
  pages 1414--1423. PMLR, 2017.

\bibitem[Horowitz(2014{\natexlab{a}})]{horowitz2014adaptive}
Joel~L. Horowitz.
\newblock Adaptive nonparametric instrumental variables estimation: Empirical
  choice of the regularization parameter.
\newblock \emph{Journal of Econometrics}, 180\penalty0 (2):\penalty0 158--173,
  2014{\natexlab{a}}.
\newblock \doi{10.1016/j.jeconom.2014.03.006}.

\bibitem[Horowitz(2014{\natexlab{b}})]{horowitz2014ill}
Joel~L Horowitz.
\newblock Ill-posed inverse problems in economics.
\newblock \emph{Annual Review of Economics}, 6\penalty0 (1):\penalty0 21--51,
  2014{\natexlab{b}}.
\newblock \doi{10.1146/annurev-economics-080213-041213}.

\bibitem[Johannes and Schwarz(2011)]{johannes2011partially}
Jan Johannes and Maik Schwarz.
\newblock Partially adaptive nonparametric instrumental regression by model
  selection.
\newblock \emph{Journal of the Indian Statistical Association}, 49:\penalty0
  149--175, 2011.
\newblock URL \url{https://arxiv.org/abs/1003.3128}.

\bibitem[Lewis and Syrgkanis(2018)]{lewis2018adversarial}
Greg Lewis and Vasilis Syrgkanis.
\newblock Adversarial generalized method of moments.
\newblock \emph{arXiv preprint arXiv:1803.07164}, 2018.

\bibitem[Li et~al.(2023)Li, Miao, and Tchetgen~Tchetgen]{li2023non}
Wei Li, Wang Miao, and Eric~J. Tchetgen~Tchetgen.
\newblock Non-parametric inference about mean functionals of non-ignorable
  non-response data without identifying the joint distribution.
\newblock \emph{Journal of the Royal Statistical Society Series B: Statistical
  Methodology}, 85\penalty0 (3):\penalty0 913--935, 2023.
\newblock \doi{10.1093/jrsssb/qkad047}.

\bibitem[Li et~al.(2024)Li, Lan, Syrgkanis, Wang, and
  Uehara]{li2024regularized}
Zihao Li, Hui Lan, Vasilis Syrgkanis, Mengdi Wang, and Masatoshi Uehara.
\newblock Regularized {DeepIV} with model selection.
\newblock \emph{arXiv preprint arXiv:2403.04236}, 2024.
\newblock URL \url{https://arxiv.org/abs/2403.04236}.

\bibitem[Liao et~al.(2020)Liao, Chen, Yang, Dai, Kolar, and
  Wang]{liao2020provably}
Luofeng Liao, You-Lin Chen, Zhuoran Yang, Bo~Dai, Mladen Kolar, and Zhaoran
  Wang.
\newblock Provably efficient neural estimation of structural equation models:
  An adversarial approach.
\newblock \emph{Advances in Neural Information Processing Systems},
  33:\penalty0 8947--8958, 2020.

\bibitem[Miao et~al.(2024{\natexlab{a}})Miao, Liu, Li, Tchetgen~Tchetgen, and
  Geng]{miao2015identification}
Wang Miao, Lan Liu, Yilin Li, Eric~J. Tchetgen~Tchetgen, and Zhi Geng.
\newblock Identification and semiparametric efficiency theory of nonignorable
  missing data with a shadow variable.
\newblock \emph{ACM/IMS Journal of Data Science}, 1\penalty0 (2):\penalty0
  1--23, 2024{\natexlab{a}}.
\newblock \doi{10.1145/3592389}.
\newblock Article 5.

\bibitem[Miao et~al.(2024{\natexlab{b}})Miao, Shi, Li, and
  Tchetgen~Tchetgen]{miao2018confounding}
Wang Miao, Xu~Shi, Yilin Li, and Eric~J. Tchetgen~Tchetgen.
\newblock A confounding bridge approach for double negative control inference
  on causal effects.
\newblock \emph{Statistical Theory and Related Fields}, 8\penalty0
  (4):\penalty0 262--273, 2024{\natexlab{b}}.
\newblock \doi{10.1080/24754269.2024.2390748}.

\bibitem[Morozov(1966)]{morozov1966solution}
V.~A. Morozov.
\newblock On the solution of functional equations by the method of
  regularization.
\newblock \emph{Doklady Akademii Nauk SSSR}, 167\penalty0 (3):\penalty0
  510--512, 1966.
\newblock URL \url{https://www.mathnet.ru/eng/dan32161}.

\bibitem[Morozov(1984)]{morozov1984methods}
Vladimir~A. Morozov.
\newblock \emph{Methods for Solving Incorrectly Posed Problems}.
\newblock Springer, New York, 1984.
\newblock \doi{10.1007/978-1-4612-5280-1}.

\bibitem[Muandet et~al.(2020)Muandet, Mehrjou, Lee, and Raj]{muandet2020dual}
Krikamol Muandet, Arash Mehrjou, Si~Kai Lee, and Anant Raj.
\newblock Dual instrumental variable regression.
\newblock \emph{Advances in Neural Information Processing Systems},
  33:\penalty0 2710--2721, 2020.

\bibitem[Nair et~al.(2003)Nair, Schock, and Tautenhahn]{nair2003morozov}
M.~Thamban Nair, E.~Schock, and Ulrich Tautenhahn.
\newblock {Morozov}'s discrepancy principle under general source conditions.
\newblock \emph{Zeitschrift f{\"u}r Analysis und ihre Anwendungen}, 22\penalty0
  (1):\penalty0 199--214, 2003.
\newblock \doi{10.4171/ZAA/1140}.

\bibitem[Newey(2013)]{neweyNonparametricInstrumentalVariables2013}
Whitney~K. Newey.
\newblock Nonparametric {{Instrumental Variables Estimation}}.
\newblock \emph{American Economic Review}, 103\penalty0 (3):\penalty0 550--556,
  2013.
\newblock \doi{10.1257/aer.103.3.550}.

\bibitem[Newey and Powell(2003)]{neweyInstrumentalVariableEstimation2003}
Whitney~K. Newey and James~L. Powell.
\newblock Instrumental variable estimation of nonparametric models.
\newblock \emph{Econometrica}, 71\penalty0 (5):\penalty0 1565--1578, 2003.
\newblock \doi{10.1111/1468-0262.00459}.

\bibitem[Singh et~al.(2019)Singh, Sahani, and Gretton]{singh2019kernel}
Rahul Singh, Maneesh Sahani, and Arthur Gretton.
\newblock Kernel instrumental variable regression.
\newblock \emph{Advances in Neural Information Processing Systems}, 32, 2019.

\bibitem[Tchetgen~Tchetgen et~al.(2024)Tchetgen~Tchetgen, Ying, Cui, Shi, and
  Miao]{tchetgen2020introduction}
Eric~J. Tchetgen~Tchetgen, Andrew Ying, Yifan Cui, Xu~Shi, and Wang Miao.
\newblock An introduction to proximal causal inference.
\newblock \emph{Statistical Science}, 39\penalty0 (3):\penalty0 375--390, 2024.
\newblock \doi{10.1214/23-STS911}.

\bibitem[Xu et~al.(2021)Xu, Chen, Srinivasan, de~Freitas, Doucet, and
  Gretton]{xu2021features}
Liyuan Xu, Yutian Chen, Siddarth Srinivasan, Nando de~Freitas, Arnaud Doucet,
  and Arthur Gretton.
\newblock Learning deep features in instrumental variable regression.
\newblock In \emph{International Conference on Learning Representations}, 2021.
\newblock URL \url{https://openreview.net/forum?id=sy4Kg_ZQmS7}.

\end{thebibliography}
\bibliographystyle{plainnat}

\newpage

\appendix 

\section{Proof of the General Oracle Guarantee}\label{sec:proof_common_dp}

\begin{proof}[Proof of \cref{thm:common_dp}]
We work on $\mathcal E_n$. The proof has three steps: the oracle point passes the test, so the rule stops no later than that point, and acceptance together with monotonicity then turns the error bounds into the stated rates.

First, by condition (i),
$D_n(\hat h_{\lambda^\circ})\leqslant\tfrac32 A_n(\lambda^\circ)$.
Thus $\lambda^\circ$ passes the test, so the rule accepts some candidate and the fallback is not reached. Second, since \cref{alg:discrepancy} searches from larger to smaller regularization and stops at the first acceptance, it returns a point
$\lambda_{\mathrm{dp}}\geqslant\lambda^\circ$.
Third, condition (ii) covers every grid point between $\lambda^\circ$ and $2$, so
\eqref{eq:common_accepted_bounds} applies at the selected point. The rule accepted
$\lambda_{\mathrm{dp}}$, so $D_n(\hat h_{\lambda_{\mathrm{dp}}})\leqslant\tfrac32A_n(\lambda_{\mathrm{dp}})$ and
\[
D_n(\hat h_{\lambda_{\mathrm{dp}}})+A_n(\lambda_{\mathrm{dp}})
\leqslant\frac52 A_n(\lambda_{\mathrm{dp}}).
\]
Because $A_n$ is nonincreasing over the grid points in $[\lambda^\circ,2]$ by condition (iii) and $\lambda\mapsto1/\lambda$ decreases, the product $A_n(\lambda)/\lambda$ is nonincreasing there as well, so
\[
A_n(\lambda_{\mathrm{dp}})
\leqslant A_n(\lambda^\circ),
\qquad
\frac{A_n(\lambda_{\mathrm{dp}})}{\lambda_{\mathrm{dp}}}
\leqslant\frac{A_n(\lambda^\circ)}{\lambda^\circ}.
\]
Substituting the last two displays into \eqref{eq:common_accepted_bounds} proves
\eqref{eq:common_rate}.
\end{proof}

The following consequence of the source condition allows us to compare
both one-step estimators directly with $h_0$.

\begin{lemma}[Variational source inequality]\label{lemma:variational_source}
Let $\mathcal A=\mathcal T^*\mathcal T$ and
$b=\min\{\beta,1\}$, so that $p=1+b$ and $\alpha=b/(1+b)$. Under
\cref{asp:source_condition}, every $u\in\tilde{\mathcal H}$ satisfies
\[
|\langle h_0,u\rangle|
\leqslant \|w_0\|\,\|\mathcal Tu\|^b\|u\|^{1-b}.
\]
Consequently, for every fixed $\epsilon\in(0,1)$ and
$\lambda\in(0,2]$,
\begin{align*}
2\lambda|\langle h_0,u\rangle|
&\leqslant
\epsilon\{\|\mathcal Tu\|^2+\lambda\|u\|^2\}
+C_\epsilon\|w_0\|^2\lambda^p,\\
2|\langle h_0,u\rangle|
&\leqslant
\epsilon\|u\|^2
+C_\epsilon\|w_0\|^{2/(1+b)}
\|\mathcal Tu\|^{2b/(1+b)}.
\end{align*}
\end{lemma}

\begin{proof}
For $0<\beta\leqslant1$, spectral interpolation and
$\|\mathcal A\|\leqslant1$ give
\[
|\langle h_0,u\rangle|
=|\langle w_0,\mathcal A^{\beta/2}u\rangle|
\leqslant
\|w_0\|\,\|\mathcal A^{1/2}u\|^\beta\|u\|^{1-\beta}.
\]
For $\beta>1$, the positive self-adjoint operator $\mathcal A$ has spectrum
in $[0,1]$ because $\|\mathcal A\|\leqslant1$. Let $\mu_u$ denote its
spectral measure associated with $u$. Since $s^\beta\leqslant s$ for
$s\in[0,1]$,
\[
\|\mathcal A^{\beta/2}u\|^2
=\int_{[0,1]}s^\beta\,d\mu_u(s)
\leqslant\int_{[0,1]}s\,d\mu_u(s)
=\|\mathcal A^{1/2}u\|^2.
\]
Since $\|\mathcal A^{1/2}u\|=\|\mathcal Tu\|$, these two cases prove
the first display. To obtain the first consequence, set
$E_\lambda=\|\mathcal Tu\|^2+\lambda\|u\|^2$ and observe that
\[
2\lambda|\langle h_0,u\rangle|
\leqslant 2\|w_0\|\lambda^{(1+b)/2}E_\lambda^{1/2}.
\]
The inequality $2uv\leqslant\epsilon u^2+\epsilon^{-1}v^2$ gives the first consequence with $C_\epsilon=\epsilon^{-1}$, uniformly over $b\in(0,1]$. Its coefficient is independent of $\beta$ and $\|w_0\|$. Applying the weighted form
of the same inequality to
$2\|w_0\|\|\mathcal Tu\|^b\|u\|^{1-b}$ gives the second consequence.
\end{proof}

\section[Proofs for RDIV]{Proofs of \cref{sec:rdiv}}

In this section, $p$ and $\alpha$ are the exponents in \eqref{eq:p_alpha}. Throughout the appendix, $R^\lambda(h)=L(h)+\lambda\|h\|^2$ and
$R_n^\lambda(h)=L_n(h)+\lambda\|h\|_{2,n}^2$, where
\[
\|h\|_{2,n}:=\{\mathbb E_n[h(X)^2]\}^{1/2}.
\]
In this section, $\mathbb E_n$ is the average on $I_{\mathrm{fit}}$,
$L(h)=\sigma^2_{Y\mid Z}+\|\mathcal T(h-h_0)\|^2$, and $D_n$ is the projected discrepancy in \eqref{eq:rdiv_projected_loss}. We use the scales in \eqref{eq:delta_deepiv}. Constants may depend on the fixed bounds and population quantities; they are uniform over $n$ and the regularization grid.

\begin{lemma}[Preliminary RDIV nuisance and operator bounds]\label{lemma:t_err_dp}
Under \cref{asp:re_cd,asp:cr_deepiv}, with probability at least $1-\xi$,
\begin{align}
\mathbb E_Z\!\left[H^2\{\hat g(\cdot\mid Z),g_0(\cdot\mid Z)\}\right]
&\leqslant C\delta_{n,\mathcal G}^2,
\label{eq:rdiv_mle_hellinger}\\
\|\hat r-r_0\|_{L_2(Z)}&\leqslant \delta_{n,r},\nonumber\\
\sup_{h\in\mathcal H}
\|(\hat{\mathcal T}-\mathcal T)h\|_{L_2(Z)}
&\leqslant C\delta_{n,\mathcal G},\label{eq:rdiv_op_global}\\
\|(\hat{\mathcal T}-\mathcal T)(h-h')\|_{L_1(Z)}
&\leqslant C\delta_{n,\mathcal G}\|h-h'\|,
\qquad h,h'\in\mathcal H.\label{eq:rdiv_op_relative}
\end{align}
Moreover, if
\[
\widehat L(h):=\mathbb E[(Y-\hat{\mathcal T}h(Z))^2\mid I_{\mathrm{pre}}],
\]
then, uniformly over $h,h'\in\mathcal H$,
\begin{equation}\label{eq:rdiv_first_stage_loss}
|\{\widehat L(h)-\widehat L(h')\}-\{L(h)-L(h')\}| \leqslant C\delta_{n,\mathcal G}\|h-h'\|.
\end{equation}
\end{lemma}
\begin{proof}
Allocate failure probability $\xi/2$ to the conditional-density estimator and
$\xi/2$ to the pilot regression. As elsewhere, evaluate the corresponding
deterministic accuracy bounds at these confidence levels and relabel them
$\delta_{n,\mathcal G}$ and $\delta_{n,r}$.
The conditional-MLE bound of \citet[Theorem~24]{li2024regularized}, applied
with $g_0\in\mathcal G$ and the critical radius of $\mathcal G_0$, gives
\eqref{eq:rdiv_mle_hellinger} with probability at least $1-\xi/2$. The
nuisance-rate assumption gives
$\|\hat r-r_0\|_{L_2(Z)}\leqslant\delta_{n,r}$ with probability at least
$1-\xi/2$. Therefore, without requiring independence between $\hat g$ and
$\hat r$, the union bound shows that both inequalities hold simultaneously
with probability at least $1-\xi$. We work on this intersection for
the rest of the proof. Here
\[
H^2(g_1,g_2)=\int\{\sqrt{g_1(x)}-\sqrt{g_2(x)}\}^2\,\mu(dx).
\]

Let $P_X$ be the true marginal distribution of $X$ and let $\hat P_X$ be the marginal induced by $Z\sim P_Z$ and $X\mid Z\sim\hat g(\cdot\mid Z)$. The density bounds in \cref{asp:re_cd} give $\hat g(x\mid z)/g_0(x\mid z)\leqslant C_G/c_0$ and hence
\[
\frac{d\hat P_X}{dP_X}\leqslant \frac{C_G}{c_0}.
\]
The total-variation--Hellinger inequality and $\|h\|_\infty\leqslant B$ give
\[
\|(\hat{\mathcal T}-\mathcal T)h\|^2
\leqslant B^2\mathbb E_Z
\left[\left\{\int|\hat g(x\mid Z)-g_0(x\mid Z)|\,\mu(dx)\right\}^2\right]
\leqslant C B^2\delta_{n,\mathcal G}^2,
\]
which proves \eqref{eq:rdiv_op_global}. For $u=h-h'$, factor
$\hat g/g_0-1=(\sqrt{\hat g/g_0}-1)(\sqrt{\hat g/g_0}+1)$.
Cauchy--Schwarz and the conditional density-ratio bound yield
\begin{align*}
\|(\hat{\mathcal T}-\mathcal T)u\|_1
&\leqslant
\mathbb E\left[|u(X)|
\left|\sqrt{\frac{\hat g(X\mid Z)}{g_0(X\mid Z)}}-1\right|
\left\{\sqrt{\frac{\hat g(X\mid Z)}{g_0(X\mid Z)}}+1\right\}\right]\\
&\leqslant C\|u\|
\mathbb E_Z\!\left[H^2\{\hat g(\cdot\mid Z),g_0(\cdot\mid Z)\}\right]^{1/2}
\leqslant C\delta_{n,\mathcal G}\|u\|,
\end{align*}
which proves \eqref{eq:rdiv_op_relative}. \citet{li2024regularized} use the same operator argument in their MLE analysis.

Finally, write $e_h=(\hat{\mathcal T}-\mathcal T)h$. Pointwise,
\begin{align*}
&(Y-\hat{\mathcal T}h)^2-(Y-\hat{\mathcal T}h')^2
-\{(Y-\mathcal T h)^2-(Y-\mathcal T h')^2\}\\
&\quad=-(\hat{\mathcal T}-\mathcal T)u
(2Y-\hat{\mathcal T}h-\hat{\mathcal T}h')
+\mathcal T u(e_h+e_{h'}).
\end{align*}
The expectation of the first term is at most
$4B\|(\hat{\mathcal T}-\mathcal T)u\|_1$; the second is at most
$\|\mathcal T u\|(\|e_h\|+\|e_{h'}\|)$. Apply
\eqref{eq:rdiv_op_global}, \eqref{eq:rdiv_op_relative}, and conditional Jensen to obtain \eqref{eq:rdiv_first_stage_loss}.
\end{proof}

\begin{lemma}[Uniform RDIV empirical comparisons]\label{lemma:err_l_dp}
Under \cref{asp:range,asp:closeness_deepiv,asp:re_cd,asp:cr_deepiv}, there is an event $\mathcal E_n^{\mathrm R}$ with probability at least $1-\xi$ on which, simultaneously for every $\lambda\in\Lambda_n$ and $h\in\mathcal H$,
\begin{align}
&\left|\{\widehat R^\lambda(h)-\widehat R^\lambda(h_0)\}
-\{R_n^\lambda(h)-R_n^\lambda(h_0)\}\right|
\leqslant C\{\delta_{n,\mathcal H}\|h-h_0\|
+\delta_{n,\mathcal H}^2\},\label{eq:rdiv_fit_comparison}\\
&\left|D_n(\hat h_\lambda)
-\|\mathcal T(\hat h_\lambda-h_0)\|^2\right|
\leqslant C\nu_n\|\mathcal T(\hat h_\lambda-h_0)\|+C\nu_n^2,
\label{eq:rdiv_projected_comparison}
\end{align}
where $\widehat R^\lambda(h)=\widehat L(h)+\lambda\|h\|^2$.
\end{lemma}
\begin{proof}
First condition on $I_{\mathrm{pre}}$, so that $\hat{\mathcal T}$ is fixed and the expected critical-radius bounds apply to the independent fitting sample. For each fixed $\lambda\in\Lambda_n$, apply the localized concentration inequality in \citet[Lemma~14]{foster2023orthogonal} to the two-coordinate class
\[
h\longmapsto\{\hat{\mathcal T}h(Z),h(X)\}
\]
at the fixed comparator $h_0$ and to the loss
$\phi_\lambda(y,u,v)=(y-u)^2+\lambda v^2$. This loss is uniformly Lipschitz on the bounded range because $\lambda\leqslant2$. The critical-radius bounds in \cref{asp:cr_deepiv} give
\[
\left|\{\widehat R^\lambda(h)-\widehat R^\lambda(h_0)\}
-\{R_n^\lambda(h)-R_n^\lambda(h_0)\}\right|
\leqslant C\left[\delta_{n,\mathcal H}
\{\|\hat{\mathcal T}(h-h_0)\|+\|h-h_0\|\}
+\delta_{n,\mathcal H}^2\right].
\]
Conditional Jensen and the density-ratio bound in the proof of \cref{lemma:t_err_dp} give
$\|\hat{\mathcal T}u\|\leqslant(C_G/c_0)^{1/2}\|u\|$.
Thus, a union bound over the $K_n+1$ deterministic grid points gives \eqref{eq:rdiv_fit_comparison}. The definition of $\eta_n$ makes its total failure probability at most $\xi$.

For calibration, condition on $I_{\mathrm{pre}}$ and $I_{\mathrm{fit}}$, and set
\[
q_\lambda=\hat r-\hat{\mathcal T}\hat h_\lambda, \qquad u_\lambda=r_0-\mathcal T\hat h_\lambda.
\]
The functions $q_\lambda$ are fixed under this conditioning and bounded by $2B$. Bernstein's inequality and a union bound over $\Lambda_n$ give, with conditional failure probability at most $\xi$,
\[
|D_n(\hat h_\lambda)-\|q_\lambda\|^2|
\leqslant C\{\delta_{n,\xi}\|q_\lambda\|+\delta_{n,\xi}^2\},
\qquad \lambda\in\Lambda_n.
\]
The simultaneous preliminary-sample bounds in \cref{lemma:t_err_dp} give
\[
\|q_\lambda-u_\lambda\|
\leqslant\|\hat r-r_0\|
+\|(\hat{\mathcal T}-\mathcal T)\hat h_\lambda\|
\leqslant \delta_{n,r}+C\delta_{n,\mathcal G}.
\]
The identity $|\|q_\lambda\|^2-\|u_\lambda\|^2|\leqslant
2\|u_\lambda\|\|q_\lambda-u_\lambda\|+\|q_\lambda-u_\lambda\|^2$, together with
$\|q_\lambda\|\leqslant\|u_\lambda\|+\delta_{n,r}+C\delta_{n,\mathcal G}$, proves \eqref{eq:rdiv_projected_comparison}. Combining the preliminary, fitting, and calibration events under the confidence convention gives the stated probability. All remaining RDIV arguments take place on $\mathcal E_n^{\mathrm R}$.
\end{proof}

\begin{lemma}[RDIV population--empirical risk comparison]\label{lemma:err_r_dp}
On $\mathcal E_n^{\mathrm R}$, uniformly over $h\in\mathcal H$ and $\lambda\in\Lambda_n$,
\[
\left|
R^\lambda(h)-R^\lambda(h_0)
-\{R_n^\lambda(h)-R_n^\lambda(h_0)\}
\right|
\leqslant C\{\delta_n\|h-h_0\|+\delta_n^2\}.
\]
\end{lemma}
\begin{proof}
Apply \eqref{eq:rdiv_first_stage_loss} and \eqref{eq:rdiv_fit_comparison}. Since
$\delta_n\geqslant\delta_{n,\mathcal G}+\delta_{n,\mathcal H}$, their sum is bounded by the displayed right-hand side. The fitting event already holds simultaneously over the finite grid.
\end{proof}

\begin{lemma}[Direct RDIV control]\label{lemma:str_dp}
On $\mathcal E_n^{\mathrm R}$, for every $\lambda\in\Lambda_n$,
\[
\|\mathcal T(\hat h_\lambda-h_0)\|^2
+\lambda\|\hat h_\lambda-h_0\|^2
\leqslant C\left\{\frac{\delta_n^2}{\lambda}
+\|w_0\|^2\lambda^p\right\}.
\]
\end{lemma}
\begin{proof}
Set $e_\lambda=\hat h_\lambda-h_0$. Since $h_0\in\mathcal H$,
empirical optimality and \cref{lemma:err_r_dp} give
\[
\|\mathcal T e_\lambda\|^2+\lambda\|e_\lambda\|^2
\leqslant
C\{\delta_n\|e_\lambda\|+\delta_n^2\}
+2\lambda|\langle h_0,e_\lambda\rangle|.
\]
Apply the first consequence of \cref{lemma:variational_source} and
$C\delta_n\|e_\lambda\|\leqslant
\lambda\|e_\lambda\|^2/4+C\delta_n^2/\lambda$.
Since $\lambda\leqslant2$, the remaining $\delta_n^2$ term is covered by
$C\delta_n^2/\lambda$. Absorbing the quadratic terms proves the claim.
\end{proof}

\begin{corollary}[RDIV discrepancy bound]\label{cor:rdiv_discrepancy_bound}
Under the assumptions of \cref{thm:deepiv_rate}, the calibrated RDIV scales satisfy \eqref{eq:rdiv_discrepancy_bound} simultaneously for every $\lambda\in\Lambda_n$ on $\mathcal E_n^{\mathrm R}$, for a fixed constant $C_{\mathrm R}>0$.
\end{corollary}
\begin{proof}
Let $\zeta_\lambda=\|\mathcal T(\hat h_\lambda-h_0)\|$. By
\cref{lemma:err_l_dp,lemma:str_dp},
\[
D_n(\hat h_\lambda)
\leqslant \zeta_\lambda^2+C\nu_n\zeta_\lambda+C\nu_n^2
\leqslant a_{\mathrm R}
\left(\frac{\delta_n^2}{\lambda}+\nu_n^2\right)
+C_{\mathrm R}\|w_0\|^2\lambda^p
\]
for fixed constants $a_{\mathrm R},C_{\mathrm R}>0$, where the second inequality follows from Young's inequality. The constant $a_{\mathrm R}$ depends only on the fixed bounds and population constants. Multiplying both $\delta_n$ and $\nu_n$ by $\sqrt{\max\{a_{\mathrm R},1\}}$ and relabeling the resulting upper bounds gives \eqref{eq:rdiv_discrepancy_bound}. The noise coefficients are independent of $\beta$ and $\|w_0\|$; the source dependence remains in the displayed bias term. These fixed multipliers are chosen in advance under the noise-calibration convention in \cref{sec:dp_regularization}.
\end{proof}

\begin{proof}[Proof of \cref{prop:exist_lambda_deepiv}]
Throughout the proof, we work on $\mathcal E_n^{\mathrm R}$. Since the projected residual is bounded by $2B$,
$D_n(\hat h_\lambda)\leqslant4B^2$. At the terminal grid point,
\[
\frac{\delta_n^2}{\lambda_{K_n}}
\geqslant\frac{\delta_{n,\xi}^2}{\lambda_{K_n}}
\geqslant n\log\{(K_n+1)/\xi\}
\geqslant n\log2.
\]
For $n\geqslant N_{\mathrm R}(B)$, its threshold is therefore at least $4B^2$, so the terminal point is accepted and the algorithm stops within $K_n+1=O(\log n)$ fits.

If $h_0=0$, take the source element $w_0=0$. Then \eqref{eq:rdiv_discrepancy_bound} gives
\[
D_n(\hat h_2)\leqslant\delta_n^2/2+\nu_n^2 \leqslant \tfrac32(\delta_n^2/2+\nu_n^2),
\]
so the initial point is accepted. In general, let $k^*$ be the first accepted index. If $k^*=0$, the algorithm returns through the initial-acceptance branch. If $k^*\geqslant1$, its predecessor is rejected, and $\rho\geqslant1/2$ gives \eqref{eq:discrepancy_principle}.
\end{proof}

\begin{lemma}[RDIV error bounds]\label{lemma:rdiv_accepted_bounds}
Under the assumptions of \cref{thm:deepiv_rate}, for every
$\lambda\in\Lambda_n$, on $\mathcal E_n^{\mathrm R}$,
\begin{align*}
\|\mathcal T(\hat h_\lambda-h_0)\|^2
&\leqslant C \{D_n(\hat h_\lambda)+A_n^{\mathrm R}(\lambda)\},\\
\|\hat h_\lambda-h_0\|^2
&\leqslant C\left[\frac{\delta_n^2}{\lambda^2}
+\{D_n(\hat h_\lambda)+A_n^{\mathrm R}(\lambda)\}^\alpha\right],
\end{align*}
where the noise level is defined in \eqref{eq:rdiv_common_scales}.
\end{lemma}
\begin{proof}
Let $\zeta_\lambda=\|\mathcal T(\hat h_\lambda-h_0)\|$. By \eqref{eq:rdiv_projected_comparison},
\[
\zeta_\lambda^2\leqslant D_n(\hat h_\lambda)
+C\nu_n\zeta_\lambda+C\nu_n^2.
\]
Since $\nu_n^2\leqslant A_n^{\mathrm R}(\lambda)$, Young's inequality therefore yields
\[
\zeta_\lambda^2\leqslant C \{D_n(\hat h_\lambda)+A_n^{\mathrm R}(\lambda)\}.
\]
Set $e_\lambda=\hat h_\lambda-h_0$. Empirical optimality and
\cref{lemma:err_r_dp} give
\[
\lambda\{\|\hat h_\lambda\|^2-\|h_0\|^2\}
\leqslant C\{\delta_n\|e_\lambda\|+\delta_n^2\}.
\]
Using $\|\hat h_\lambda\|^2-\|h_0\|^2
=\|e_\lambda\|^2+2\langle h_0,e_\lambda\rangle$ and dividing by
$\lambda$ yield
\[
\|e_\lambda\|^2
\leqslant C\frac{\delta_n}{\lambda}\|e_\lambda\|
+C\frac{\delta_n^2}{\lambda}
+2|\langle h_0,e_\lambda\rangle|.
\]
Apply the second consequence of \cref{lemma:variational_source}, the
weak bound above, and Young's inequality. Since $\lambda\leqslant2$, this gives
\[
\|e_\lambda\|^2
\leqslant C\left[\frac{\delta_n^2}{\lambda^2}
+\{D_n(\hat h_\lambda)+A_n^{\mathrm R}(\lambda)\}^\alpha\right],
\]
which is the second claim.
\end{proof}

\begin{proof}[Proof of \cref{thm:deepiv_rate}]
If $h_0=0$, \cref{prop:exist_lambda_deepiv} selects $\lambda=2$, and \cref{lemma:str_dp} gives both squared errors of order $\delta_n^2$.

Suppose $h_0\ne0$. If $\max\{\delta_n,\nu_n\}>1$, boundedness and the contraction property of $\mathcal T$ give the stated rates directly. We may therefore assume $\max\{\delta_n,\nu_n\}\leqslant1$. The discrepancy bound \eqref{eq:rdiv_discrepancy_bound} gives
\eqref{eq:common_discrepancy_bound} at every grid point with
$A_n=A_n^{\mathrm R}$ and $B_n(\lambda)=C_{\mathrm R}\|w_0\|^2\lambda^p$.
\Cref{lemma:rdiv_accepted_bounds} establishes
\eqref{eq:common_accepted_bounds} at every grid point with
$\gamma=\alpha$, since $D_n\geqslant0$ and
$\delta_n^2/\lambda^2\leqslant A_n^{\mathrm R}(\lambda)/\lambda$.
For all sufficiently large $n$, the grid point
$\lambda_{\mathrm R}^\circ$ defined before
\eqref{eq:rdiv_oracle_scales} exists and satisfies
\eqref{eq:common_oracle_crossing}, so condition (i) of \cref{thm:common_dp} holds by \cref{rem:verify_oracle}. The noise level in
\eqref{eq:rdiv_common_scales} is nonincreasing. Moreover,
$A_n^{\mathrm R}(\lambda_{\mathrm R}^\circ)/\lambda_{\mathrm R}^\circ
=\delta_n^2/(\lambda_{\mathrm R}^\circ)^2+\nu_n^2/\lambda_{\mathrm R}^\circ$, and
$\lambda_{\mathrm R}^\circ\asymp\max\{\delta_n^{2/(p+1)},\nu_n^{2/p}\}$ and $\alpha=(p-1)/p$ give
$\delta_n^2/(\lambda_{\mathrm R}^\circ)^2=O\{\delta_n^{2(p-1)/(p+1)}\}$ and
$\nu_n^2/\lambda_{\mathrm R}^\circ=O(\nu_n^{2\alpha})$, which is the second display in \eqref{eq:rdiv_oracle_scales}.
Apply
\cref{thm:common_dp} and then
\eqref{eq:rdiv_oracle_scales} to obtain the refined weak- and
strong-error bounds, using $\nu_n^{2\alpha}\leqslant\varepsilon_{n,\beta}^\alpha$. Boundedness covers the remaining finitely many $n$, after we enlarge the fixed constants. Finally,
$(p-1)/p=\alpha$ and $\varepsilon_{n,\beta}\geqslant
\delta_n^{2p/(p+1)}$ imply its simpler stated form.
\end{proof}

\subsection{Iterative RDIV}

The population recursion \eqref{eq:iter_population} uses the squared weak
loss for both RDIV and TRAE. The following spectral calculation gives its
bias bounds.

\begin{lemma}[Iterative Tikhonov bias]\label{lemma:iter_bias}
Let $\mathcal A=\mathcal T^*\mathcal T$ and $r_\lambda(u)=\lambda/(u+\lambda)$. For every $j\geqslant1$,
\begin{align}
h_0-h_{j,\lambda}^*&=r_\lambda(\mathcal A)^j h_0
=\mathcal A^{\beta/2}r_\lambda(\mathcal A)^j w_0,
\label{eq:iter_bias_identity}\\
\|h_{j,\lambda}^*-h_0\|^2
&\leqslant C\|w_0\|^2\lambda^{\min\{\beta,2j\}},
\label{eq:iter_strong_bias}\\
\|\mathcal T(h_{j,\lambda}^*-h_0)\|^2
&\leqslant C\|w_0\|^2\lambda^{\min\{\beta+1,2j\}}.
\label{eq:iter_weak_bias}
\end{align}
Moreover, if $z_{j,\lambda}=r_\lambda(\mathcal A)^jw_0$, then $\|z_{j,\lambda}\|\leqslant\|w_0\|$ and
\begin{equation}\label{eq:iter_bias_interpolation}
\|h_{j,\lambda}^*-h_0\|^2
\leqslant
\|w_0\|^{2/(\beta+1)}
\|\mathcal T(h_{j,\lambda}^*-h_0)\|^{2\beta/(\beta+1)}.
\end{equation}
\end{lemma}
\begin{proof}
The normal equation for \eqref{eq:iter_population} is
\[
(\mathcal A+\lambda I)h_{j,\lambda}^* =\mathcal A h_0+\lambda h_{j-1,\lambda}^*.
\]
Subtracting this identity from $(\mathcal A+\lambda I)h_0$ and iterating from $h_{0,\lambda}^*=0$ gives \eqref{eq:iter_bias_identity}. Conditional Jensen gives $\|\mathcal A\|\leqslant1$. For every $\gamma>0$,
\[
\sup_{0\leqslant u\leqslant1}
u^\gamma\left(\frac{\lambda}{u+\lambda}\right)^{2j}
\leqslant C_{\gamma,j}\lambda^{\min\{\gamma,2j\}}.
\]
When $\gamma\leqslant2j$, substitute $u=\lambda v$ and bound $v^\gamma/(1+v)^{2j}$. When $\gamma>2j$, use $u^{\gamma-2j}\leqslant1$ and $u\lambda/(u+\lambda)\leqslant\lambda$. Functional calculus with $\gamma=\beta$ and $\gamma=\beta+1$ proves \eqref{eq:iter_strong_bias} and \eqref{eq:iter_weak_bias}. Finally, H\"older's inequality for the spectral measure of $\mathcal A$, applied to $\mathcal A^{\beta/2}z_{j,\lambda}$, gives \eqref{eq:iter_bias_interpolation}.
\end{proof}

\begin{lemma}[Recursive RDIV stability]\label{lemma:iter_rdiv_stability}
Under the assumptions of \cref{thm:iter_rdiv_rate}, there is an event $\mathcal E_{n,t}^{\mathrm{IR}}$ with probability at least $1-\xi$ on which, simultaneously for $j=1,\ldots,t$ and $\lambda\in\Lambda_n\cap[\lambda_{\mathrm{IR}}^\circ,2]$,
\begin{align}
\lambda\|\hat h_{j,\lambda}-h_{j,\lambda}^*\|^2
+\|\mathcal T(\hat h_{j,\lambda}-h_{j,\lambda}^*)\|^2
&\leqslant a_j\frac{\delta_n^2}{\lambda},
\label{eq:iter_rdiv_stability}\\
\|\hat h_{j,\lambda}-h_{j,\lambda}^*\|^2
&\leqslant a_j\frac{\delta_n^2}{\lambda^2},
\label{eq:iter_rdiv_strong_stability}
\end{align}
where $a_j\geqslant1$ is a deterministic constant that depends only on $j$
and fixed population quantities.
On the same event, the final fitted candidates also satisfy
\begin{equation}\label{eq:iter_rdiv_projected_comparison}
\left|D_n(\hat h_{t,\lambda})
-\|\mathcal T(\hat h_{t,\lambda}-h_0)\|^2\right|
\leqslant C\nu_n\|\mathcal T(\hat h_{t,\lambda}-h_0)\|+C\nu_n^2
\end{equation}
simultaneously over $\lambda\in\Lambda_n\cap[\lambda_{\mathrm{IR}}^\circ,2]$.
\end{lemma}
\begin{proof}
Fix $j$ and $\lambda$, and write
\[
e_j=\hat h_{j,\lambda}-h_{j,\lambda}^*, \qquad e_{j-1}=\hat h_{j-1,\lambda}-h_{j-1,\lambda}^*.
\]
Apply the fitting and first-stage comparisons in
\cref{lemma:t_err_dp,lemma:err_l_dp} with fixed comparator
$h_{j,\lambda}^*$ and penalty center $h_{j-1,\lambda}^*$.
The fitting loss is Lipschitz in the two coordinates
$\{\hat{\mathcal T}h(Z),h(X)\}$, uniformly over these bounded centers.
The critical-radius convention for $\mathcal H$ controls its differences
from the fixed comparator. A union bound over the deterministic
iteration--grid pairs, with probability allocation $\eta_{n,t}$, gives
\begin{align}
&\lambda\|e_j\|^2+\|\mathcal T e_j\|^2\nonumber\\
&\quad\leqslant
R_{n,j}^{\lambda,*}(\hat h_{j,\lambda})
-R_{n,j}^{\lambda,*}(h_{j,\lambda}^*)
+C\{\delta_n\|e_j\|+\delta_n^2\},
\label{eq:iter_rdiv_fixed_center}
\end{align}
where
\[
R_{n,j}^{\lambda,*}(h) =L_n(h)+\lambda\|h-h_{j-1,\lambda}^*\|_{2,n}^2.
\]
Empirical optimality in \eqref{eq:iter_rdiv_estimator} and the identity
\begin{align*}
&\|\hat h_{j,\lambda}-h_{j-1,\lambda}^*\|_{2,n}^2
-\|h_{j,\lambda}^*-h_{j-1,\lambda}^*\|_{2,n}^2\\
&\quad-
\|\hat h_{j,\lambda}-\hat h_{j-1,\lambda}\|_{2,n}^2
+\|h_{j,\lambda}^*-\hat h_{j-1,\lambda}\|_{2,n}^2
=2\langle e_j,e_{j-1}\rangle_n
\end{align*}
show that the empirical difference in \eqref{eq:iter_rdiv_fixed_center} is at most $2\lambda\langle e_j,e_{j-1}\rangle_n$.
Apply the localized squared-loss comparison to
$(h-h_{k,\lambda}^*)^2$, whose comparator and center are fixed.
Young's inequality gives, on the same simultaneous event,
\[
\|e_k\|_{2,n}\leqslant C(\|e_k\|+\delta_n), \qquad k=j-1,j.
\]
Young's inequality yields, for a fixed constant $C>1$ independent of
$n$, $\lambda$, $j$, $\beta$, and $\|w_0\|$,
\[
\lambda\|e_j\|^2+\|\mathcal T e_j\|^2
\leqslant
C\frac{\delta_n^2}{\lambda}+C\lambda\|e_{j-1}\|^2.
\]
Starting from $e_0=0$, choose $a_0=0$ and
$a_j=C(1+a_{j-1})$. Induction proves
\eqref{eq:iter_rdiv_stability}; division by $\lambda$ gives
\eqref{eq:iter_rdiv_strong_stability}. Hence, for each fixed $j$, $a_j$
depends only on $j$ and the fixed constants.

For calibration, condition on the preliminary and fitting samples, so the final candidates are fixed functions. The Bernstein and nuisance-comparison argument in the proof of \cref{lemma:err_l_dp}, followed by a union bound over the oracle grid segment, gives \eqref{eq:iter_rdiv_projected_comparison}. Intersecting this calibration event with the recursive fitting event under the confidence allocation $\eta_{n,t}$ gives $\mathcal E_{n,t}^{\mathrm{IR}}$ with probability at least $1-\xi$.
\end{proof}

\begin{corollary}[Iterative RDIV discrepancy bound]\label{cor:iter_rdiv_discrepancy_bound}
Under the assumptions of \cref{thm:iter_rdiv_rate}, on $\mathcal E_{n,t}^{\mathrm{IR}}$, simultaneously for every $\lambda\in\Lambda_n\cap[\lambda_{\mathrm{IR}}^\circ,2]$,
\[
D_n(\hat h_{t,\lambda})
\leqslant
\frac{a_t\delta_n^2}{\lambda}+\nu_n^2
+C_{\mathrm{IR}}\|w_0\|^2\lambda^{q_t},
\qquad q_t=\min\{\beta+1,2t\}.
\]
\end{corollary}
\begin{proof}
By \cref{lemma:iter_bias,lemma:iter_rdiv_stability} and the squared triangle inequality,
\[
\|\mathcal T(\hat h_{t,\lambda}-h_0)\|^2
\leqslant
C\frac{a_t\delta_n^2}{\lambda}
+C\|w_0\|^2\lambda^{q_t}.
\]
The calibration comparison \eqref{eq:iter_rdiv_projected_comparison} applies to the final fitted candidate at each grid point. Young's inequality absorbs its linear term into the displayed weak error and $\nu_n^2$. Enlarging the fixed constants if necessary gives \eqref{eq:iter_rdiv_discrepancy_bound}.
\end{proof}

\begin{lemma}[Iterative RDIV error bounds]\label{lemma:iter_rdiv_accepted_bounds}
Under the assumptions of \cref{thm:iter_rdiv_rate}, for every
$\lambda\in\Lambda_n\cap[\lambda_{\mathrm{IR}}^\circ,2]$, on
$\mathcal E_{n,t}^{\mathrm{IR}}$,
\begin{align*}
\|\mathcal T(\hat h_{t,\lambda}-h_0)\|^2
&\leqslant C_t \{D_n(\hat h_{t,\lambda})+A_{n,t}^{\mathrm R}(\lambda)\},\\
\|\hat h_{t,\lambda}-h_0\|^2
&\leqslant C_t\left[
\frac{s_{n,t}^{\mathrm R}}{\lambda^2}
+\{D_n(\hat h_{t,\lambda})+A_{n,t}^{\mathrm R}(\lambda)\}^{\beta/(\beta+1)}
\right],
\end{align*}
where the noise level is defined in
\eqref{eq:iter_rdiv_common_scales}.
\end{lemma}
\begin{proof}
Let $\zeta=\|\mathcal T(\hat h_{t,\lambda}-h_0)\|$. By
\eqref{eq:iter_rdiv_projected_comparison},
\[
\zeta^2 \leqslant D_n(\hat h_{t,\lambda}) +C\nu_n\zeta+C\nu_n^2,
\]
and $\nu_n^2\leqslant A_{n,t}^{\mathrm R}(\lambda)$.
Young's inequality proves the first bound. The same bound, the weak part of
\eqref{eq:iter_rdiv_stability}, and the squared triangle inequality give
\[
\|\mathcal T(h_{t,\lambda}^*-h_0)\|^2
\leqslant C_t\{D_n(\hat h_{t,\lambda})+A_{n,t}^{\mathrm R}(\lambda)\}.
\]
Apply \eqref{eq:iter_bias_interpolation} to this bound, combine the result
with \eqref{eq:iter_rdiv_strong_stability}, and use the squared triangle
inequality to obtain the second claim.
\end{proof}

\begin{proof}[Proof of \cref{prop:exist_lambda_iter_rdiv}]
At the terminal point, $D_n(\hat h_{t,\lambda_{K_n}})\leqslant4B^2$, while
\[
\frac{a_t\delta_n^2}{\lambda_{K_n}}
\geqslant\frac{\delta_{n,\xi}^2}{\lambda_{K_n}}
\geqslant n\log2.
\]
Thus the terminal candidate is accepted for $n\geqslant N_{\mathrm R}(B)$. If
$h_0=0$, every population iterate equals $0\in\mathcal H$, so the
discrepancy-bound argument at $\lambda=2$, with $w_0=0$, gives initial acceptance on
the corresponding concentration event. For $h_0\ne0$, the first accepted index is either zero or
has a rejected predecessor in $[\lambda,2\lambda]$. Each candidate requires
$t$ recursive fits, so the total number is at most $t(K_n+1)$.
\end{proof}

\begin{proof}[Proof of \cref{thm:iter_rdiv_rate}]
Write $s=s_{n,t}^{\mathrm R}$ and $q=q_t$. Boundedness gives the result directly when $s>1$ or $\nu_n>1$, so consider
$s\vee\nu_n\leqslant1$. The discrepancy bound
\eqref{eq:iter_rdiv_discrepancy_bound} gives
\eqref{eq:common_discrepancy_bound} with the bounds in
\eqref{eq:iter_rdiv_common_scales} and $B_n(\lambda)=C_{\mathrm{IR}}\|w_0\|^2\lambda^{q}$.
\Cref{lemma:iter_rdiv_accepted_bounds} establishes
\eqref{eq:common_accepted_bounds} on
$\Lambda_n\cap[\lambda_{\mathrm{IR}}^\circ,2]$ with
$\gamma=\beta/(\beta+1)$, since $D_n\geqslant0$ and
$s_{n,t}^{\mathrm R}/\lambda^2\leqslant A_{n,t}^{\mathrm R}(\lambda)/\lambda$.
For all sufficiently large $n$, the point
$\lambda_{\mathrm{IR}}^\circ$ defined in the main text exists and
satisfies \eqref{eq:common_oracle_crossing}, so condition (i) holds by \cref{rem:verify_oracle}. The noise level is
nonincreasing. Applying \cref{thm:common_dp} gives
$\lambda_{\mathrm{dp},t}^{\mathrm R}\geqslant\lambda_{\mathrm{IR}}^\circ$
and the stated weak bound. For the strong bound, use the sharper
method-specific bound in
\cref{lemma:iter_rdiv_accepted_bounds}. Acceptance and the monotonicity of
$A_{n,t}^{\mathrm R}$ and $\lambda\mapsto s_{n,t}^{\mathrm R}/\lambda^2$ give
\[
\begin{aligned}
\|\hat h_{t,\lambda_{\mathrm{dp},t}^{\mathrm R}}-h_0\|^2
&\leqslant C_t\left[
\frac{s_{n,t}^{\mathrm R}}{(\lambda_{\mathrm{dp},t}^{\mathrm R})^2}
+\{D_n(\hat h_{t,\lambda_{\mathrm{dp},t}^{\mathrm R}})
+A_{n,t}^{\mathrm R}(\lambda_{\mathrm{dp},t}^{\mathrm R})\}^{\beta/(\beta+1)}
\right]\\
&\leqslant C_t\left[
\frac{s_{n,t}^{\mathrm R}}{(\lambda_{\mathrm{IR}}^\circ)^2}
+A_{n,t}^{\mathrm R}(\lambda_{\mathrm{IR}}^\circ)^{\beta/(\beta+1)}
\right].
\end{aligned}
\]
Now \eqref{eq:iter_rdiv_oracle_scales} proves the stated strong bound.
Boundedness covers the remaining finitely many $n$, after we enlarge the
fixed constants. When $2t\geqslant\beta+1$, $q=\beta+1$.
Under $\nu_n=O(\delta_n)$ and fixed $t$, the strong bound has order
$\delta_n^{2\beta/(\beta+2)}$, and the weak bound has order
$\delta_n^{2(\beta+1)/(\beta+2)}$.
\end{proof}

\section[Proofs for TRAE]{Proofs of \cref{sec:trae}}

Throughout this section, $L_n$ and $L$ denote the TRAE losses, $D_n=L_n$, $p$ and $\alpha$ are the exponents in \eqref{eq:p_alpha}, and $\delta_n$ is the TRAE scale in \eqref{eq:delta_trae}. The finite grid $\Lambda_n=\{2\rho^k:k=0,\ldots,K_n\}$ is defined in \cref{alg:discrepancy}, and $\mathcal E_n$ is the common event defined after \eqref{eq:delta_trae}. All bounds below hold simultaneously on $\mathcal E_n$.

\begin{lemma}[Uniform TRAE loss comparison]\label{lemma:concentrate_l}
Under \cref{asp:closeness,asp:mean_sq_continuity,asp:critical_radius}, on $\mathcal E_n$, the following inequalities hold simultaneously for all $h\in\mathcal H$:
\[
\frac{1}{2} L( h) -O\left( \delta _{n}^{2}\right) \leqslant L_{n}( h) \leqslant \frac{5}{4} L( h) +O\left( \delta _{n}^{2}\right) .
\]
\end{lemma}
\begin{proof}
Let $f_h=\mathcal T(h_0-h)\in\mathcal F$.
The localized concentration bound used in \citet[Lemma~9]{bennett2023source}
gives
\begin{align*}
L( h) &=\sup_{f\in \mathcal{F}}\mathbb{E}\left[ 2( m( W;f) -f( Z) h( X)) -f^{2}( Z)\right] \\
& =\mathbb{E}\left[ 2( m( W;f_{h}) -f_{h}( Z) h( X)) -f_{h}^{2}( Z)\right]\\
 & \leqslant \mathbb{E}_{n}\left[ 2( m( W;f_{h}) -f_{h}( Z) h( X)) -f_{h}^{2}( Z)\right] +O\left( \delta _{n}^{2} +\delta _{n} \| f_{h} \| \right)\\
 & \leqslant \sup_{f\in \mathcal{F}}\mathbb{E}_{n}\left[ 2( m( W;f) -f( Z) h( X)) -f^{2}( Z)\right] +O\left( \delta _{n}^{2} +\delta _{n} \| f_{h} \| \right)\\
 & \leqslant L_{n}( h) +O\left( \delta _{n}^{2} +\delta _{n} \| f_{h} \| \right) .
\end{align*}
Since $\|f_h\|^2=L(h)$, Young's inequality gives
\[
L( h) \leqslant 2L_{n}( h) +O\left( \delta _{n}^{2}\right) .
\]
Similarly, 
\begin{align*}
L_{n}( h)&  =\sup_{f\in \mathcal{F}}\mathbb{E}_{n}\left[ 2( m( W;f) -f( Z) h( X)) -f^{2}( Z)\right] \\
 & \leqslant \sup_{f\in\mathcal F}\left\{\mathbb E\left[2(m(W;f)-f(Z)h(X))-f(Z)^2\right]+C\left(\delta_n^2+\delta_n\|f\|\right)\right\}\\
 & \leqslant \sup_{f\in\tilde{\mathcal Q}}\mathbb{E}\left[ 2( m( W;f) -f( Z) h( X)) -\frac{4}{5} f^{2}( Z)\right] +O\left( \delta _{n}^{2}\right)\\
 & =\frac{5}{4}\|\mathcal T(h_0-h)\|^2+O\left( \delta _{n}^{2}\right)\\
 & =\frac{5}{4} L( h) +O\left( \delta _{n}^{2}\right) .
\end{align*}
Here Young's inequality bounds $C\delta_n\|f\|$ by $\|f\|^2/5+O(\delta_n^2)$. We then enlarge the supremum from $\mathcal F$ to $\tilde{\mathcal Q}$ and complete the square. The feasible residual $f_h$ gives the lower bound, and the enlarged supremum gives the upper bound under \cref{asp:closeness}.
\end{proof}

\begin{lemma}[Direct TRAE control]\label{lemma:strong_convexity}
Under the assumptions of \cref{thm:tikhonv_rate}, on $\mathcal E_n$, the
following bound holds simultaneously for all $\lambda\in\Lambda_n$:
\[
L_n(\hat h_\lambda)+\frac{\lambda}{2}
\|\hat h_\lambda-h_0\|^2
\leqslant C\{\delta_n^2+\|w_0\|^2\lambda^p\}.
\]
Moreover, for all $\lambda\in\Lambda_n$,
\[
\|\hat h_\lambda-h_0\|^2
\leqslant C\left[\frac{\delta_n^2}{\lambda}
+\{L_n(\hat h_\lambda)+\delta_n^2\}^{\alpha}\right].
\]
\end{lemma}
\begin{proof}
Write $e_\lambda=\hat h_\lambda-h_0$. The localized squared-norm inequality
in \citet[Lemma~14]{foster2023orthogonal} gives, uniformly over
$h\in\mathcal H$,
\[
\left|\{\|h\|_{2,n}^2-\|h_0\|_{2,n}^2\}
-\{\|h\|^2-\|h_0\|^2\}\right|
\leqslant C\{\delta_n\|h-h_0\|+\delta_n^2\}.
\]
The event $\mathcal E_n$ includes this bound and
$L_n(h_0)\leqslant C\delta_n^2$. Empirical optimality relative to
$h_0\in\mathcal H$ therefore gives
\begin{equation}\label{eq:trae_empirical_optimality}
L_n(\hat h_\lambda)
+\lambda\{\|e_\lambda\|^2+2\langle h_0,e_\lambda\rangle\}
\leqslant C\{\delta_n^2+\lambda\delta_n\|e_\lambda\|\}.
\end{equation}

By \cref{lemma:concentrate_l},
$L(\hat h_\lambda)\leqslant2L_n(\hat h_\lambda)+C\delta_n^2$.
Apply the first consequence of \cref{lemma:variational_source} to the
inner-product term in \eqref{eq:trae_empirical_optimality}. Young's inequality also
gives
$C\lambda\delta_n\|e_\lambda\|
\leqslant\lambda\|e_\lambda\|^2/8+C\delta_n^2$.
Absorbing the resulting empirical-loss and squared-norm terms proves the
first bound.

For the second bound, the lower comparison in
\cref{lemma:concentrate_l} gives
\[
L(\hat h_\lambda)\leqslant2L_n(\hat h_\lambda)+C\delta_n^2
\leqslant C \{L_n(\hat h_\lambda)+\delta_n^2\}.
\]
By \cref{asp:closeness_deepiv,asp:closeness}, taking $h=h_0$ gives
$0=\mathcal T(h_0-h_0)\in\mathcal F$. Hence
$L_n(\hat h_\lambda)\geqslant0$. Drop this term
from \eqref{eq:trae_empirical_optimality}, divide by $\lambda$, and apply Young's
inequality to obtain
\[
\|e_\lambda\|^2
\leqslant C\frac{\delta_n^2}{\lambda}
+C|\langle h_0,e_\lambda\rangle|.
\]
The second consequence of \cref{lemma:variational_source} and the
weak bound above complete the proof.
\end{proof}

\begin{corollary}[TRAE discrepancy bound]\label{cor:trae_discrepancy_bound}
Under the assumptions of \cref{thm:tikhonv_rate}, the calibrated TRAE scale satisfies \eqref{eq:trae_discrepancy_bound} simultaneously for every $\lambda\in\Lambda_n$ on $\mathcal E_n$, for a fixed constant $C_{\mathrm T}>0$.
\end{corollary}
\begin{proof}
Because $L(h_0)=0$, \cref{lemma:concentrate_l} gives
$L_n(h_0)\leqslant C_0\delta_n^2$. The first bound in
\cref{lemma:strong_convexity} gives
\[
L_n(\hat h_\lambda)
\leqslant a_{\mathrm T}\delta_n^2
+C_{\mathrm T}\|w_0\|^2\lambda^{p}
\]
for fixed constants $a_{\mathrm T},C_{\mathrm T}>0$. The constants $C_0$ and $a_{\mathrm T}$ depend only on the prescribed boundedness and continuity constants, independently of $\beta$ and $\|w_0\|$. Multiplying $\delta_n$ by $\sqrt{\max\{C_0,a_{\mathrm T},1\}}$ and relabeling the resulting upper bound gives \eqref{eq:trae_discrepancy_bound}.
\end{proof}

The next lemma bounds the weak and strong errors at every grid point.
\begin{lemma}[TRAE error bounds]\label{lemma:trae_accepted_bounds}
Suppose that the assumptions of \cref{thm:tikhonv_rate} hold. For every
$\lambda\in\Lambda_n$, on $\mathcal E_n$,
\begin{align*}
\|\mathcal T(\hat h_\lambda-h_0)\|^2
&\leqslant C \{L_n(\hat h_\lambda)+\delta_n^2\},\\
\|\hat h_\lambda-h_0\|^2
&\leqslant C\left[\frac{\delta_n^2}{\lambda}
+\{L_n(\hat h_\lambda)+\delta_n^2\}^\alpha\right].
\end{align*}
\end{lemma}

\begin{proof}
\Cref{lemma:concentrate_l} gives
\[
\|\mathcal T(\hat h_\lambda-h_0)\|^2=L(\hat h_\lambda)
\leqslant 2L_n(\hat h_\lambda)+O(\delta_n^2)
\leqslant C \{L_n(\hat h_\lambda)+\delta_n^2\}.
\]
The second conclusion of \cref{lemma:strong_convexity} gives
\[
\|\hat h_\lambda-h_0\|^2
\leqslant C\left[\frac{\delta_n^2}{\lambda}
+\{L_n(\hat h_\lambda)+\delta_n^2\}^{\alpha}\right].
\]
\end{proof}
We first show that the rule accepts a candidate, then combine the discrepancy and error bounds over the grid to prove the rates.

\begin{proof}[Proof of \cref{prop:exist_lambda}]
Throughout the proof, we work on $\mathcal E_n$. Empirical optimality and $h_0\in\mathcal H$ give
\[
L_n(\hat h_{\lambda_{K_n}})\leqslant L_n(h_0)+\lambda_{K_n}\|h_0\|_{2,n}^2
\leqslant\delta_n^2+B^2n^{-2}.
\]
Here we used \eqref{eq:trae_discrepancy_bound}, $\lambda_{K_n}\leqslant n^{-2}$, and the bound $|h_0|\leqslant B$. Since $\delta_n^2\geqslant n^{-1}$ and $n\geqslant N_{\mathrm T}(B)$,
\[
B^2n^{-2}\leqslant\tfrac12 n^{-1} \leqslant\tfrac12\delta_n^2.
\]
Thus the terminal point passes the threshold, and the algorithm returns an estimator after at most $K_n+1=O(\log n)$ optimization problems.

Let the selected index be
\[
k^*:=\min\{0\leqslant k\leqslant K_n:L_n(\hat h_{2\rho^k})\leqslant \tfrac32\delta_n^2\},
\]
whose existence follows from the preceding argument. If $k^*=0$, the algorithm returns through the initial-acceptance branch. If $k^*\geqslant1$, its minimality gives
\[
L_n(\hat h_{2\rho^{k^*}}) \leqslant \tfrac32\delta_n^2 < L_n(\hat h_{2\rho^{k^*-1}}).
\]
Because $1/2\leqslant\rho<1$, the companion parameter $2\rho^{k^*-1}=\lambda_{\mathrm{dp}}/\rho$ belongs to $[\lambda_{\mathrm{dp}},2\lambda_{\mathrm{dp}}]$. The argument does not use $h_0\neq0$.
\end{proof}

\Cref{thm:common_dp} now combines the discrepancy bound at the oracle point with the weak- and strong-error bounds at the grid points above it. The same argument is uniform over a fixed source ball $\{(\mathcal T^*\mathcal T)^{\beta/2}w:\|w\|\leqslant \varrho\}$ when the realizability conditions and the concentration event $\mathcal E_n$ hold uniformly over that ball; the stated theorem is pointwise in $h_0$.

\begin{proof}[Proof of \cref{thm:tikhonv_rate}]
Throughout the proof, we work on $\mathcal E_n$.
If $\delta_n>1$, boundedness and the contraction property of $\mathcal T$ give both stated rates directly. Suppose $\delta_n\leqslant1$. The discrepancy bound \eqref{eq:trae_discrepancy_bound} gives \eqref{eq:common_discrepancy_bound} at every grid point with the bounds in \eqref{eq:trae_common_scales} and $B_n(\lambda)=C_{\mathrm T}\|w_0\|^2\lambda^p$. \Cref{lemma:trae_accepted_bounds} establishes \eqref{eq:common_accepted_bounds} at every grid point with $\gamma=(p-1)/p$, since $D_n=L_n\geqslant0$ and $\delta_n^2/\lambda=A_n^{\mathrm T}(\lambda)/\lambda$. For all sufficiently large $n$, the point $\lambda_{\mathrm T}^\circ$ defined in the main text exists and satisfies \eqref{eq:common_oracle_crossing}, so condition (i) holds by \cref{rem:verify_oracle}; the noise level is nonincreasing. Apply \cref{thm:common_dp} and \eqref{eq:trae_oracle_scales}. The strong bound has order $\delta_n^{2(p-1)/p}=\delta_n^{2\alpha}$, which is the stated rate. Boundedness covers the remaining finitely many $n$, after we enlarge the fixed constants.
\end{proof}

\subsection{Iterative TRAE}

For this subsection, $\mathcal E_{n,t}^{\mathrm{IT}}$ denotes the common concentration event formed with probability allocation $\eta_{n,t}=\xi/\{t(K_n+1)\}$. The additional class $\operatorname{star}(\mathcal H-\mathcal H)$ in \cref{asp:critical_radius} controls the random centers in the iterative penalties.

\begin{lemma}[Iterative TRAE propagation]
Under the assumptions of \cref{thm:iter_trae_rate}, there is a nondecreasing deterministic sequence with $b_0=0$ and $b_j\geqslant2$, $j=1,\ldots,t$, where each $b_j$ depends only on $j$ and the prescribed boundedness and mean-square-continuity constants, independently of $\beta$ and $\|w_0\|$. This sequence may be calibrated so that, on $\mathcal E_{n,t}^{\mathrm{IT}}$, simultaneously for $\lambda\in\Lambda_n\cap[\lambda_{\mathrm{IT}}^\circ,2]$,
\begin{align}
\|\hat h_{j,\lambda}-h_{j,\lambda}^*\|^2
&\leqslant C\frac{b_j\delta_n^2}{\lambda^2},
\qquad j=1,\ldots,t,
\label{eq:iter_trae_strong_propagation}\\
L(\hat h_{t,\lambda})
&\leqslant C\left[
\delta_n^2+\min\left\{\lambda,
\frac{b_t\delta_n^2}{\lambda}\right\}
+\|w_0\|^2\lambda^{q_t}
\right].
\label{eq:iter_trae_weak_propagation}
\end{align}
\end{lemma}
\begin{proof}
We adapt the iterative argument of \citet[Theorem~6]{bennett2023source} to the finite grid and give the recursive comparison explicitly. Fix $j$ and $\lambda$, and set
\[
e_j=\|\hat h_{j,\lambda}-h_{j,\lambda}^*\|,
\qquad
u_j=\|\mathcal T(\hat h_{j,\lambda}-h_{j,\lambda}^*)\|,
\qquad
f_j^*=\mathcal T(h_0-h_{j,\lambda}^*).
\]
Also define
\[
M_{\leqslant j} =\max_{1\leqslant\ell\leqslant j}
\left\{\lambda,\|f_\ell^*\|_\infty^2\right\}.
\]
By \cref{asp:iter_realizability_trae,asp:closeness}, $f_\ell^*\in\mathcal F$. The boundedness part of \cref{asp:critical_radius} therefore gives
$M_{\leqslant j}\leqslant\max\{2,B^2\}$ for every grid point under
consideration.

To track the random penalty center, write
\begin{align*}
I_{n,j}
&=\|h_{j,\lambda}^*-\hat h_{j-1,\lambda}\|_{2,n}^2
-\|\hat h_{j,\lambda}-\hat h_{j-1,\lambda}\|_{2,n}^2,\\
I_j
&=\|h_{j,\lambda}^*-h_{j-1,\lambda}^*\|^2
-\|\hat h_{j,\lambda}-h_{j-1,\lambda}^*\|^2.
\end{align*}
The excess-loss comparison of
\citet[Lemma~11]{bennett2023source}, applied at this fixed iteration--grid pair, gives
\begin{equation}\label{eq:iter_trae_biasless}
L(\hat h_{j,\lambda})-L(h_{j,\lambda}^*)
\leqslant \lambda I_{n,j}
+C\left\{
\delta_ne_j\|f_j^*\|_\infty
+\delta_nu_j+\delta_n^2
\right\}.
\end{equation}
This comparison uses the localized bounds for
$\operatorname{star}(\mathcal H\cdotp\mathcal F)$,
$\operatorname{star}(m\circ\mathcal F)$, and
$\operatorname{star}(\mathcal F)$ from \cref{asp:critical_radius}, together
with \cref{asp:iter_mean_sq_continuity}. In
particular, that bound controls
$\mathbb E[m(W;f_{\hat h_{j,\lambda}}-f_j^*)^2]$ by $Cu_j^2$,
where $f_h=\mathcal T(h_0-h)$.
The population normal equation for $h_{j,\lambda}^*$ gives the exact quadratic lower bound
\begin{equation}\label{eq:iter_trae_population_curvature}
\lambda e_j^2+u_j^2 \leqslant L(\hat h_{j,\lambda})-L(h_{j,\lambda}^*)-\lambda I_j.
\end{equation}
The quadratic identity for \eqref{eq:iter_population}, together with
$L(h)=\|\mathcal T(h-h_0)\|^2$ on $\mathcal H$, proves this bound.

Expanding the squared penalties separates the center error as
$2\langle\hat h_{j,\lambda}-h_{j,\lambda}^*,
\hat h_{j-1,\lambda}-h_{j-1,\lambda}^*\rangle_n$.
The localized squared-norm comparison for
$\operatorname{star}(\mathcal H-\mathcal H)$ then yields
\begin{equation}\label{eq:iter_trae_center_comparison}
I_{n,j}-I_j \leqslant C\delta_n(e_j+e_{j-1})+C\delta_n^2+2Ce_je_{j-1}.
\end{equation}
Combining \eqref{eq:iter_trae_biasless}--\eqref{eq:iter_trae_center_comparison}
and applying Young's inequality gives
\[
e_j^2\leqslant C\frac{\delta_n^2M_{\leqslant j}}{\lambda^2}+Ce_{j-1}^2.
\]
All terms proportional to $\delta_n^2$, including those multiplied by $\lambda\leqslant2$, are covered by
$\delta_n^2M_{\leqslant j}/\lambda$ before division by $\lambda$.
Since $e_0=0$ and $M_{\leqslant j}$ is uniformly bounded, enlarge the fixed
constant $C\geqslant2$ if necessary, set $b_0=0$, and recursively choose
$b_j=C(1+b_{j-1})$. Induction gives
\eqref{eq:iter_trae_strong_propagation}. Thus each $b_j$ depends only on $j$
and the fixed constants.

For the weak bound, the ordinary excess-loss comparison
\citep[Lemma~10]{bennett2023source} retains the population bias. With
$v_j=\|\mathcal T(h_{j,\lambda}^*-h_0)\|$, its combination with
\eqref{eq:iter_trae_population_curvature} gives
\[
\lambda e_j^2+u_j^2 \leqslant C\{v_j^2+\delta_nu_j+\delta_n^2\} +\lambda(I_{n,j}-I_j).
\]
Substituting \eqref{eq:iter_trae_center_comparison} and absorbing the linear terms gives
\[
\frac{\lambda}{2}e_j^2+\frac12u_j^2 \leqslant C\left(v_j^2+\delta_n^2+\lambda e_{j-1}^2\right).
\]
At the final iteration, boundedness of $\mathcal H$ and
\eqref{eq:iter_trae_strong_propagation} give the two simultaneous bounds
\[
\lambda e_{t-1}^2\leqslant C\lambda,
\qquad
\lambda e_{t-1}^2
\leqslant C\frac{b_t\delta_n^2}{\lambda}.
\]
Taking their minimum, applying the squared triangle inequality, and using the
weak-bias bound \eqref{eq:iter_weak_bias} prove
\eqref{eq:iter_trae_weak_propagation}. The argument used only
$\lambda\in(0,2]$, with the upper endpoint absorbed into fixed constants.
Each comparison is applied to the fixed iteration--grid pairs in the
oracle-relevant segment.
The definition of $\eta_{n,t}$ and a union bound give the simultaneous event.
\end{proof}

\begin{corollary}[Iterative TRAE discrepancy bound]\label{cor:iter_trae_discrepancy_bound}
Under the assumptions of \cref{thm:iter_trae_rate}, the calibrated iterative TRAE scales satisfy \eqref{eq:iter_trae_discrepancy_bound} simultaneously for all $\lambda\in\Lambda_n\cap[\lambda_{\mathrm{IT}}^\circ,2]$ on $\mathcal E_{n,t}^{\mathrm{IT}}$.
\end{corollary}
\begin{proof}
On the oracle grid segment,
$\lambda\geqslant\sqrt{s_{n,t}^{\mathrm T}}$, so
$A_{n,t}^{\mathrm T}(\lambda)=\delta_n^2+b_t\delta_n^2/\lambda$.
Use the propagated branch $b_t\delta_n^2/\lambda$ in
\eqref{eq:iter_trae_weak_propagation} and apply the upper comparison in
\cref{lemma:concentrate_l}. Calibrate the coefficients of $\delta_n^2$
and $b_t\delta_n^2/\lambda$ into the prescribed scales. The remaining
bias term has coefficient $C_{\mathrm{IT}}\|w_0\|^2$, which proves
\eqref{eq:iter_trae_discrepancy_bound} on this segment.
\end{proof}

The next lemma establishes error bounds at every grid point in the stable branch of the iterative noise level.

\begin{lemma}[Iterative TRAE error bounds]\label{lemma:iter_trae_accepted_stability}
Let $\lambda\in\Lambda_n\cap[\lambda_{\mathrm{IT}}^\circ,2]$ satisfy
$\lambda\geqslant\sqrt{s_{n,t}^{\mathrm T}}$. On
$\mathcal E_{n,t}^{\mathrm{IT}}$,
\begin{align}
&\|\mathcal T(\hat h_{t,\lambda}-h_0)\|^2
\leqslant C_t \{D_n(\hat h_{t,\lambda})+A_{n,t}^{\mathrm T}(\lambda)\},
\label{eq:iter_trae_accepted_total}\\
&\|\hat h_{t,\lambda}-h_0\|^2
\leqslant C_t\left[
\frac{1}{\lambda}\{D_n(\hat h_{t,\lambda})+A_{n,t}^{\mathrm T}(\lambda)\}
+\{D_n(\hat h_{t,\lambda})+A_{n,t}^{\mathrm T}(\lambda)\}^{\beta/(\beta+1)}
\right],
\label{eq:iter_trae_accepted_total_strong}\\
&\lambda\|\hat h_{t,\lambda}-h_{t,\lambda}^*\|^2
+\|\mathcal T(\hat h_{t,\lambda}-h_{t,\lambda}^*)\|^2
\leqslant C_t \{D_n(\hat h_{t,\lambda})+A_{n,t}^{\mathrm T}(\lambda)\},
\label{eq:iter_trae_accepted_variance}\\
&\|\mathcal T(h_{t,\lambda}^*-h_0)\|^2
\leqslant C_t \{D_n(\hat h_{t,\lambda})+A_{n,t}^{\mathrm T}(\lambda)\},
\label{eq:iter_trae_accepted_bias}\\
&\|\hat h_{t,\lambda}-h_{t,\lambda}^*\|^2
\leqslant \frac{C_t}{\lambda} \{D_n(\hat h_{t,\lambda})+A_{n,t}^{\mathrm T}(\lambda)\}.
\label{eq:iter_trae_accepted_strong}
\end{align}
\end{lemma}
\begin{proof}
Set
\[
e=\|\hat h_{t,\lambda}-h_{t,\lambda}^*\|,
\quad
u=\|\mathcal T(\hat h_{t,\lambda}-h_{t,\lambda}^*)\|,
\quad
v=\|\mathcal T(h_{t,\lambda}^*-h_0)\|.
\]
Let
\begin{align*}
Q_t^\lambda(h)
&=L(h)+\lambda\|h-h_{t-1,\lambda}^*\|^2,\\
\widetilde Q_{n,t}^\lambda(h)
&=L_n(h)+\lambda\|h-h_{t-1,\lambda}^*\|_{2,n}^2,\\
Q_{n,t}^\lambda(h)
&=L_n(h)+\lambda\|h-\hat h_{t-1,\lambda}\|_{2,n}^2.
\end{align*}
The loss comparison in \cref{lemma:concentrate_l} and the localized
squared-norm comparison, applied at the fixed comparator
$h_{t,\lambda}^*\in\mathcal H$, give
\begin{align*}
&Q_t^\lambda(\hat h_{t,\lambda})-Q_t^\lambda(h_{t,\lambda}^*)
-\{\widetilde Q_{n,t}^\lambda(\hat h_{t,\lambda})
-\widetilde Q_{n,t}^\lambda(h_{t,\lambda}^*)\}\\
&\qquad\leqslant
L_n(\hat h_{t,\lambda})+\frac14v^2
+C\{\delta_n^2+\lambda\delta_ne\}.
\end{align*}
Empirical optimality gives
$Q_{n,t}^\lambda(\hat h_{t,\lambda})
\leqslant Q_{n,t}^\lambda(h_{t,\lambda}^*)$.
Expanding the two empirical penalties yields
\[
\left|
\{\widetilde Q_{n,t}^\lambda(\hat h_{t,\lambda})
-\widetilde Q_{n,t}^\lambda(h_{t,\lambda}^*)\}
-\{Q_{n,t}^\lambda(\hat h_{t,\lambda})
-Q_{n,t}^\lambda(h_{t,\lambda}^*)\}
\right|
\leqslant2\lambda\|e_t\|_{2,n}\|e_{t-1}\|_{2,n},
\]
where $e_j=\hat h_{j,\lambda}-h_{j,\lambda}^*$ in this display.
The localized norm comparison and \eqref{eq:iter_trae_strong_propagation}
bound the last expression together with the term $C\lambda\delta_ne$
in the fixed-center comparison by
\[
\epsilon\lambda e^2 +C_{\epsilon,t}\left( \delta_n^2+\frac{b_t\delta_n^2}{\lambda}\right)
\]
for any fixed $\epsilon\in(0,1)$, where a single $\epsilon$ covers both terms. Taking $\epsilon=1/4$ and using the
population quadratic identity \eqref{eq:iter_trae_population_curvature} yields
\begin{equation}\label{eq:iter_trae_refined_convexity}
\frac34\lambda e^2+u^2
\leqslant
L_n(\hat h_{t,\lambda})+\frac14v^2
+C_t\left(\delta_n^2+\frac{b_t\delta_n^2}{\lambda}\right).
\end{equation}
On the branch $\lambda\geqslant\sqrt{s_{n,t}^{\mathrm T}}$,
\[
A_{n,t}^{\mathrm T}(\lambda) =\delta_n^2+\frac{s_{n,t}^{\mathrm T}}{\lambda}.
\]
Since $D_n=L_n$, \eqref{eq:iter_trae_refined_convexity} therefore gives
\begin{equation}\label{eq:iter_trae_before_absorb}
\frac34\lambda e^2+u^2 \leqslant \frac14v^2+C_t\{D_n(\hat h_{t,\lambda})+A_{n,t}^{\mathrm T}(\lambda)\}.
\end{equation}
The lower comparison in \cref{lemma:concentrate_l} implies
\[
L(\hat h_{t,\lambda})
\leqslant2L_n(\hat h_{t,\lambda})+C\delta_n^2
\leqslant C_t\{D_n(\hat h_{t,\lambda})+A_{n,t}^{\mathrm T}(\lambda)\}.
\]
This proves \eqref{eq:iter_trae_accepted_total}.
Since $v^2\leqslant2L(\hat h_{t,\lambda})+2u^2$, substituting \eqref{eq:iter_trae_before_absorb} on the right gives
\[
v^2\leqslant C_t\{D_n(\hat h_{t,\lambda})+A_{n,t}^{\mathrm T}(\lambda)\}+\frac12v^2.
\]
Absorb the last term to obtain \eqref{eq:iter_trae_accepted_bias}, and substitute it into \eqref{eq:iter_trae_before_absorb}. Multiplying the resulting bound by $4/3$ gives \eqref{eq:iter_trae_accepted_variance}. Dividing \eqref{eq:iter_trae_accepted_variance} by $\lambda$ proves
\eqref{eq:iter_trae_accepted_strong}. Apply
\eqref{eq:iter_bias_interpolation} to
\eqref{eq:iter_trae_accepted_bias} and combine the result with
\eqref{eq:iter_trae_accepted_strong} by the squared triangle inequality.
This proves \eqref{eq:iter_trae_accepted_total_strong}.
\end{proof}

\begin{proof}[Proof of \cref{prop:exist_lambda_iter_trae}]
Throughout the proof, we work on $\mathcal E_{n,t}^{\mathrm{IT}}$. At the terminal grid point, empirical optimality at the final iteration and $h_0\in\mathcal H$ give
\[
L_n(\hat h_{t,\lambda_{K_n}})
\leqslant L_n(h_0)+\lambda_{K_n}\|h_0-\hat h_{t-1,\lambda_{K_n}}\|_{2,n}^2
\leqslant\delta_n^2+4B^2n^{-2}.
\]
Since $A_{n,t}^{\mathrm T}(\lambda)\geqslant\delta_n^2$, $\delta_n^2\geqslant n^{-1}$, and $n\geqslant N_{\mathrm{IT}}(B)$, the terminal point is accepted. The first accepted index is therefore well defined; it is either zero or has a rejected predecessor in $[\lambda,2\lambda]$. Each grid candidate requires $t$ minimax problems.
\end{proof}

\begin{proof}[Proof of \cref{thm:iter_trae_rate}]
Write $s=s_{n,t}^{\mathrm T}$ and $q=q_t$. Boundedness gives the result directly when $s>1$, so consider $s\leqslant1$. The discrepancy bound \eqref{eq:iter_trae_discrepancy_bound} gives \eqref{eq:common_discrepancy_bound} with the bounds in \eqref{eq:iter_trae_scale} and $B_n(\lambda)=C_{\mathrm{IT}}\|w_0\|^2\lambda^{q}$. \Cref{lemma:iter_trae_accepted_stability} establishes \eqref{eq:common_accepted_bounds} for every $\lambda\in\Lambda_n\cap[\lambda_{\mathrm{IT}}^\circ,2]$ with $\gamma=\beta/(\beta+1)$. For all sufficiently large $n$, the point $\lambda_{\mathrm{IT}}^\circ$ defined in the main text exists and satisfies \eqref{eq:common_oracle_crossing}, so condition (i) holds by \cref{rem:verify_oracle}. Since $\lambda_{\mathrm{IT}}^\circ\geqslant\sqrt{s}$, on this segment,
\[
A_{n,t}^{\mathrm T}(\lambda)=\delta_n^2+\frac{s}{\lambda}.
\]
The noise level is therefore nonincreasing, and $b_t\geqslant2$ with $\lambda\leqslant2$ give $A_{n,t}^{\mathrm T}(\lambda)/\lambda\leqslant 2s/\lambda^2$ on that segment, which yields the second display in \eqref{eq:iter_trae_oracle_scales}. Apply \cref{thm:common_dp} and \eqref{eq:iter_trae_oracle_scales}. The weak bound follows directly, while the source-bias contribution to the strong bound satisfies
\[
(s^{q/(q+1)})^{\beta/(\beta+1)} \leqslant s^{(q-1)/(q+1)}
\]
because $q\leqslant\beta+1$. This proves the two stated rates. Boundedness covers the remaining finitely many $n$, after we enlarge the fixed constants. Finally, $2t\geqslant\beta+1$ gives $q=\beta+1$ and the displayed unsaturated rates.
\end{proof}

\begin{proof}[Proof of \cref{cor:normality}]
Apply \cref{thm:tikhonv_rate} to both nuisance estimators at the
prescribed confidence levels $\xi_n/2$. Their rate bounds then hold on
a common event with probability at least $1-\xi_n\to1$. Because
$\beta_h,\beta_q>0$ and $\delta_n\to0$, these bounds give
\begin{equation}\label{eq:conv_hq}
\| \hat{h} -h_{0} \| =o_{p}( 1) ,\qquad \| \hat{q} -q_{0}\| =o_{p}( 1) .
\end{equation}
The remaining bias condition is
\begin{equation}\label{eq:square_root}
\sqrt{n}\mathbb{E}[(\hat{q}( Z) -q_{0}( Z))(\hat{h}( X) -h_{0}( X))] =o_{p}( 1) .
\end{equation}
Suppose first that $\beta_m=\beta_h\geqslant\beta_q$. By the adjoint identity and the Cauchy--Schwarz inequality,
\begin{align*}
\left|\mathbb{E}[(\hat{q}( Z) -q_{0}( Z))(\hat{h}( X) -h_{0}( X))]\right| & \leqslant \|\hat h-h_0\|\,\|\mathcal T^*(\hat q-q_0)\|\\
 & =O_{p}\left( \delta _{n}^{\min\{\beta _{m} ,1\} /( 1+\min\{\beta _{m} ,1\})} \cdot \delta _{n}\right)\\
 & =O_{p}\left( \delta _{n}^{1+\min\{\beta _{m} ,1\} /( 1+\min\{\beta _{m} ,1\})}\right) .
\end{align*}
If $\beta_m=\beta_q\geqslant\beta_h$, the symmetric bound
\[
\left|\mathbb{E}[(\hat{q}( Z) -q_{0}( Z))(\hat{h}( X) -h_{0}( X))]\right|
\leqslant \|\hat q-q_0\|\,\|\mathcal T(\hat h-h_0)\|
\]
gives the same rate. By assumption, $\delta _{n} =o\left( n^{-( 1+\min\{\beta _{m} ,1\}) /( 2+4\min\{\beta _{m} ,1\})}\right)$, and hence
\[
\left|\mathbb{E}[(\hat{q}( Z) -q_{0}( Z))(\hat{h}( X) -h_{0}( X))]\right| =o_{p}\left( n^{-1/2}\right),
\]
which proves \eqref{eq:square_root}.

For completeness, define the uncentered score
\[
S(W;h,q)=\tilde m(W;h)+m(W;q)-h(X)q(Z),
\qquad
d_n(W)=S(W;\hat h,\hat q)-S(W;h_0,q_0).
\]
Conditional on the nuisance-training data, the evaluation observations
are independent and identically distributed. Boundedness of the nuisance
classes and the score continuity assumptions imply
\[
\mathbb E[d_n(W)^2\mid\text{training}]
\leqslant C\left\{
\|\hat h-h_0\|^{\gamma_h}+\|\hat q-q_0\|^{\gamma_q}
+\|\hat h-h_0\|^2+\|\hat q-q_0\|^2
\right\}=o_p(1).
\]
Conditional Chebyshev's inequality therefore gives
$\sqrt n(\mathbb E_{n,\mathrm{eval}}-\mathbb E)d_n=o_p(1)$,
where $\mathbb E d_n$ integrates a fresh observation conditional on
the training data. The primal and dual moment equations give
\[
\mathbb E d_n
=-\mathbb E[(\hat h(X)-h_0(X))(\hat q(Z)-q_0(Z))].
\]
Combining this identity with \eqref{eq:square_root} yields
\[
\sqrt n\{\hat\theta(\hat h,\hat q)-\theta_0\}
=\frac1{\sqrt n}\sum_{i=1}^n\psi_0(W_i)+o_p(1).
\]
The moment equations also imply $\mathbb E\psi_0(W)=0$. The
classical central limit theorem applies because
$0<\mathbb E[\psi_0(W)^2]<\infty$, proving the stated normal limit.
\end{proof}

\section{Additional Experiment Details}\label{sec:experiment_details}

\paragraph{Designs.}
In the linear design, the structural function is a Gaussian bump of width $s$, centered and scaled to unit variance under $X\sim N(0,1)$:
\[
h_s(x)=\frac{\exp\{-x^2/(2s^2)\}-m_s}{d_s},
\qquad
m_s=\frac{s}{\sqrt{s^2+1}},
\qquad
d_s^2=\frac{s}{\sqrt{s^2+2}}-\frac{s^2}{s^2+1}.
\]
Since the disturbance $aV+\sigma\varepsilon$ has mean zero given $Z$, $h_s$ solves $\mathbb E[Y-h(X)\mid Z]=0$. The share of $\operatorname{Var}(Y)$ due to noise grows from about $20\%$ at $\sigma=0.1$ to about $90\%$ at $\sigma=3$. In the nonlinear design, with $Z$, $V$, and $\varepsilon$ as in the linear design, let $A\mid Z\sim\operatorname{Bernoulli}\{\operatorname{expit}(Z)\}$ and
\[
X=\rho\{Z+0.6\sin(1.5Z)\}+0.5(A-\tfrac12)+\sqrt{1-\rho^2}\,V,
\qquad
Y=g(X)+A\tau(X)+aV+\sigma\kappa(Z)\varepsilon,
\]
where $g(x)=\tanh(x)+0.8\exp\{-(x-0.5)^2/0.5\}$, $\tau(x)=1+0.5\sin x$, and $\kappa(z)^2=0.5+\operatorname{expit}(z)$. The hypothesis takes $(A,X)$ and the instrument is $(A,Z)$, so that $h_0(a,x)=g(x)+a\tau(x)$. Given $A$, the first stage is a Gaussian location family whose location is increasing in $z$ with range $\mathbb R$, so the operator is injective. The RDIV runs use the first $20$ of the $50$ data sets of each regime. The data sets are nested in $n$ and share random numbers across the other factors, so all comparisons between methods and between regimes are paired.

\paragraph{Source condition in the linear design.}
Let $\phi_k=\mathrm{He}_k/\sqrt{k!}$ denote the normalized probabilists' Hermite polynomials, an orthonormal basis of $L_2(N(0,1))$. Since $(X,Z)$ is a standard bivariate normal pair with correlation $\rho$, Mehler's formula gives $\mathbb E[\phi_k(X)\mid Z]=\rho^k\phi_k(Z)$, so $\mathcal T^*\mathcal T\phi_k=\rho^{2k}\phi_k$. The function $h_s$ is even and centered, so $\langle h_s,\phi_0\rangle=0$ and $\langle h_s,\phi_{2k+1}\rangle=0$. For $k\geqslant1$, a direct calculation gives
\[
\langle h_s,\phi_{2k}\rangle=\frac{m_s}{d_s}\,\frac{(-1)^k\sqrt{(2k)!}}{2^k\,k!\,(1+s^2)^k}.
\]
The source condition $h_s=(\mathcal T^*\mathcal T)^{\beta/2}w_0$ holds if and only if
\[
\|w_0\|^2=\sum_{k\geqslant1}\frac{\langle h_s,\phi_{2k}\rangle^2}{\rho^{4k\beta}}
=\frac{m_s^2}{d_s^2}\sum_{k\geqslant1}\binom{2k}{k}\frac{q^k}{4^k}
=\frac{m_s^2}{d_s^2}\left\{(1-q)^{-1/2}-1\right\},
\qquad
q=\frac{\rho^{-4\beta}}{(1+s^2)^2},
\]
is finite, that is, if and only if $q<1$. This is the condition $\beta<\beta_{\max}(s,\rho)$ in \cref{sec:experiment_design}. \Cref{tab:experiment_beta} gives its values.

\begin{table}[!htp]
\centering
\caption{Upper end $\beta_{\max}(s,\rho)$ of the source exponents satisfied by $h_s$ in the linear design. The source condition holds for every $\beta<\beta_{\max}$ and fails at $\beta_{\max}$.}
\label{tab:experiment_beta}
\begin{tabular}{lccc}
\toprule
 & $\rho=0.65$ & $\rho=0.8$ & $\rho=0.95$\\
\midrule
$s=0.7$ & $0.46$ & $0.89$ & $3.89$\\
$s=2.2$ & $2.05$ & $3.95$ & $17.2$\\
\bottomrule
\end{tabular}
\end{table}

\paragraph{Estimators.}
Both estimators use the kernel $k(u,u')=1+\exp(-\|u-u'\|^2/2)$ for the hypothesis and, in TRAE, for the critic. The constant term allows functions that do not vanish at infinity, such as $h_s$ and its conditional mean. Outcomes are standardized on the fitting block. Each hypothesis is restricted to the eigenvectors of its Gram matrix whose eigenvalues exceed $\sqrt{N\epsilon_{\mathrm{mach}}}$ times the largest one, where $N$ is the size of the fitting block.

TRAE fits on the first half of the sample. Its critic is the RKHS of the same kernel with regularization term $\eta\|f\|^2$, $\eta=10^{-3}$, which we profile out in closed form (\cref{sec:rkhs_computation}). RDIV splits the sample into four blocks of size $n/4$, used for the preliminary fits, the fit, calibration, and evaluation. On the preliminary block, a mixture density network estimates the conditional density of the continuous regressor given the instrument, and kernel ridge regression gives $\hat r$. Fifty draws per observation from the fitted density define $\hat{\mathcal T}$. The second stage is then a kernel ridge problem, and $D_n$ is the projected residual \eqref{eq:rdiv_projected_loss} on the calibration block. On every data set we fit all grid points and $\lambda=0$; the grid has between $16$ and $26$ points (\cref{tab:experiment_grid}), so two methods that select the same $\lambda$ return the same function.

\paragraph{Noise levels.}
The fixed-scale rule uses \eqref{eq:delta_trae} with constant one,
\[
\delta_N^2=\frac{\log\log(eN)+\log\{(K_N+1)/\xi\}}{N},
\qquad \xi=0.1,
\]
with $A_n=\delta_N^2$ for TRAE and $A_n(\lambda)=\delta_N^2/\lambda+\nu_N^2$ for RDIV, where $\nu_N$ has the same form on the calibration block. For TRAE with the RKHS critic, the statistic is a quadratic form in the residual $\boldsymbol e$ of $\hat h_\lambda$ on the fitting block,
\[
D_n(\hat h_\lambda)=\frac1N\,\boldsymbol e^\top P_\eta\boldsymbol e,
\]
where $P_\eta$ is the smoother matrix of the critic (\cref{sec:rkhs_computation}), so $D_n$ measures how much of the residual the critic can predict from $Z$. If $\boldsymbol e$ were pure noise, with independent entries of mean zero given $Z$, the cross terms would vanish in expectation and $\mathbb E[D_n(\hat h_\lambda)]=N^{-1}\sum_i[P_\eta]_{ii}\,\mathbb E[e_i^2]$: even then the critic fits part of the noise by chance. The noise-floor rule estimates $\mathbb E[e_i^2]$ by the squared part of the residual that the critic does not explain, which gives
\[
A_n(\lambda)=\frac1N\sum_{i=1}^N[P_\eta]_{ii}\,[(I-P_\eta)\boldsymbol e]_i^2 .
\]
Equivalently, $A_n(\lambda)$ is the expected statistic of the vector $(I-P_\eta)\boldsymbol e$ after each entry is multiplied by an independent random sign. The signs remove any dependence on $Z$ and keep the scale of the residual, and the expectation over the signs has the closed form above, so no signs are drawn. This is the wild-bootstrap null mean of the statistic.
For RDIV, $A_n(\lambda)$ adds the variance of $\hat{\mathcal T}\hat h_\lambda$ on the calibration block, the variance of $\hat r$, and an estimate of $\|(\mathcal T-\hat{\mathcal T})\hat h_\lambda\|^2$; the paragraph on RDIV nuisances below gives the formulas. Both noise-floor levels are bounded below by $1/N$.

\paragraph{Errors.}
The target is $\theta_0=\mathbb E[h_0(X+1)-h_0(X)]$ in the linear design and $\theta_0=\mathbb E[h_0(1,X)-h_0(0,X)]$ in the nonlinear design, and the target error is $(\hat\theta-\theta_0)^2$, where $\hat\theta$ is the plug-in average over held-out observations. Strong errors use a trapezoid rule on $2001$ points in $[-8,8]$ for the standard normal regressor. Weak errors integrate the analytic conditional law of $X$ given the instrument over a grid in $(Z,V)$ on $[-6,6]^2$, with $241\times161$ points in the $90$-regime study and $121\times81$ points in the two-design study. In the nonlinear design, both errors use a joint grid over $(A,Z,V)$ weighted by $\mathbb P(A\mid Z)$. The plug-in target uses the second half of the sample for TRAE and the last quarter for RDIV. In the nonlinear design, the continuous regressor is standardized with the mean and standard deviation of the first half of the sample.

\paragraph{RDIV nuisances and noise level.}
The mixture density network has $20$ components and hidden width $64$. It is trained by Adam with learning rate $10^{-3}$ and weight decay $10^{-4}$ for at most $300$ epochs and stops after $25$ epochs without improvement on a quarter of the preliminary block. The pilot $\hat r$ is kernel ridge regression on the preliminary block, with the kernel above and ridge $10^{-3}$. Fifty draws per observation from the fitted density, frozen across candidates, define $\hat{\mathcal T}$.

Let $G$ and $G_{\mathrm{cal}}$ be the matrices that map the coordinates $\boldsymbol b$ of a hypothesis to $\hat{\mathcal T}h$ on the fitting and calibration blocks, and write $G^\top G=W\operatorname{diag}(\mu_j)W^\top$. The candidate is $\hat{\boldsymbol b}_\lambda=(G^\top G+\lambda I)^{-1}G^\top\boldsymbol y$. Let $\hat\sigma^2$ be the mean squared difference between $\boldsymbol y$ and $\hat r(Z)$ on the fitting block. Treating the noise as homoskedastic with variance $\hat\sigma^2$, the variance of $\hat{\mathcal T}\hat h_\lambda$ on the calibration block is
\[
\hat\sigma^2\sum_j\frac{\mu_j}{(\mu_j+\lambda)^2}\,\frac{\|G_{\mathrm{cal}}W\boldsymbol e_j\|^2}{N_{\mathrm{cal}}},
\]
where $\boldsymbol e_j$ is the $j$th unit vector. The pilot has the form $\hat r(z)=\boldsymbol w(z)^\top\boldsymbol y_{\mathrm{pre}}$, and its variance term is $\hat\nu^2=\hat\sigma^2N_{\mathrm{cal}}^{-1}\sum_{i\in I_{\mathrm{cal}}}\|\boldsymbol w(Z_i)\|^2$. For the operator error, the difference $\hat h_\lambda(X_i)-\hat{\mathcal T}\hat h_\lambda(Z_i)$ has conditional mean $\{(\mathcal T-\hat{\mathcal T})\hat h_\lambda\}(Z_i)$ given $Z_i$. On the calibration block, we evaluate the TRAE statistic with the kernel critic on $Z$ at this difference, subtract its sign-flip level, and take the positive part. The noise level $A_n(\lambda)$ is the sum of the three terms.

\paragraph{Grids.}
\Cref{tab:experiment_grid} lists the fitting-block sizes and grid lengths. Every grid starts at $\lambda_0=2$ and halves at each step. The best fixed penalty of a regime is taken over its whole grid, and the best single fixed penalty over the grid values present at every sample size.

\begin{table}[!htp]
\centering
\caption{Fitting-block sizes and grid lengths. $M=K_N+1$ is the number of candidates, and the grid ends at $2^{1-K_N}$.}
\label{tab:experiment_grid}
\small
\begin{tabular}{r r r c r r c}
\toprule
$n$ & $N_{\mathrm{TRAE}}$ & $M$ & Final $\lambda$ & $N_{\mathrm{RDIV}}$ & $M$ & Final $\lambda$\\
\midrule
500 & 250 & 18 & $2^{-16}$ & 125 & 16 & $2^{-14}$\\
1000 & 500 & 20 & $2^{-18}$ & 250 & 18 & $2^{-16}$\\
2000 & 1000 & 22 & $2^{-20}$ & 500 & 20 & $2^{-18}$\\
4000 & 2000 & 24 & $2^{-22}$ & 1000 & 22 & $2^{-20}$\\
8000 & 4000 & 26 & $2^{-24}$ & 2000 & 24 & $2^{-22}$\\
\bottomrule
\end{tabular}
\end{table}

\paragraph{Pilot runs.}
Two pilot runs on seeds $1001$--$1008$ fixed the noise-floor rule before the main seeds $2001$--$2050$ were drawn. The first used $12$ regimes of the linear design, with $\rho\in\{0.65,0.95\}$, $s=0.7$, $\sigma\in\{0.1,3\}$, and $n\in\{500,2000,8000\}$; the second used $16$ regimes of both designs at $n\in\{1000,4000\}$. For RDIV, the operator-error term was added after the pilot showed that without it the search reached the end of the grid in up to $38\%$ of the data sets. On the pilot, the fixed-scale RDIV level $\delta_N^2/\lambda+\nu_N^2$ performed about as well on the strong error and worse on the weak and target errors.

\paragraph{Computation.}
The results come from two studies: the $90$-regime TRAE study of the linear design and the two-design study of TRAE and RDIV. In both studies the linear-design TRAE fits at $s=0.7$ coincide. The kernel computations ran on CPUs; the mixture density networks for RDIV ran on GPUs.

\subsection{Additional results}\label{sec:experiment_more_results}

\Cref{tab:kernel_trae} summarizes the $90$ TRAE regimes in the linear design. The rule has lower risk than $\lambda=0$ in all $90$ regimes. Its largest ratios occur at $\rho=0.65$ and $n=500$, where the best penalty sits at the top of the grid and the rule stops at a smaller one. The fixed-scale rule has geometric mean ratios $1.24$ for the strong error and $1.59$ for the weak error. Its constant also matters: with constant $0.5$ or $2$ in place of one, the worst strong-error ratio becomes $14.8$ or $8.2$. The noise-floor rule has no such constant. At the threshold where the noise-floor rule stopped, the analytic $h_0$ would also have passed in $86\%$ of the data sets. The search never reached the end of the grid and accepted the initial value $\lambda_0=2$ in $3\%$ of the data sets.

\begin{table}[!htp]
\centering
\caption{TRAE with the kernel class in the linear Gaussian design, $90$ regimes. Each entry summarizes the risk ratio over regimes: its geometric mean, median, and largest value, and the share of regimes with ratio at most $1.5$. The best single fixed penalty is chosen after seeing the results, separately for each error.}
\label{tab:kernel_trae}
\small
\begin{tabular}{llcccc}
\toprule
Error & Method & Geometric mean & Median & Worst & Within $1.5$\\
\midrule
Strong & Noise-floor rule & $1.15$ & $1.13$ & $1.83$ & $96\%$\\
 & Fixed-scale rule & $1.24$ & $1.21$ & $2.52$ & $92\%$\\
 & Best single fixed, $\lambda=2^{-4}$ & $1.56$ & $1.24$ & $8.17$ & $61\%$\\
 & $\lambda=0$ & $172$ & $173$ & $7908$ & $0\%$\\
\addlinespace
Weak & Noise-floor rule & $1.25$ & $1.23$ & $1.98$ & $90\%$\\
 & Fixed-scale rule & $1.59$ & $1.51$ & $3.56$ & $50\%$\\
 & Best single fixed, $\lambda=2^{-5}$ & $1.56$ & $1.27$ & $7.75$ & $64\%$\\
 & $\lambda=0$ & $105$ & $90$ & $5179$ & $0\%$\\
\bottomrule
\end{tabular}
\end{table}

In the two-design study, the smallest TRAE margin over the best single fixed penalty is for the weak error in the nonlinear design, $1.31$ against $1.41$. In the nonlinear design, RDIV has ratio $1.22$ against $1.40$ for the target error, $1.30$ against $1.26$ for the strong error, and $1.64$ against $1.23$ for the weak error. In the linear design, the RDIV search accepted $\lambda_0$ in $53\%$ of the data sets at $n=500$, and it reached the end of the grid in $7\%$ of the data sets overall and in $17\%$ to $20\%$ at $n=8000$, mostly at the lowest noise level. The RDIV statistic increased by more than $5\%$ in $18\%$ of the grid steps, while the TRAE statistic never did. \Cref{fig:rdiv_curves} is the RDIV counterpart of \cref{fig:kernel_curves}. The adaptive curve again follows the best fixed penalty in most regimes, but it is noisier than for TRAE, with visible departures at $\rho=0.65$, $\sigma=3$, $n=1000$ and at $\rho=0.95$, $\sigma=1$, $n=2000$.

\begin{figure}[!htp]
\centering
\includegraphics[width=\linewidth]{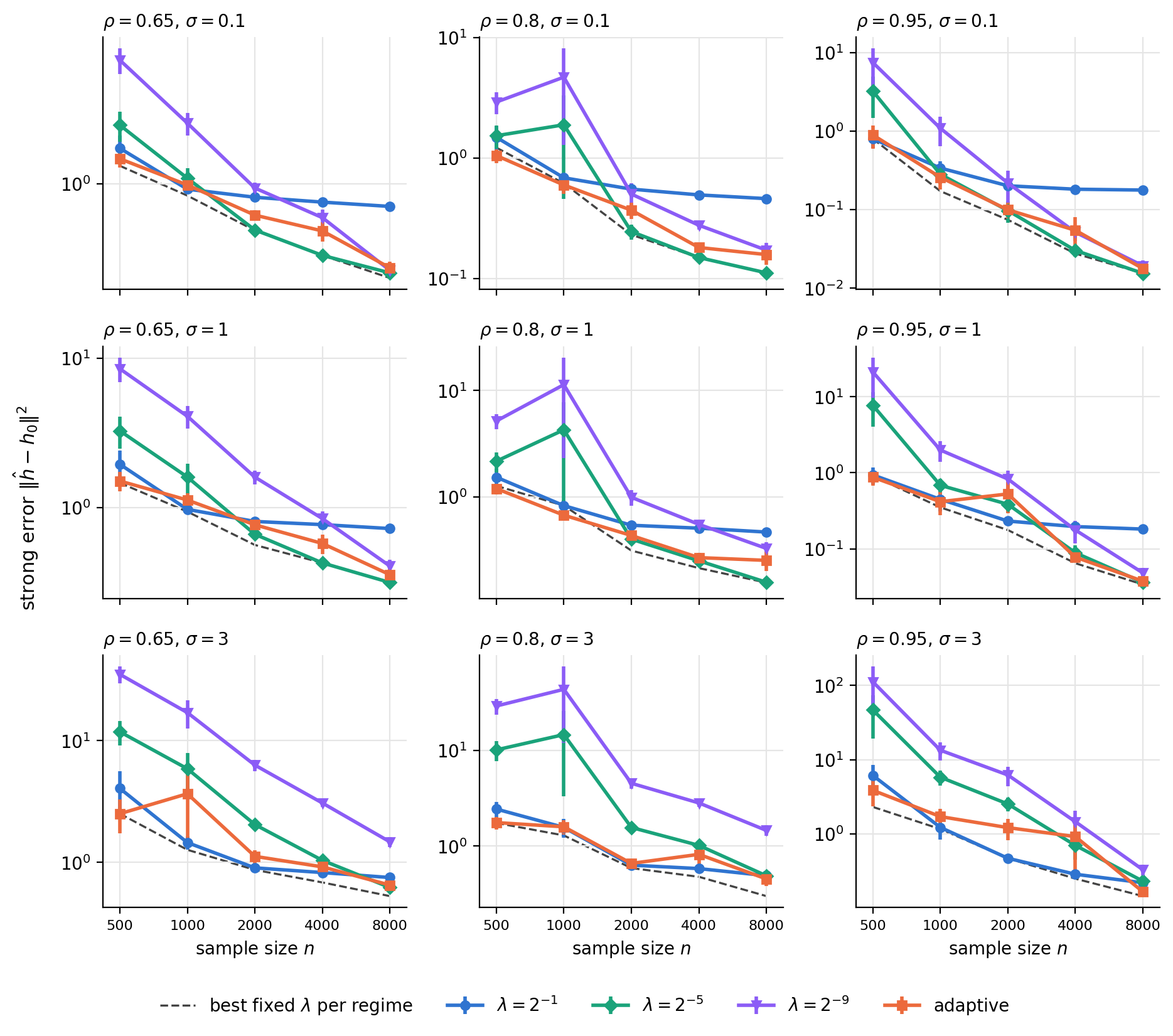}
\caption{Strong error of RDIV against the sample size in the linear design with $s=0.7$, for three fixed penalties and the noise-floor rule. The dashed curve is the best fixed penalty of each regime. Bars show $\pm1$ Monte Carlo standard error over $20$ data sets.}
\label{fig:rdiv_curves}
\end{figure}

\subsection{RKHS computation}\label{sec:rkhs_computation}

Let $N$ be the fitting-sample size, $\boldsymbol y=(Y_1,\ldots,Y_N)^\top$, and define the Gram matrices
\[
[K_{\mathcal H}]_{ij}=k_{\mathcal H}(X_i,X_j),
\qquad
[K_{\mathcal F}]_{ij}=k_{\mathcal F}(Z_i,Z_j),
\qquad
P_\eta=K_{\mathcal F}(K_{\mathcal F}+N\eta I_N)^{-1}.
\]
For the minimum-RKHS-norm representative $h_{\boldsymbol c}(x)=\sum_{j=1}^Nc_j k_{\mathcal H}(X_j,x)$, the critic can be profiled out in closed form. Before the spectral restriction below, the RKHS TRAE fit at a positive candidate $\lambda$ solves the convex quadratic program
\begin{equation}
\label{eq:rkhs_qp}
\widehat{\boldsymbol c}_\lambda
\in\arg\min_{\boldsymbol c\in\mathbb R^N}
\left\{
(\boldsymbol y-K_{\mathcal H}\boldsymbol c)^\top P_\eta
(\boldsymbol y-K_{\mathcal H}\boldsymbol c)
+\lambda\boldsymbol c^\top K_{\mathcal H}^2\boldsymbol c
\right\}.
\end{equation}
Equivalently, after dropping the constant $\boldsymbol y^\top P_\eta\boldsymbol y$, the objective is
\[
\boldsymbol c^\top(K_{\mathcal H}P_\eta K_{\mathcal H}
+\lambda K_{\mathcal H}^2)\boldsymbol c
-2\boldsymbol c^\top K_{\mathcal H}P_\eta\boldsymbol y.
\]
Both $P_\eta$ and $K_{\mathcal H}$ are positive semidefinite, so this objective is convex and its stationarity condition
\begin{equation}\label{eq:rkhs_normal}
\left(K_{\mathcal H}P_\eta K_{\mathcal H}+\lambda K_{\mathcal H}^{2}\right)\boldsymbol c
=K_{\mathcal H}P_\eta\boldsymbol y
\end{equation}
identifies the minimizer. \Cref{prop:rkhs_qp} gives its minimum-Euclidean-norm solution, $\widehat{\boldsymbol c}_\lambda=(K_{\mathcal H}P_\eta K_{\mathcal H}+\lambda K_{\mathcal H}^{2})^{\dagger}K_{\mathcal H}P_\eta\boldsymbol y$. Each candidate therefore costs a linear solve, with no iterative optimization.

A change of variable gives a well-conditioned computation on a prescribed spectral subspace. Let $Q_{\mathcal H}$ contain the orthonormal eigenvectors of $K_{\mathcal H}$ whose eigenvalues exceed $\epsilon_{\mathcal H}\lambda_{\max}(K_{\mathcal H})$, and let $D_{\mathcal H}$ contain those eigenvalues. The default relative tolerance is $\epsilon_{\mathcal H}=\sqrt{N\epsilon_{\mathrm{mach}}}$. Restrict the hypothesis to the corresponding kernel span. Its fitting values have the form $\boldsymbol v=Q_{\mathcal H}\boldsymbol a$, and its coefficients are $\boldsymbol c=Q_{\mathcal H}D_{\mathcal H}^{-1}\boldsymbol a$. With
\[
M=Q_{\mathcal H}^\top P_\eta Q_{\mathcal H},
\qquad
\boldsymbol b=Q_{\mathcal H}^\top P_\eta\boldsymbol y,
\]
the restricted quadratic program becomes
\begin{equation}\label{eq:rkhs_value_program}
\widehat{\boldsymbol a}_\lambda
=\arg\min_{\boldsymbol a}
\left\{(\boldsymbol y-Q_{\mathcal H}\boldsymbol a)^\top P_\eta
(\boldsymbol y-Q_{\mathcal H}\boldsymbol a)+\lambda\|\boldsymbol a\|^2\right\},
\qquad
(M+\lambda I)\widehat{\boldsymbol a}_\lambda=\boldsymbol b.
\end{equation}
The matrix $M$ is positive semidefinite with spectrum in $[0,1)$, so the system has condition number at most $(1+\lambda)/\lambda$ for $\lambda>0$.

Write $M=V\operatorname{diag}(m_i)V^\top$ and $\boldsymbol t=V^\top\boldsymbol b$. Then the whole regularization path follows from
\begin{equation}\label{eq:rkhs_path}
\widehat{\boldsymbol a}_\lambda
=V\operatorname{diag}\!\left(\frac{1}{m_i+\lambda}\right)\boldsymbol t,
\qquad
L_{N,\eta}(\widehat h_\lambda)
=\frac1N(\boldsymbol y-Q_{\mathcal H}\widehat{\boldsymbol a}_\lambda)^\top
P_\eta(\boldsymbol y-Q_{\mathcal H}\widehat{\boldsymbol a}_\lambda).
\end{equation}
At $\lambda=0$, we use the minimum-Euclidean-norm solution after retaining compressed eigenvalues larger than the absolute numerical tolerance $10^{-12}$. The selection statistic and predictions use the same fitted function, since $K_{\mathcal H}\widehat{\boldsymbol c}_\lambda=Q_{\mathcal H}\widehat{\boldsymbol a}_\lambda$ on the retained subspace. Kernel and compressed-matrix eigendecompositions are computed once; each subsequent candidate costs at most $O(N^2)$ operations. The computed fit minimizes the objective on the retained hypothesis span, up to numerical error. This restriction acts as additional regularization and is part of the estimator used in all experiments.

For RDIV, let $\boldsymbol y$ be the standardized outcomes on the fitting block, and restrict the hypothesis to the retained span as above, with coordinates $\boldsymbol b$. Averaging the hypothesis features over the frozen draws from the fitted density gives a matrix $G$ with $G\boldsymbol b=(\hat{\mathcal T}h(Z_i))_{i=1}^N$. The candidate minimizes $N^{-1}\{\|\boldsymbol y-G\boldsymbol b\|^2+\lambda\|\boldsymbol b\|^2\}$, so that
\[
(G^\top G+\lambda I)\widehat{\boldsymbol b}_\lambda=G^\top\boldsymbol y,
\]
and one eigendecomposition of $G^\top G$ gives the whole path. The discrepancy is $N_{\mathrm{cal}}^{-1}\|\hat{\boldsymbol r}_{\mathrm{cal}}-G_{\mathrm{cal}}\widehat{\boldsymbol b}_\lambda\|^2$, where $G_{\mathrm{cal}}$ is formed in the same way on the calibration block and $\hat{\boldsymbol r}_{\mathrm{cal}}$ contains the pilot values there. The retained span and the critic ridge are computational choices; the guarantees for bounded classes do not follow from these reductions alone.

\subsection{RKHS reduction}

For the kernel TRAE estimator in \cref{sec:experiment_estimators}, define the ridge-stabilized empirical adversarial loss
\begin{equation}
\label{eq:rkhs_adversarial_loss}
L_{N,\eta}(h)
:=\sup_{f\in\mathcal F}
\left\{
\frac{2}{N}\sum_{i=1}^N\{Y_i-h(X_i)\}f(Z_i)
-\frac{1}{N}\sum_{i=1}^N f(Z_i)^2
-\eta\|f\|_{\mathcal F}^2
\right\},
\end{equation}
where $\mathcal F$ is the RKHS generated by $k_{\mathcal F}$ and $\eta>0$. The fitted hypothesis minimizes
\[
L_{N,\eta}(h)+\frac{\lambda}{N}\sum_{i=1}^N h(X_i)^2.
\]

\begin{proposition}[RKHS quadratic reduction]
\label{prop:rkhs_qp}
Let $\mathcal H$ and $\mathcal F$ be the RKHSs generated by $k_{\mathcal H}$ and $k_{\mathcal F}$, and let $K_{\mathcal H}$, $K_{\mathcal F}$, and $P_\eta$ be defined as in \cref{sec:rkhs_computation}. Among functions with the same fitting-sample evaluations, take the representative with minimum RKHS norm. Then the minimizer of
\[
L_{N,\eta}(h)+\lambda\mathbb E_N[h(X)^2]
\]
has the form $\widehat h_\lambda(\cdot)=\sum_{j=1}^N\widehat c_{\lambda,j}k_{\mathcal H}(X_j,\cdot)$, where $\widehat{\boldsymbol c}_\lambda$ solves \eqref{eq:rkhs_qp}. The minimum-Euclidean-norm coefficient vector is
\begin{equation}
\label{eq:rkhs_closed_form}
\widehat{\boldsymbol c}_\lambda
=\left(K_{\mathcal H}P_\eta K_{\mathcal H}
+\lambda K_{\mathcal H}^2\right)^\dagger
K_{\mathcal H}P_\eta\boldsymbol y,
\end{equation}
where $A^\dagger$ denotes the Moore--Penrose pseudoinverse of $A$.
\end{proposition}

\begin{proof}
Fix $h$ and write $\boldsymbol r_h=(Y_i-h(X_i))_{i=1}^N$. The representer theorem gives a maximizing critic of the form
\[
f_\gamma(\cdot)=\sum_{j=1}^N\gamma_j k_{\mathcal F}(Z_j,\cdot).
\]
Its fitting-sample evaluations are $K_{\mathcal F}\gamma$ and its squared RKHS norm is $\gamma^\top K_{\mathcal F}\gamma$. Substitution into \eqref{eq:rkhs_adversarial_loss} gives
\[
\sup_{\gamma\in\mathbb R^N}
\left\{
\frac{2}{N}\boldsymbol r_h^\top K_{\mathcal F}\gamma
-\frac{1}{N}\gamma^\top K_{\mathcal F}^2\gamma
-\eta\gamma^\top K_{\mathcal F}\gamma
\right\}.
\]
The first-order condition is
\[
K_{\mathcal F}
\{(K_{\mathcal F}+N\eta I_N)\gamma-\boldsymbol r_h\}=0.
\]
We may therefore take
\[
\widehat\gamma_h=(K_{\mathcal F}+N\eta I_N)^{-1}\boldsymbol r_h.
\]
Evaluating the objective at this coefficient vector yields
\begin{equation}
\label{eq:rkhs_profiled_loss}
L_{N,\eta}(h)
=\frac{1}{N}\boldsymbol r_h^\top
K_{\mathcal F}(K_{\mathcal F}+N\eta I_N)^{-1}
\boldsymbol r_h
=\frac{1}{N}\boldsymbol r_h^\top P_\eta\boldsymbol r_h.
\end{equation}

The objective depends on $h$ only through its fitting-sample evaluations. Its minimum-RKHS-norm representative consequently has the form $h_{\boldsymbol c}(\cdot)=\sum_{j=1}^Nc_j k_{\mathcal H}(X_j,\cdot)$, with evaluation vector $K_{\mathcal H}\boldsymbol c$. Combining \eqref{eq:rkhs_profiled_loss} with
\[
\mathbb E_N[h_{\boldsymbol c}(X)^2]
=\frac{1}{N}\boldsymbol c^\top K_{\mathcal H}^2\boldsymbol c
\]
and multiplying the objective by $N$ gives \eqref{eq:rkhs_qp}. Since $P_\eta$ is positive semidefinite, this is a convex quadratic program. Its first-order condition is
\[
\left(K_{\mathcal H}P_\eta K_{\mathcal H}
+\lambda K_{\mathcal H}^2\right)\boldsymbol c
=K_{\mathcal H}P_\eta\boldsymbol y.
\]
The Moore--Penrose solution in \eqref{eq:rkhs_closed_form} is the minimum-Euclidean-norm solution of this system.
\end{proof}

\end{document}